\documentclass{article}

\usepackage[main, final]{neurips_2026}
\usepackage{titletoc}
\usepackage[utf8]{inputenc} 
\usepackage[T1]{fontenc}    
\usepackage{hyperref}       
\usepackage{url}            
\usepackage{booktabs}       
\usepackage{amsfonts}       
\usepackage{nicefrac}       
\usepackage{microtype}      
\usepackage{xcolor}         
\usepackage{microtype}
\usepackage{graphicx}
\usepackage{subcaption}
\usepackage{booktabs} 
\usepackage{multirow}
\usepackage{enumitem}
\usepackage{makecell}
\usepackage{empheq}
\usepackage{amsmath}
\usepackage{amssymb}
\usepackage{mathtools}
\usepackage{amsthm}
\usepackage{subcaption}
\usepackage{natbib}
\usepackage{algorithm}
\usepackage{algpseudocode}
\usepackage{caption}
\usepackage{wrapfig}

\usepackage{amsmath,amsfonts,bm}

\def\eqref#1{equation~\ref{#1}}

\def\1{\bm{1}}

\DeclareMathAlphabet{\mathsfit}{\encodingdefault}{\sfdefault}{m}{sl}
\SetMathAlphabet{\mathsfit}{bold}{\encodingdefault}{\sfdefault}{bx}{n}

\usepackage{tikz}
\usepackage{pgfplots}
\usepackage{wrapfig}
\pgfplotsset{compat=1.18}
\usepackage[most]{tcolorbox}
\usepackage{tabularx}
\usepackage{array}
\usepackage{makecell}
\newcolumntype{L}[1]{>{\raggedright\arraybackslash}p{#1}}
\newcolumntype{C}[1]{>{\centering\arraybackslash}p{#1}}
\definecolor{algteal}{RGB}{35,115,115}
\definecolor{algmagenta}{RGB}{170,35,110}
\definecolor{algblue}{RGB}{35,70,160}
\definecolor{myorange}{RGB}{255,127,14}
\definecolor{myblue}{RGB}{0,174,239}

\usepackage{colortbl} 
\definecolor{oursrow}{RGB}{235,232,226} 

\newcommand{\alglabel}[1]{\textcolor{algteal}{\textbf{#1}}}
\newcommand{\algaction}[1]{\textcolor{algmagenta}{\textbf{#1}}}
\newcommand{\algmath}[1]{\textcolor{algblue}{$#1$}}

\usepackage[capitalize,noabbrev]{cleveref}
\theoremstyle{plain}
\newtheorem{theorem}{Theorem}[section]
\newtheorem{proposition}[theorem]{Proposition}
\newtheorem{lemma}[theorem]{Lemma}
\newtheorem{corollary}[theorem]{Corollary}
\theoremstyle{definition}

\newtheorem{assumption}[theorem]{Assumption}
\theoremstyle{remark}
\newtheorem{remark}[theorem]{Remark}
\algrenewcommand\algorithmicrequire{\textbf{Input:}}
\algrenewcommand\algorithmicensure{\textbf{Output:}}

\usepackage[textsize=tiny]{todonotes}

\title{Preconditioned Flow Matching}

\author{%
  Shadab Ahamed\\
  University of British Columbia,\\Vancouver, Canada\\
  \texttt{shadab.ahamed@hotmail.com}
 \And
  Eshed Gal \\
  University of British Columbia,\\Vancouver, Canada\\
  \texttt{eshedg@cs.ubc.ca} \\
  \And 
  Md Shahriar Rahim Siddiqui\\
  University of British Columbia,\\Vancouver, Canada\\
  \texttt{shahriar.siddiqui@ubc.ca} \\
  \And 
  Simon Ghyselincks \\
  University of British Columbia,\\Vancouver, Canada\\
  \texttt{sghyseli@cs.ubc.ca} \\
  \And 
  Moshe Eliasof \\
  Ben-Gurion University of the Negev,\\Beersheba, Israel\\
  \texttt{eliasof@bgu.ac.il} \\
  \And 
  Eldad Haber\\
  University of British Columbia,\\Vancouver, Canada\\
  \texttt{ehaber@eoas.ubc.ca} \\
}

\begin{document}

\maketitle
\vspace{-0.9cm}
\begin{center}
  \textbf{Project page:}
  \url{https://ahxmeds.github.io/preconditioned-flow-matching/}
\end{center}

\begin{abstract}
Flow matching (FM) learns vector fields by regressing stochastic velocity targets along intermediate distributions $p_t$. 
We identify a geometric optimization bottleneck in this regression problem: when the covariance $\Sigma_t$ of $p_t$ is ill-conditioned, gradient-based training rapidly fits high-variance directions while making slow progress along low-variance ones. 
In an exactly solvable Gaussian setting, we prove that the excess risk is weighted by $\Sigma_t$, and that both gradient descent and stochastic gradient descent inherit condition-number-dependent convergence. 
We then extend the analysis to Gaussian mixtures, showing that multimodality does not average away this effect; instead, the slowest and worst-conditioned component can control optimization. 
Motivated by this analysis, we propose \emph{preconditioned flow matching}, a precondition-then-match framework that transforms the target distribution into a more isotropic representation, trains the main flow in the transformed space, and maps generated samples back through the inverse transformation. 
We show theoretically that preconditioning reshapes the intermediate FM path and improves its conditioning. 
Across controlled Gaussian and Gaussian-mixture experiments, latent MNIST and other high resolution image datasets up to $512{\times}512$ resolution, preconditioning improves path-conditioning diagnostics, low-eigenvalue recovery, FID, MMD, precision, and recall. 
Compute-matched baselines and preconditioner-quality ablations further show that the gains are not explained merely by additional preconditioner parameters, but by improved geometry of the downstream flow matching problem. GitHub: \url{https://github.com/ahxmeds/preconditioned-flow-matching.git}.
\end{abstract}

\section{Introduction}

Flow matching \citep{lipman2023flow} and score-based diffusion models \citep{song2020score,ho2020denoising} have become central tools for continuous-time generative modeling, achieving state-of-the-art performance across image \citep{caetano2025symmetrical_flow_matching}, audio \citep{guan2024lafma}, and 3D synthesis \citep{wang2025ctflow, danese2026flowlet}. In both paradigms, generation is framed as the process of transporting samples from a simple reference distribution to a complex data distribution by learning a time-dependent dynamical system. Despite their empirical success, these models exhibit a persistent and increasingly well-documented optimization phenomenon: \textcolor{myblue}{\emph{training loss often plateaus long before sample quality saturates}}, suggesting that convergence of the aggregate objective can conceal errors in aspects of the learned dynamics that are important for generation \citep{lin2023diffusion, xu2026diagnosing}.


Existing mitigation strategies, such as carefully tuned noise schedules \citep{aranguri2025optimizing}, time-dependent loss reweighting \citep{billera2026time, diefenbacher2020dctrgan}, or probability flow distillation \citep{zhou2025score}, can improve sample quality in practice. However, these methods primarily modify the training schedule, loss weighting, or sampling procedure, rather than directly addressing the conditioning of the regression problems induced along the flow matching path. This distinction motivates our focus on the geometry of the intermediate distributions themselves. We discuss additional related work in \Cref{app:related_work}.

In this work, we show that a key contributor to slow convergence lies in the \emph{conditioning} of the regression problems induced by flow matching objectives. At each time step, training reduces to the estimation of a vector field from samples, drawn from an intermediate distribution along the transport path. When these distributions exhibit strong anisotropy, the resulting optimization problem becomes ill-conditioned: \textcolor{myblue}{\emph{gradient-based methods rapidly fit high-variance directions while making only marginal progress along low-variance ones}}.
\begin{wrapfigure}[20]{r}{0.45\columnwidth}
    \vspace{-0.1em}
    \centering
    \includegraphics[width=0.34\columnwidth]{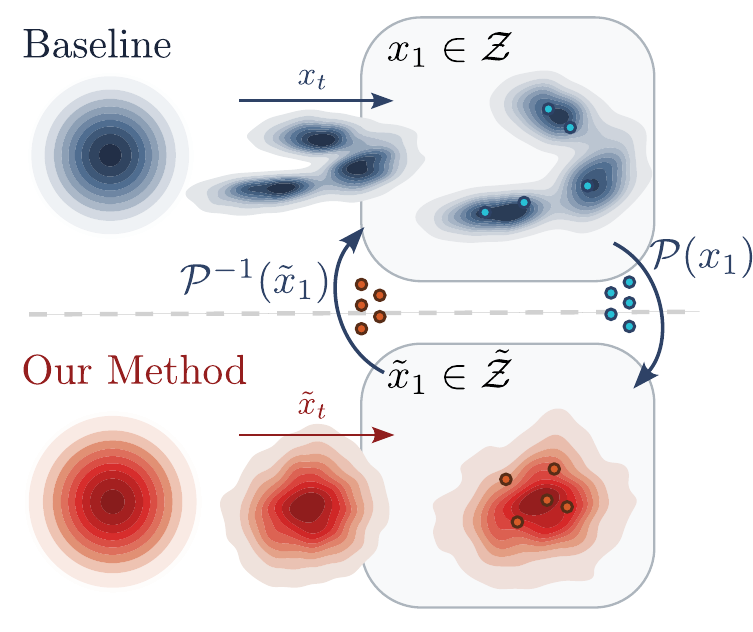}
    \caption{
        \textbf{Precondition-then-match.}
        \textbf{Top:} Standard flow matching learns through anisotropic intermediate distributions, which stagnates optimization.
        \textbf{Bottom:} A reversible preconditioner $\mathcal{P}$ maps the target to a more Gaussian-like space, improving the conditioning of the transport path. Samples are mapped back with $\mathcal{P}^{-1}$.
        }
    \label{fig:process_flow}
\end{wrapfigure}

To isolate and study this effect, we analyze flow matching in controlled settings where the exact transport dynamics are analytically tractable. We begin with a linear-Gaussian model problem, followed by a Gaussian-mixture extension that captures multimodality. In both cases, we show that optimization dynamics closely mirror those of gradient descent and stochastic gradient descent on ill-conditioned least-squares objectives, as commonly used in flow matching \citep{dao2023flow, lipman2023flow, lipman2024flow}. Crucially, these difficulties arise even when the model class is expressive enough to represent the exact velocity field, demonstrating that optimization, rather than approximation,  is the fundamental bottleneck in convergence.

Motivated by this understanding, we propose a mitigation approach for flow matching: \emph{preconditioning}. In its broader form, preconditioning refers to strategies that reshape the geometry of the learning problem without modifying the underlying generative process. This should be distinguished from standard latent generative modeling, where an
autoencoder or VAE is used primarily to obtain a compact representation for efficient generation. Nevertheless, such latent spaces need not be isotropic or well-conditioned for the subsequent flow matching regression problem. Preconditioning is therefore complementary to latent
flow matching and latent diffusion: it can be applied either in data space or on top of an existing latent representation to further improve the conditioning of the intermediate FM path. In our setting, this requires an additional preconditioning stage and associated design choices, \textit{\textcolor{myorange}{but the overhead is targeted at improving the geometry of the downstream flow matching problem rather than merely increasing model capacity}}.

\paragraph{Contributions.}
Our main contributions are: (i) A theoretical analysis of the optimization process in flow matching, revealing how data anisotropy governs the optimization speed. (ii) A principled precondition-then-match framework that improves the conditioning of the downstream FM regression problem while leaving the main flow architecture and sampling ODE unchanged. (iii) Identification and derivation of FM-specific consequences of this regression geometry: small FM loss can hide large velocity error in low-variance directions; scalar time-reweighting cannot remove fixed-time directional stiffness, and preconditioning changes the covariance of the FM path itself rather than merely rescaling the objective. (iv) We experiment with a number of datasets, from 2D points to MNIST and other image datasets, and demonstrate the novelty of our method.

\section{Flow Matching as Stochastic Regression}

Continuous-time generative models construct a transport from a simple reference distribution $p_0$ (typically Gaussian) to a target data distribution $p_1$. Rather than learning a direct map $x_0 \mapsto x_1$, flow matching introduces a family of intermediate states
\begin{equation}
\label{eq:sint}
x_t = s(t)x_1 + c(t)x_0,
\end{equation}
where $x_0 \sim p_0$, $x_1 \sim p_1$, and $s(0)=0$, $s(1)=1$, $c(0)=1$, $c(1)=0$. The goal is to learn a velocity field governing the evolution of the intermediate distributions $p_t$ over time. Flow matching (FM) learns the right-hand side of the transport ODE \citep{lipman2023flow} $dx_t / dt = v_\theta(x_t,t)$. Given the interpolation path in \Cref{eq:sint}, the target velocity is
$v_t^\star(x_0,x_1) = s'(t)x_1 + c'(t)x_0$. The standard conditional FM objective can be written as
\begin{equation}
\label{eq:cfm_expectation}
\mathcal{L}_{\mathrm{CFM}}(\theta)
=
\mathbb{E}_{t,x_0,x_1}
\Big[
\|v_\theta(x_t,t)-v_t^\star(x_0,x_1)\|^2
\Big],
\end{equation}
where $t \sim \mathrm{Unif}[0,1]$, $x_0 \sim p_0$, and $x_1 \sim p_1$.

\begin{proposition}[Conditional-mean characterization of CFM]
\label{prop:cfm_regression_decomposition}
For any measurable vector field $g:\mathbb{R}^d \times [0,1] \to \mathbb{R}^d$, define
\[
\mathcal{L}(g)
=
\mathbb{E}\big[\|g(x_t,t)-v_t^\star(x_0,x_1)\|^2\big],
\]
where the expectation is over $t \sim \mathrm{Unif}[0,1]$, $x_0 \sim p_0$, and $x_1 \sim p_1$. Then the population minimizer is
\begin{equation}
\label{eq:cfm_population_minimizer}
g^\star(x,t)
=
\underbrace{\mathbb{E}\big[v_t^\star(x_0,x_1)\mid x_t=x,\, t\big]}_{\text{conditional mean target}}.
\end{equation}
Moreover,
\begin{equation}
\label{eq:regression_decomposition}
\mathcal{L}(g)
=
\underbrace{\mathbb{E}\!\big[\|g(x_t,t)-g^\star(x_t,t)\|^2\big]}_{\text{optimization / approximation error}}
\;+\;
\underbrace{
\mathbb{E}\!\left[
\operatorname{Tr}\!\left(
\operatorname{Cov}\!\left(
v_t^\star(x_0,x_1)\mid x_t,t
\right)
\right)
\right]
}_{\text{irreducible stochastic-target error}}.
\end{equation}

\end{proposition}

The trace-of-covariance term in \Cref{eq:regression_decomposition} is the \textit{irreducible stochastic-target error (ISTE)}. In particular, $\mathcal{L}(g)-\mathcal{L}(g^\star) = \mathbb{E}\!\big[\|g(x_t,t)-g^\star(x_t,t)\|^2\big]$ represents the \textit{excess risk} (proof in \Cref{app:cfm_regression_decomposition_proof}). \Cref{prop:cfm_regression_decomposition} shows that conditional FM is a stochastic-target regression problem: the target velocity $v_t^\star(x_0,x_1)$ is generally not a deterministic function of $(x_t,t)$, so the trace-of-covariance term in \Cref{eq:regression_decomposition} is an irreducible error floor, while the population minimizer remains the conditional mean velocity field. This establishes conditional FM as regression under the intermediate distributions $p_t$. In \Cref{gaussian_model_problem}, we make the role of data geometry explicit in analytically tractable Gaussian and Gaussian-mixture flow matching settings.

\section{Exact Conditioning Theory in Gaussian and Gaussian-Mixture FM}
\label{gaussian_model_problem}

We now make the role of data geometry explicit in analytically tractable FM settings. Our goal in this section is not to prove a global convergence theorem for arbitrary nonlinear networks, but rather to isolate the optimization mechanism induced by the intermediate distributions $p_t$ in settings where the population conditional FM target can be characterized exactly. We begin with a Gaussian model, where the target velocity is exactly linear, and then extend the analysis to a Gaussian mixture model, where multimodality sharpens the same conditioning bottleneck.

\subsection{Gaussian FM as an exactly solvable regression problem}

We consider the linear interpolation $x_t = (1-t)x_0 + t x_1$ with $t \in [0,1]$, $x_0 \sim \mathcal{N}(0,I), x_1 \sim \mathcal{N}(0,H)$, where $H \in \mathbb{R}^{d \times d}$ is symmetric positive definite and $x_0$ and $x_1$ are independent. Along the interpolation, the intermediate covariance transitions from $I$ to $H$ as $t$ goes from 0 to 1, so the regression geometry gradually inherits the anisotropy of the target distribution. 

We write the eigendecomposition, $H = U \Lambda U^\top$ with $\Lambda = \operatorname{diag}(\lambda_1,\dots,\lambda_d)$ and $\lambda_i>0$. We say that $H$ is ill-conditioned when the condition number 
\begin{equation}
\kappa(H)=\frac{\lambda_{\max}(H)}{\lambda_{\min}(H)}
\end{equation}
is large.

\begin{proposition}[Exact Gaussian CFM target and intermediate covariance]
\label{prop:gaussian_exact_target}
Under the linear interpolation, the intermediate state $x_t$ is Gaussian with covariance
\begin{equation}
\Sigma_t := \mathbb{E}[x_t x_t^\top] = (1-t)^2 I + t^2 H.
\label{eq:Sigma_t}
\end{equation}
Moreover, for the linear interpolation, the conditional FM population target is
\begin{equation}
g_t^\star(x_t)
:=
\mathbb{E}[x_1-x_0 \mid x_t]
=
A^\star(t)x_t,
\label{eq:gaussian_conditional_mean_target}
\end{equation}
where
\begin{equation}
A^\star(t)
=
\operatorname{Cov}(x_1-x_0,x_t)\,\Sigma_t^{-1}
=
\bigl(tH-(1-t)I\bigr)\bigl((1-t)^2 I + t^2 H\bigr)^{-1}.
\label{eq:A_star_t}
\end{equation}
\end{proposition}

\Cref{prop:gaussian_exact_target} shows that in the Gaussian case, the population conditional FM target is exactly linear in $x_t$ (proof in \Cref{app:gaussian_exact_target_proof}). We now quantify the regression geometry induced by this exact target.

\begin{lemma}[Exact excess-risk decomposition for linear predictors]
\label{lem:gaussian_excess_risk}
Fix $t \in [0,1]$ and define
\begin{equation}
\mathcal{L}_t(A)
=
\mathbb{E}\bigl[\|A x_t - (x_1-x_0)\|^2\bigr].
\label{eq:LtA}
\end{equation}
Then the unique minimizer is $A^\star(t)$ from \Cref{eq:A_star_t}, and for every $A \in \mathbb{R}^{d \times d}$,
\begin{equation}
\label{eq:excess_risk_trace}
\mathcal{L}_t(A)
=
\underbrace{\mathcal{L}_t(A^\star(t))}_{\text{minimum achievable risk}}
+
\underbrace{\operatorname{Tr}\!\Big((A-A^\star(t))\Sigma_t(A-A^\star(t))^\top\Big)}_{\text{geometry-controlled excess risk}}.
\end{equation}
Equivalently,
\begin{equation}
\label{eq:gaussian_risk_decomp}
\mathcal{L}_t(A)
=
\underbrace{\mathbb{E}\!\left[
\operatorname{Tr}\!\left(
\operatorname{Cov}(x_1-x_0\mid x_t)
\right)
\right]}_{\text{irreducible stochastic-label error}}
+
\underbrace{\operatorname{Tr}\!\Big((A-A^\star(t))\Sigma_t(A-A^\star(t))^\top\Big)}_{\text{geometry-controlled excess risk}}.
\end{equation}
In particular,
\begin{equation}
\label{eq:gaussian_excess_risk_only}
\mathcal{L}_t(A)-\mathcal{L}_t(A^\star(t))
=
\underbrace{\operatorname{Tr}\!\Big((A-A^\star(t))\Sigma_t(A-A^\star(t))^\top\Big)}_{\text{excess risk weighted by }\Sigma_t}.
\end{equation}
\end{lemma}

\paragraph{Consequence.} \Cref{lem:gaussian_excess_risk} identifies the exact population geometry of Gaussian flow matching: after removing the irreducible stochastic-target variance, the optimization term is a quadratic form governed by $\Sigma_t$ (proof in \Cref{app:gaussian_excess_risk_proof}). The role of data geometry therefore enters directly through the spectrum of $\Sigma_t$.

\vspace{1em}

\begin{tcolorbox}[
colback=blue!4,
colframe=blue!35!black,
boxrule=0.35pt,
arc=1pt,
left=4pt,right=4pt,top=4pt,bottom=4pt,
boxsep=0pt,
before skip=3pt,
after skip=3pt
]
\textbf{Key observation:}
Even when the population conditional FM target is exactly representable, optimization can be arbitrarily slow due solely to the ill-conditioning of the intermediate covariance $\Sigma_t$, which suppresses learning along low-variance directions.
\end{tcolorbox}

\subsection{Conditioning controls optimization in Gaussian FM}

Since $\Sigma_t$ and $H$ share the same eigenvectors, the eigenvalues of $\Sigma_t$ are $\sigma_i(t) = (1-t)^2 + t^2 \lambda_i$ with $i=1,\dots,d$. Hence,
\begin{equation}
\kappa(\Sigma_t)
=
\frac{\max_i \sigma_i(t)}{\min_i \sigma_i(t)}
=
\frac{(1-t)^2 + t^2 \lambda_{\max}}{(1-t)^2 + t^2 \lambda_{\min}}.
\label{eq:kappa_sigma_t}
\end{equation}
For $t=0$, $\Sigma_t=I$ and the problem is perfectly conditioned. For $t=1$, $\Sigma_t=H$, so the regression inherits the full anisotropy of the target distribution. For intermediate times, $\kappa(\Sigma_t)$ interpolates between these two regimes.

\begin{theorem}[Gradient-descent convergence is controlled by $\Sigma_t$]
\label{thm:gd_gaussian}
Fix $t \in [0,1]$ and consider the quadratic population loss $\mathcal{L}_t(A)$ where $A \in \mathbb{R}^{d\times d}$ is the optimization variable, $x_t=(1-t)x_0+t x_1$, and $\Sigma_t=\mathbb{E}[x_t x_t^\top]$. Let the cross-covariance be
\begin{equation}
C_t := \operatorname{Cov}(x_1-x_0,x_t)=tH-(1-t)I,
\label{eq:C_t_theorem}
\end{equation}
and let $A^\star(t)$ denote the unique minimizer of $\mathcal{L}_t$, given by $A^\star(t)\Sigma_t = C_t$. Consider full-batch gradient descent applied to $\mathcal{L}_t(A)$ with step size $\eta>0$:
\begin{equation}
A_{k+1}=A_k-2\eta\bigl(A_k\Sigma_t-C_t\bigr).
\label{eq:gd_update_A_refined}
\end{equation}
Define the error as $E_k := A_k-A^\star(t)$. Then the error evolves according to
\begin{equation}
E_{k+1}=E_k\bigl(I-2\eta\Sigma_t\bigr).
\label{eq:error_recursion_gd_refined}
\end{equation}

Consequently, if $0<\eta<1/\sigma_{\max}(t)$, where $\sigma_{\max}(t)$ and $\sigma_{\min}(t)$ denote the largest and smallest eigenvalues of $\Sigma_t$, then gradient descent converges linearly to $A^\star(t)$. With the optimal fixed step size $\eta^\star = 1/(\sigma_{\max}(t)+\sigma_{\min}(t))$, the number of iterations required to reduce the error by a factor $\varepsilon \in (0,1)$ satisfies
\begin{equation}
k = \mathcal{O}\!\Big(\kappa(\Sigma_t)\log(1/\varepsilon)\Big).
\label{eq:gd_iteration_complexity_refined}
\end{equation}
\end{theorem}
Proof of \Cref{thm:gd_gaussian} is presented in \Cref{app:gd_gaussian_proof}.

\begin{corollary}[Low-variance directions are the Gaussian FM bottleneck]
\label{cor:low_variance_bottleneck}
Under the setting of \Cref{thm:gd_gaussian}, directions corresponding to small eigenvalues $\sigma_i(t)$ converge most slowly under gradient descent. In particular, if $H$ is ill-conditioned and $t$ is not close to zero, then optimization becomes dominated by the low-variance directions of the intermediate distribution (proof in \Cref{app:low_variance_bottleneck_proof}). 
\end{corollary}

We next summarize the corresponding stochastic-gradient behavior.

\begin{proposition}[Mode-wise SGD mean dynamics]
\label{cor:sgd_gaussian}
Consider unbiased single-sample stochastic gradient descent applied to \Cref{eq:LtA} at a fixed time $t$. When the error is projected onto the eigenbasis of $\Sigma_t$, each scalar mode satisfies, in conditional expectation,
\[
\mathbb{E}[e_{i,m+1}\mid e_{i,m}]
=
(1-2\eta\sigma_i(t))e_{i,m}.
\]
Thus the mean SGD dynamics inherit the same mode-dependent contraction as full-batch GD: directions with smaller $\sigma_i(t)$ contract more slowly. The constant-step stationary variance depends on the covariance of the stochastic gradient noise and does not, in general, scale universally as $1/\sigma_i(t)$ (proof in \Cref{app:sgd_gaussian_proof}).
\end{proposition}

\begin{remark}[Scope of the optimizer analysis]
\label{rem:optimizer_scope}
\Cref{thm:gd_gaussian,cor:sgd_gaussian} are exact population-level results for Gaussian flow matching under linear predictors. They should be interpreted as optimizer-level diagnostics for first-order training rather than as a global theorem for every adaptive optimizer or arbitrary nonlinear network. Their role is to isolate the conditioning mechanism that later motivates preconditioning.
\end{remark}

\subsection{Gaussian-mixture FM: multimodality worsens the bottleneck}

We now extend the same analysis to a Gaussian mixture target,
\begin{equation}
x_1 \sim \sum_{k=1}^K \pi_k \,\mathcal{N}(0,H_k),
\qquad
\pi_k>0,
\qquad
\sum_{k=1}^K \pi_k = 1,
\label{eq:gmm_target}
\end{equation}
with $x_0 \sim \mathcal{N}(0,I), x_t = (1-t)x_0 + t x_1$. The Gaussian-mixture setting is useful because it introduces multimodality while remaining analytically structured. It therefore provides a stylized model for datasets with multiple semantic clusters or modes.

\begin{assumption}[Perfect component assignment]
\label{ass:gmm_perfect_gating}
For the purpose of isolating conditioning effects, we assume that the mixture responsibilities are known exactly, so that optimization decomposes componentwise.
\end{assumption}

\begin{proposition}[Exact conditional-mean target in Gaussian-mixture FM]
\label{prop:gmm_conditional_mean}
Under \Cref{eq:gmm_target}, let
\begin{equation}
w_k(x_t,t) := \mathbb{P}(x_1 \text{ belongs to component } k \mid x_t,t)
\end{equation}
denote the posterior component weights. Then the conditional FM population target is
\begin{equation}
g_t^\star(x_t)
=
\mathbb{E}[x_1-x_0 \mid x_t]
=
\sum_{k=1}^K w_k(x_t,t)\,A_k^\star(t)\,x_t,
\label{eq:gmm_population_target}
\end{equation}
\end{proposition}
\textit{where} $A_k^\star(t)=\bigl(tH_k-(1-t)I\bigr)\bigl((1-t)^2 I + t^2 H_k\bigr)^{-1}$ (proof in \Cref{app:gmm_conditional_mean_proof}). Under \Cref{ass:gmm_perfect_gating}, the problem decomposes into $K$ componentwise regression subproblems with covariances $\Sigma_{t,k} = (1-t)^2 I + t^2 H_k$. If $\lambda_{k,i}$ denotes the $i^\text{th}$ eigenvalue of $H_k$, then the corresponding eigenvalues of $\Sigma_{t,k}$ are $\sigma_{k,i}(t) = (1-t)^2 + t^2 \lambda_{k,i}$ .

\begin{theorem}[Worst-component bottleneck in Gaussian-mixture FM]
\label{thm:gmm_bottleneck}
Under \Cref{ass:gmm_perfect_gating}, gradient-based optimization of the componentwise Gaussian-mixture FM problem is governed by the slowest componentwise subproblem. In particular, both full-batch GD and the mean dynamics of SGD are controlled by the slowest componentwise eigendirection, $\min_{k,i}\sigma_{k,i}(t)$ that is, by the smallest eigenvalue across all mixture components and directions. Consequently, even if most components are well-conditioned, a single poorly conditioned component can dominate the overall convergence time scale (proof in \Cref{app:gmm_bottleneck_proof}).
\end{theorem}

In particular, if one seeks comparable accuracy across all components under the decomposition in \Cref{ass:gmm_perfect_gating}, the overall rate is governed by the slowest subproblem.

\begin{corollary}[Multimodality amplifies conditioning effects]
\label{cor:gmm_multimodal_takeaway}
Under \Cref{thm:gmm_bottleneck}, multimodality does not average away conditioning effects: the overall optimization bottleneck is governed by the hardest mixture component (proof in \Cref{app:gmm_multimodal_takeaway_proof}).
\end{corollary}

\begin{tcolorbox}[
colback=blue!4,
colframe=blue!35!black,
boxrule=0.35pt,
arc=1pt,
left=4pt,right=4pt,top=4pt,bottom=4pt,
boxsep=0pt,
before skip=3pt,
after skip=3pt
]
\textbf{Key observation:}
In Gaussian-mixture FM, optimization is controlled by the worst-conditioned component rather than by average data geometry. Even if most components are well-behaved, a single anisotropic component can dominate convergence, so multimodality sharpens rather than averages away the conditioning bottleneck.
\end{tcolorbox}

\paragraph{Scope of the theory.} The results above are exact for Gaussian and Gaussian-mixture flow matching at the population level. They show that, even when the target velocity is exactly representable, optimization can be slowed dramatically by anisotropy in the intermediate distributions. The zero-mean Gaussian assumption is only for exposition: \Cref{app:nonzero_mean_gaussian_and_gmm} extends the analysis to nonzero means, for which the conditional target becomes affine while the covariance-controlled excess-risk geometry and linear gradient dynamics remain unchanged after centering. This mechanism motivates preconditioning in the next section. Here, we do not claim a closed-form global convergence theorem for arbitrary nonlinear flow networks; rather, the purpose of this section is to isolate the conditioning effect in solvable settings and derive testable predictions for the practical models studied later.

\section{Preconditioning for Flow Matching}
\label{preconditioning}
\begin{figure*}
    \centering
    \includegraphics[width=\linewidth]{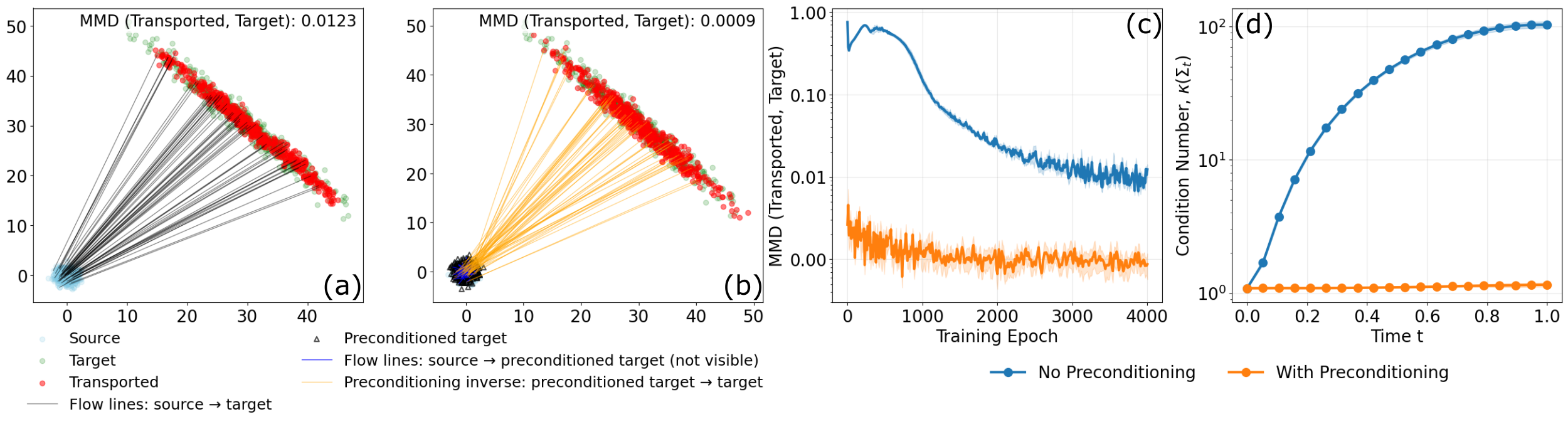}
    \caption{\textbf{2D Gaussian transport with and without preconditioning.}
    \textbf{(a)} Standard FM transports samples from an isotropic source to an elongated target, but might struggle along the narrow direction, resulting in a higher MMD. \textbf{(b)} Preconditioned FM transports samples to a whitened target and maps them back using the inverse preconditioner. \textbf{(c)} Preconditioning reduces MMD and avoids early optimization stagnation. \textbf{(d)} It also lowers the condition number of intermediate covariance $\Sigma_t$ along the transport path, improving sample alignment and stability. Curves show the mean over 10 independent runs with different random seeds; shaded regions denote one standard deviation.}
    \label{fig:gaussian_to_elongated_gaussian}
\end{figure*}

The results of \Cref{gaussian_model_problem} show that optimization in FM is bottlenecked by the conditioning of the intermediate covariances $\Sigma_t$ (for Gaussian) and $\Sigma_{t,k}$ (for Gaussian mixtures). This suggests a simple remedy: \textit{\textcolor{myorange}{before training the main FM model, apply an invertible map $\mathcal{P}$ to the target data so that the pushforward target distribution is more isotropic}}. The main flow is then trained on the transformed transport $\mathcal{N}(0,I)\to \tilde{x}_1$ where $\tilde{x}_1=\mathcal{P}x_1$, and samples are mapped back to the original space using $\mathcal{P}^{-1}$. In our setting, a preconditioner is any invertible or approximately invertible transformation (learned or analytically specified) that maps the target data to a representation with more isotropic geometry.

This is analogous to classical preconditioning in numerical linear algebra, where one improves optimization by transforming an ill-conditioned problem into a better-conditioned one (see \citet{BENZI2002418,saade,golubye}). This design introduces an additional preconditioner-training stage and associated hyperparameters, but its role is targeted: it aims to improve the geometry of the downstream FM path rather than to simply increase the capacity of the generative model.

\begin{theorem}[Preconditioning reshapes the Gaussian FM path]
\label{thm:fm_preconditioning}
Consider the Gaussian FM setting of \Cref{gaussian_model_problem}, where $x_0 \sim \mathcal{N}(0,I)$, $x_1 \sim \mathcal{N}(0,H)$, $x_t=(1-t)x_0+t x_1$. Let $\mathcal{P}\in\mathbb{R}^{d\times d}$ be any invertible linear map, and define the transformed target and interpolation
\[
\tilde{x}_1=\mathcal{P}x_1, \qquad \tilde{x}_t=(1-t)x_0+t\tilde{x}_1.
\]
Then the covariance of the transformed intermediate distribution is $\tilde{\Sigma}_t=(1-t)^2 I+t^2 \mathcal{P}H\mathcal{P}^\top$. Moreover:

\begin{enumerate}[label=(\alph*), leftmargin=12pt, itemindent=0pt]
\item \textbf{Exact whitening.} If $\mathcal{P}=H^{-1/2}$, then
\begin{equation}
\tilde{\Sigma}_t=\bigl((1-t)^2+t^2\bigr)I,
\qquad
\kappa(\tilde{\Sigma}_t)=1
\quad \text{for all } t\in[0,1].
\label{eq:exact_whitening_path}
\end{equation}

\item \textbf{Approximate isotropization.} If there exist constants $0<m\le M$ such that $mI \preceq \mathcal{P}H\mathcal{P}^\top \preceq MI$, then $(1-t)^2+t^2 m \;\le\; \lambda_i(\tilde{\Sigma}_t) \;\le\;(1-t)^2+t^2 M$, and hence
\begin{equation}
\kappa(\tilde{\Sigma}_t)
\le
\frac{(1-t)^2+t^2 M}{(1-t)^2+t^2 m}.
\label{eq:kappa_tilde_bound}
\end{equation}
\end{enumerate}
\end{theorem}
The proof and the corresponding componentwise GMM extension are provided in \Cref{app:fm_preconditioning_proof,app:gmm_preconditioning_note}. Applying \Cref{thm:gd_gaussian} to the transformed Gaussian FM problem shows that gradient-based optimization now scales with $\kappa(\tilde{\Sigma}_t)$ rather than $\kappa(\Sigma_t)$. In particular, exact whitening removes the conditioning bottleneck entirely, while approximate isotropization improves it in proportion to how well $\mathcal{P}H\mathcal{P}^\top$ is conditioned. The transformed regression problem is therefore better conditioned, mitigating early optimization stagnation that would otherwise lead to suboptimal solutions.

\begin{algorithm}[htbp]
\caption{Precondition-then-Match}
\label{alg:precondition-then-match}

\raggedright
\setlength{\parskip}{1pt}
\setlength{\parindent}{0pt}

\alglabel{Input:} Target samples
\algmath{\{x_1^{(i)}\}_{i=1}^N\sim p_1}, reference
\algmath{p_0=\mathcal{N}(0,I)}, preconditioner class
\algmath{\mathcal{P}_\phi}, main FM model \algmath{v_\theta}.\\

\alglabel{Output:} Generator for samples from \algmath{p_1}.

\vspace{2pt}

\begin{enumerate}[
    leftmargin=1.2em,
    itemsep=2pt,
    topsep=1pt,
    parsep=0pt
]

    \item \algaction{Learn} an invertible or approximately invertible
    preconditioner \algmath{\mathcal{P}_\phi} on target data.

    \item \algaction{Transform} targets:
    $\textcolor{algblue}{
        \tilde{x}_1^{(i)}=\mathcal{P}_\phi(x_1^{(i)})
    }$.

    \item \algaction{Train} \algmath{v_\theta} by FM on the
    preconditioned path,
    $\textcolor{algblue}{
        \tilde{x}_t=(1-t)x_0+t\tilde{x}_1,
        x_0\sim\mathcal{N}(0,I),
        t\sim\mathrm{Unif}[0,1].
    }$

    \item \algaction{Minimize}
    $\textcolor{algblue}{
        \mathbb{E}_{t,x_0,\tilde{x}_1}
        \!\left[
        \|v_\theta(\tilde{x}_t,t)-(\tilde{x}_1-x_0)\|^2
        \right].
    }$

    \item At inference, \algaction{sample}
    \algmath{z\sim\mathcal{N}(0,I)} and solve the learned ODE
    in preconditioned space to obtain
    \algmath{\hat{\tilde{x}}_1}.

    \item \algaction{Map back}:
    $\textcolor{algblue}{
        \hat{x}_1=\mathcal{P}_\phi^{-1}(\hat{\tilde{x}}_1)
    }$.

    \item \algaction{return} \algmath{\hat{x}_1}.

\end{enumerate}
\end{algorithm}

\paragraph{Interpretation.} \Cref{thm:fm_preconditioning} makes the role of preconditioning explicit in the linear-Gaussian case: the intermediate covariance changes from $\Sigma_t$ to $\tilde{\Sigma}_t$, with exact whitening as the analytically tractable ideal. For practical nonlinear preconditioners, the analogous goal is to make the transformed target distribution $\tilde{x}_1=\mathcal{P}_\phi(x_1)$ more isotropic, so that the empirical intermediate distributions along the downstream FM path are better conditioned. Thus, the goal is not exact Gaussianization, but a transformed representation whose transport path is substantially better conditioned.

\paragraph{FM-specific implications.} The preceding results have implications that are specific to flow matching, beyond the standard observation that ill-conditioned least-squares problems are slow to optimize. With $H=U\Lambda U^\top$, let $u_i$ denote the $i^{\mathrm{th}}$ column of $U$. Since $\Sigma_t$ and $H$ share eigenvectors, \Cref{eq:gaussian_excess_risk_only} gives
\begin{equation}
\mathcal{L}_t(A)-\mathcal{L}_t(A^\star(t))
=
\sum_{i=1}^d
\sigma_i(t)\,
\bigl\|(A-A^\star(t))u_i\bigr\|^2,
\qquad
\sigma_i(t)=(1-t)^2+t^2\lambda_i .
\label{eq:fm_directional_excess_risk}
\end{equation}
Thus the FM objective does not certify velocity-field accuracy uniformly across directions. In
particular, if
$\mathcal{L}_t(A)-\mathcal{L}_t(A^\star(t))\le \varepsilon$, then
\begin{equation}
\bigl\|(A-A^\star(t))u_i\bigr\|^2
\le
\frac{\varepsilon}{\sigma_i(t)} .
\label{eq:fm_low_variance_error_bound}
\end{equation}
Consequently, when $\sigma_i(t)$ is small, \textit{\textcolor{myblue}{a small fixed-time FM excess loss may still be
consistent with large velocity error in the corresponding low-variance direction}}. This gives a
mechanism for why the FM training loss can appear nearly saturated while sample quality continues
to improve: \textit{\textcolor{myblue}{the remaining error may be concentrated in directions that are weakly weighted by $p_t$}}. This complements recent work emphasizing that squared-error training objectives can be imperfect proxies for perceptual generation quality \citep{lin2023diffusion,xu2026diagnosing}.

This also separates preconditioning from scalar time reweighting and schedule-based corrections, which are common tools for improving diffusion and FM models \citep{aranguri2025optimizing,billera2026time}. If the fixed-time objective is rescaled by a scalar weight $\alpha(t)>0$, then the Hessian with respect to $A$ is simply scaled by $\alpha(t)$, and its condition number remains governed by $\kappa(\Sigma_t)$. \textit{\textcolor{myblue}{Scalar time reweighting can change the relative importance of different times, but it cannot remove the directional stiffness of the regression problem at a fixed $t$}}. \textit{\textcolor{myorange}{By contrast, \Cref{thm:fm_preconditioning} changes the covariance of the intermediate FM path itself, replacing $\Sigma_t$ by $\tilde{\Sigma}_t$. Preconditioning therefore acts on the directional geometry of the FM regression problem, rather
than merely rescaling the objective}} \citep{BENZI2002418,saade}. Full formal statements and proofs of these FM-specific consequences are provided in \Cref{cor:small_loss_hides_velocity_error,cor:time_reweighting_no_directional_fix} in \Cref{app:fm_specific_implications}.

\paragraph{Normalizing flow preconditioner.}
A natural preconditioner is an invertible normalizing flow $\mathcal{P}_\theta$ trained by maximum likelihood to map data toward a standard Gaussian, $\tilde{x}_1=\mathcal{P}_\theta(x_1)$, using the standard change-of-variables objective with log-determinant correction \citep{kobyzev2020normalizing}. Although normalizing flows are often less expressive than modern FM backbones, they are well suited for preconditioning: their invertibility allows us to transform the target data to a more isotropic representation, train the main FM model in that space, and map samples back with $\mathcal{P}_\theta^{-1}$.

\paragraph{Low-capacity/accuracy flow preconditioner.} A second option is to use a low-capacity/accuracy FM model itself as a preconditioner. We first train a velocity field $v_\eta(x_t,t)$ using the standard FM objective. To map data samples backward through the learned flow, we solve the reverse-time ODE, $d x_\tau / d\tau = - v_\eta(x_\tau, 1-\tau), x_{\tau = 0} \sim p_{t=1}, \tau \in [0,1]$. The preconditioned representation is the terminal state,
$\tilde{x}_1 := \mathcal{P}_\eta(x_{\tau=0})$. Because the preconditioning flow has limited capacity, it captures only coarse structure, often producing a target representation that is substantially more Gaussian-like than the original data. We then train the main FM model on $\tilde{x}_1$ using the standard Gaussian as the reference distribution, and map generated samples back through $\mathcal{P}_\eta^{-1}$.

\section{Experiments}
\label{sec:experiments}

We evaluate whether preconditioning improves flow matching in three regimes: 
(i) controlled low-dimensional problems where transport geometry can be inspected directly, 
(ii) latent-space MNIST generation where we can compare several preconditioning strategies and compute-matched baselines, and 
(iii) image synthesis across resolutions using flow-based preconditioning. 
Across these settings, our goal is not to introduce a new image backbone or claim state-of-the-art image generation, but to test the central prediction of the theory: \textcolor{myorange}{\textit{reshaping the target geometry should improve the conditioning of the intermediate flow matching regression problems and lead to better sample quality}}.

\begin{figure}
    \centering
    \includegraphics[width=\textwidth]{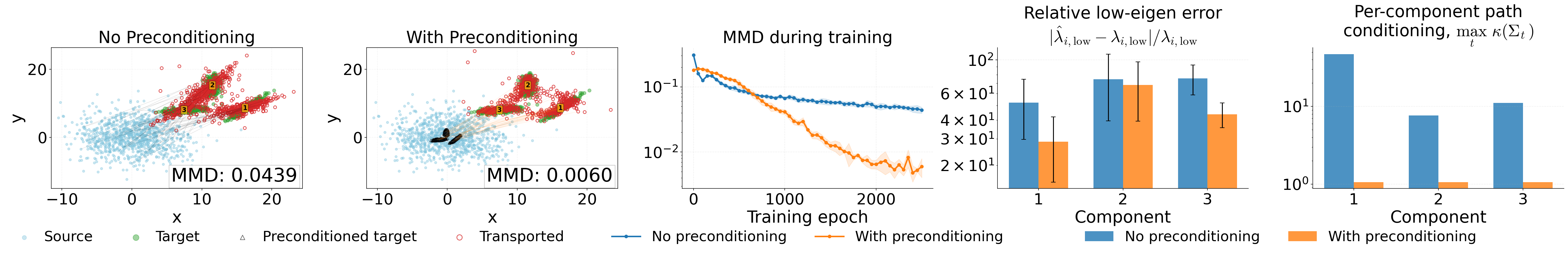}
    \caption{\textbf{Preconditioning improves Gaussian-mixture flow matching.}
        On a representative three-component anisotropic GMM, standard FM captures the dominant high-variance directions but poorly recovers the narrow components, leading to distorted transported samples and higher MMD. Preconditioning reshapes the target into a better-conditioned space, improves sample alignment, lowers MMD throughout training, reduces low-eigenvalue recovery error, and decreases the per-component path condition numbers. Plots show the mean over 10 independent runs with different random seeds; error bars / shaded regions denote one standard deviation. Experiments on additional GMM configurations are provided in \Cref{appsec:gmm_experiments}.}
    \label{fig:gmm_experiment_main}
\end{figure}

\paragraph{Metrics.} We report Fréchet Inception Distance (FID) and Maximum Mean Discrepancy (MMD) to measure distributional discrepancy, and Precision and Recall to measure complementary aspects of sample fidelity and distributional coverage. Lower is better for FID/MMD, and higher is better for Precision/Recall. 
Unless otherwise stated, all comparisons use the same downstream flow matching generator on a specific dataset.

\paragraph{Controlled diagnostics.}
The theory predicts that anisotropic intermediate distributions slow optimization by suppressing progress along low-variance directions. The Gaussian experiment in \Cref{fig:gaussian_to_elongated_gaussian} visualizes this effect: standard FM fits the high-variance direction rapidly but stagnates along the narrow direction, while preconditioning improves path conditioning and reduces MMD. We further test the same mechanism on controlled Gaussian mixtures in \Cref{fig:gmm_experiment_main}, with additional configurations in \Cref{appsec:gmm_experiments}. These experiments isolate the worst-component bottleneck predicted by \Cref{thm:gmm_bottleneck}: across mixtures with varying component numbers and condition numbers, preconditioning reduces the maximum path condition number, improves low-eigenvalue recovery, and lowers final MMD relative to standard FM. Additional Swiss-roll visualizations are provided in \Cref{subsec:point_clouds_2d}.

\paragraph{Experiments on MNIST.}
We evaluate preconditioning on MNIST in the latent space of a trained VAE. The downstream generator is a class-conditional CNN-based latent FM model, kept fixed across the main variants to isolate the effect of preconditioning. We compare standard latent FM with no preconditioning, NF and FM preconditioning, while also reporting ZCA whitening \citep{kessy2018optimal} and compute-matched unpreconditioned baselines in \Cref{tab:mnist_table}. Preconditioning consistently improves over the unpreconditioned baseline: NF gives the best FID/MMD, while FM preconditioning gives the best precision/recall. The compute-matched baseline improves only modestly, supporting that the gains come primarily from reshaping the latent geometry rather than simply adding parameters. 

\begin{table*}[h]
\centering
\caption{MNIST comparison for preconditioning strategies. Compute-matched baselines enlarge the unpreconditioned FM generator to match the total parameters of learned preconditioner + FM. Lower is better for FID/MMD; higher is better for Precision/Recall.}
\label{tab:mnist_table}

\vspace{0.5em}

\resizebox{0.85\textwidth}{!}{%
\begin{minipage}{\textwidth}
\centering
\setlength{\tabcolsep}{3.5pt}

\begin{tabular}{lcccccc}
\toprule
\textbf{\begin{tabular}[c]{@{}l@{}}Preconditioning \\ method\end{tabular}} 
& 
\begin{tabular}[c]{@{}c@{}}\textbf{Compute} \\ \textbf{matched?}\end{tabular} 
& 
\# \textbf{Params} 
& 
\textbf{FID} $\downarrow$ 
& 
\begin{tabular}[c]{@{}c@{}}\textbf{MMD} $\downarrow$ \\ ($\times 10^{-3}$)\end{tabular} 
& 
\textbf{Precision} $\uparrow$ 
& 
\textbf{Recall} $\uparrow$ 
\\ 
\midrule

No preconditioning & No  & 2.32M & 5.97 & 77.10 & 0.5774 & 0.2856 \\
No preconditioning & Yes & 3.37M & 5.42 & 75.46 & 0.5869 & 0.2618 \\

\rowcolor{oursrow}
ZCA whitening \textit{(Ours)} & No  & 2.32M & 4.20 & 6.25 & 0.6022 & 0.3107 \\

\rowcolor{oursrow}
ZCA whitening \textit{(Ours)} & Yes & 3.37M & 4.02 & 7.02 & 0.6117 & 0.2885 \\

\midrule 

\rowcolor{oursrow}
NF preconditioner \textit{(Ours)} & No  & 3.37M & \textbf{2.74} & \textbf{2.95} & 0.7271 & 0.3561 \\

\rowcolor{oursrow}
FM preconditioner \textit{(Ours)} & No  & 3.37M & 3.98 & 3.02 & \textbf{0.8440} & \textbf{0.5486} \\

\bottomrule
\end{tabular}

\end{minipage}%
}

\end{table*}

\begin{table*}[t]
\centering
\caption{FID, Precision, and Recall across diverse image datasets. We compare standard flow matching with flow matching using a learned flow preconditioner. Experiments are conducted in the latent space of a pretrained VAE for all datasets except CIFAR-10, which is modeled directly in pixel space. Lower FID and higher Precision/Recall indicate better performance.}
\label{tab:high_res_image_synthesis_results_table}

\vspace{0.5em}

\resizebox{0.95\textwidth}{!}{%
\begin{minipage}{\textwidth}
\centering
\setlength{\tabcolsep}{4pt}
\renewcommand{\arraystretch}{1.02}

\begin{tabular}{llccc}
\toprule
\textbf{Dataset}
& \textbf{Method}
& \textbf{FID} $\downarrow$
& \textbf{Precision} $\uparrow$
& \textbf{Recall} $\uparrow$ \\
\midrule

\multirow{2}{*}{\begin{tabular}[c]{@{}l@{}}LSUN Churches\\[-0.15em]\scriptsize $(256 \times 256)$\end{tabular}}
& No preconditioner
& 19.53
& 0.4540
& 0.3239 \\
&
\cellcolor{oursrow}FM preconditioner \textit{(Ours)}
& \cellcolor{oursrow}\textbf{14.47}
& \cellcolor{oursrow}\textbf{0.5313}
& \cellcolor{oursrow}\textbf{0.3255} \\
\midrule

\multirow{2}{*}{\begin{tabular}[c]{@{}l@{}}Oxford Flowers-102\\[-0.15em]\scriptsize $(256 \times 256)$\end{tabular}}
& No preconditioner
& 23.82
& 0.5213
& 0.2761 \\
&
\cellcolor{oursrow}FM preconditioner \textit{(Ours)}
& \cellcolor{oursrow}\textbf{22.64}
& \cellcolor{oursrow}\textbf{0.5428}
& \cellcolor{oursrow}\textbf{0.2854} \\
\midrule

\multirow{2}{*}{\begin{tabular}[c]{@{}l@{}}AFHQ Cats\\[-0.15em]\scriptsize $(512 \times 512)$\end{tabular}}
& No preconditioner
& 7.11
& 0.7764
& \textbf{0.5396} \\
&
\cellcolor{oursrow}FM preconditioner \textit{(Ours)}
& \cellcolor{oursrow}\textbf{4.86}
& \cellcolor{oursrow}\textbf{0.8320}
& \cellcolor{oursrow}0.5297 \\
\midrule

\multirow{2}{*}{\begin{tabular}[c]{@{}l@{}}CIFAR-10\\[-0.15em]\scriptsize $(32 \times 32)$\end{tabular}}
& No preconditioner
& 4.15
& \textbf{0.8097}
& 0.5878 \\
&
\cellcolor{oursrow}FM preconditioner \textit{(Ours)}
& \cellcolor{oursrow}\textbf{4.02}
& \cellcolor{oursrow}0.7856
& \cellcolor{oursrow}\textbf{0.5892} \\
\midrule

\multirow{2}{*}{\begin{tabular}[c]{@{}l@{}}ImageNet-1k\\[-0.15em]\scriptsize $(64 \times 64)$\end{tabular}}
& No preconditioner
& 8.66
& 0.7800
& 0.6339 \\
&
\cellcolor{oursrow}FM preconditioner \textit{(Ours)}
& \cellcolor{oursrow}\textbf{4.98}
& \cellcolor{oursrow}\textbf{0.8172}
& \cellcolor{oursrow}\textbf{0.6432} \\

\bottomrule
\end{tabular}

\end{minipage}%
}

\end{table*}

\paragraph{Image synthesis across resolutions.}
We next evaluate whether flow-based preconditioning remains useful beyond MNIST. We use a UNet or DiT-based latent flow matching generator and compare standard flow matching against flow matching with a learned flow preconditioner (UNet or DiT-based). Experiments are conducted on LSUN Churches and Oxford Flowers-102 at $256{\times}256$, AFHQ Cats at $512{\times}512$, CIFAR-10 at $32{\times}32$ and ImageNet-1k at $64{\times}64$. Both methods operate in the latent space of a pretrained VAE for all datasets except CIFAR-10, where they operate in pixel space. For these datasets, we use flow-based preconditioning rather than normalizing flow preconditioning, since normalizing flows impose architectural and invertibility constraints that are less suitable for complex image data \citep{papamakarios2021normalizing}. As shown in \Cref{tab:high_res_image_synthesis_results_table}, flow-based preconditioning improves FID on all five
datasets and generally improves precision and recall.

\paragraph{Additional analyses.}
Optimal flow matching \citep{kornilov2024optimal} and
preconditioning address complementary sources of difficulty:
the former reduces stochastic-target ambiguity, whereas the
latter improves excess-risk optimization
(\Cref{appsec:comparison_to_ofm_and_iste}).
This distinction highlights that improving the conditioning
of the regression problem and reducing the intrinsic
uncertainty of its training targets are separate, potentially
complementary objectives.
On latent MNIST, reduced path conditioning is accompanied
by more balanced learning across high- and low-variance
eigendirections
(\Cref{appsubsec:condition_number_mnist,appsec:mnist_directional_mechanism}).
These observations provide further empirical support for
the proposed mechanism, connecting improvements in
intermediate distribution geometry to more uniform
optimization across directions.
Capacity ablations show that the gains do not merely arise
from additional model capacity
(\Cref{appsec:preconditioner_quality_ablation}), while
NF preconditioning achieves lower FID than one-round Reflow
at the same main-flow NFE
(\Cref{appsec:reflow-comparison}). We further provide practical selection criteria and
inverse-Jacobian sensitivity analyses
(\Cref{appsec:preconditioner_selection,appsec:inverse-sensitivity}),
and extend preconditioning to conditional deblurring
and super-resolution tasks
(\Cref{appsec:conditional_restoration}).
These extensions illustrate the broader applicability
of the preconditioning framework beyond image
generation, while also examining practical considerations
associated with selecting and inverting the preconditioner.
Qualitative results, implementation details, and
computational resources are reported in
\Cref{qualititive_comparison_images,experimental_settings,app:computational_resources};
limitations and broader impact of our work are discussed
in \Cref{app:limitation,app:broader_impacts}.

\section{Discussion and Conclusion}
\label{sec:conclusion}
We analyzed flow matching as a time-indexed stochastic regression problem and showed that, in Gaussian and Gaussian-mixture settings, ill-conditioned intermediate distributions slow optimization along low-variance directions even when the target velocity is exactly representable. More broadly, this identifies the geometry of the intermediate transport distributions as an important factor governing how effectively flow matching models can be optimized, beyond considerations of model expressiveness alone. This perspective gives FM-specific implications: small FM loss may hide directional velocity errors, scalar time reweighting cannot remove fixed-time directional stiffness, and preconditioning helps by reshaping the path covariance from $\Sigma_t$ to a better-conditioned $\tilde{\Sigma}_t$. Motivated by this analysis, we proposed a precondition-then-match framework that improves the geometry of the learning problem through a reversible transformation of the target distribution while preserving the standard downstream FM objective. Controlled and latent-space experiments support this mechanism through improved path conditioning, more balanced directional learning, and better sample quality across datasets. Taken together, our theoretical and empirical findings suggest that improving the representation in which transport is learned can offer meaningful advantages beyond simply increasing model capacity, highlighting preconditioning as a complementary direction for advancing flow-based generative modeling. Extending the analysis to the local geometry of nonlinear networks and developing adaptive preconditioners remain important directions for future work.

\clearpage
\newpage 

\begin{ack}
This work was supported by the Natural Sciences and Engineering Research Council of Canada (NSERC) through its Discovery Grants program and Canada Graduate Scholarships – Master's (CGS-M) program. ME acknowledges support from the the Israel Science Foundation through a personal grant (ISF 3062/26) and an
equipment grant (ISF 3066/26), as well as support from the Israeli Ministry of Innovation, Science \& Technology. For this work, the authors declare no relevant competing interests.
\end{ack}


\bibliography{biblio}
\bibliographystyle{plainnat}


\clearpage
\onecolumn
\appendix

\setlength{\floatsep}{6pt plus 1pt minus 1pt}
\setlength{\textfloatsep}{8pt plus 2pt minus 2pt}
\setlength{\intextsep}{6pt plus 1pt minus 1pt}

\makeatletter
\setlength{\@fptop}{0pt}
\setlength{\@fpsep}{6pt}
\setlength{\@fpbot}{0pt plus 1fil}
\makeatother

\captionsetup[figure]{skip=3pt}

\startcontents[appendix]

\section*{Appendix Contents}


\begingroup

\setlength{\parskip}{0pt}
\setlength{\columnsep}{1.5em}
\titlecontents{section}
    [0em]
    {\addvspace{2pt}\footnotesize\bfseries}
    {\contentslabel{1.8em}}
    {}
    {\hfill\contentspage}

\titlecontents{subsection}
    [1.2em]
    {\scriptsize}
    {\contentslabel{2.5em}}
    {}
    {\hfill\contentspage}

\printcontents[appendix]{}{1}{%
    \setcounter{tocdepth}{2}
}

\endgroup

\vspace{1em}


\section{Additional Related Work}
\label{app:related_work}
\paragraph{Generative models.} Generative models learn complex data distributions to synthesize realistic samples. Early explicit density models, such as variational autoencoders and autoregressive models, use likelihood-based objectives \citep{kingma2013auto}. Normalizing flows and diffusion models map simple priors to complex targets via invertible maps and stochastic dynamics, respectively \citep{dinh2016density,ho2020denoising}. TarFlow shows that normalizing flows are competitive for image synthesis \citep{zhai2025tarflow}. Flow matching learns a time-dependent velocity field for transport between distributions \citep{lipman2023flow}. To reduce errors from trajectory drift in regions poorly covered during training, Iterative Flow Matching learns endpoint corrections or gradually refines the transport \citep{haber2026iterative}. Conservative Flows reuse pretrained flow models to generate variations of existing samples via stochastic updates designed to preserve the data distribution or a smoothed version of it \citep{gal2026conservative}. Diffusion also supports conditional generation and domain-specific tasks in computer vision and medical imaging. \citet{lugmayr2022repaint} used pretrained diffusion for image inpainting, while \citet{chen2023diffusiondet} cast object detection as iterative bounding-box denoising. \citet{chung2022score} used score-based diffusion for accelerated MRI reconstruction, an ill-posed inverse problem, while \citet{ahamed2025igconda} used class-conditional counterfactual diffusion for weakly supervised anomaly detection in PET. Together, these applications highlight the versatility of diffusion-based generative modeling beyond conventional image synthesis, spanning restoration, recognition, reconstruction, and anomaly detection.

\paragraph{Preconditioning.} Preconditioning is a classical technique in numerical analysis aimed at improving the conditioning of optimization problems \citep{Ye2024Preconditioning} and differential equations \citep{Lazzarino2025PreconditionedNormalEquations}, thereby accelerating convergence and enhancing numerical stability \citep{saad2003iterative,trefethen1997numerical}. By rescaling variables or operators to reduce anisotropy and curvature disparities, preconditioning enables more efficient iterative solvers and more stable time integration \citep{Quirynen2019PRESAS}. Similar issues arise in machine learning, where poorly conditioned objectives can hinder optimization and lead to slow or unstable training dynamics \citep{Cau2024analysis}. Consequently, preconditioning has been widely adopted in large-scale learning systems, both implicitly through adaptive optimization methods and explicitly through problem-dependent transformations, yielding substantial improvements in convergence speed and robustness \citep{dauphin2014identifying}. In the context of diffusion models, recent work has shown that appropriate reparameterization and noise-dependent preconditioning can substantially improve optimization behavior and sample quality, highlighting the critical role of conditioning in generative model training \citep{karras2022elucidating_diffusion_design}.

\section{Limitation}
\label{app:limitation}
Our analysis isolates the conditioning mechanism in Gaussian and Gaussian-mixture flow matching settings, where the population targets and optimization dynamics can be characterized exactly. These models are intentionally simplified: they expose how anisotropic intermediate distributions can slow first-order optimization, but they do not provide a global convergence theory for arbitrary nonlinear neural networks, adaptive optimizers, or large-scale generative backbones. In practical image models, the relevant geometry is local, high-dimensional, and changes during training, so our condition-number diagnostics should be viewed as empirical evidence for the proposed mechanism rather than a complete spectral characterization of the full training problem. 

Preconditioning also introduces an additional training stage, with choices such as preconditioner architecture, capacity, training duration, and inverse accuracy. Our compute-matched experiments and ablations suggest that the gains are not explained merely by extra parameters, but the best cost--quality trade-off for the preconditioner remains task-dependent. Both standard and preconditioned flow matching use numerical ODE integration at sampling time, where each solver step introduces discretization error that can propagate along the trajectory. Flow-based preconditioning requires an additional inverse ODE solve, creating a further source of solver-dependent error. Finally, our experiments cover controlled distributions, latent MNIST, and several image datasets, but broader evaluation on larger-scale conditional generation, stronger production-level backbones, and standardized compute budgets would further clarify when preconditioning is most beneficial. Future work could also study time-dependent or locally adaptive preconditioners, which may better match the changing geometry of the flow matching path.

\section{Broader Impact}
\label{app:broader_impacts}
This work is primarily theoretical and methodological, focusing on the optimization geometry of flow matching. Potential positive impacts include more stable and efficient generative modeling methods, which may benefit scientific modeling, image synthesis, and representation learning. As with other advances in generative modeling, possible negative impacts include misuse for generating misleading or deceptive synthetic content. The paper does not target such applications, does not introduce a deployed system, and uses standard public datasets and controlled synthetic experiments.

\section{Proofs for Flow Matching as Stochastic Regression and Gaussian FM}
\label{app:gaussian_fm_proofs}

\subsection{Proof of \Cref{prop:cfm_regression_decomposition}}
\label{app:cfm_regression_decomposition_proof}

We prove the standard conditional-mean characterization of the conditional flow matching objective. Define
\begin{equation}
Y := v_t^\star(x_0,x_1), \qquad X := (x_t,t),
\end{equation}
so that the population objective can be written as
\begin{equation}
\mathcal{L}(g)=\mathbb{E}\big[\|g(X)-Y\|^2\big].
\end{equation}
Let
\begin{equation}
g^\star(X):=\mathbb{E}[Y\mid X].
\end{equation}
Then
\begin{equation}
Y = g^\star(X) + \bigl(Y-g^\star(X)\bigr),
\end{equation}
where by construction
\begin{equation}
\mathbb{E}\big[Y-g^\star(X)\mid X\big]=0.
\end{equation}
Hence,
\begin{align}
\mathcal{L}(g)
&= \mathbb{E}\big[\|g(X)-g^\star(X)+g^\star(X)-Y\|^2\big] \\
&= \mathbb{E}\big[\|g(X)-g^\star(X)\|^2\big]
   + \mathbb{E}\big[\|Y-g^\star(X)\|^2\big] \nonumber\\
&\quad + 2\,\mathbb{E}\big[\langle g(X)-g^\star(X),\, g^\star(X)-Y\rangle\big].
\end{align}
The cross term vanishes because conditioning on $X$ gives
\begin{align}
&\mathbb{E}\big[\langle g(X)-g^\star(X),\, g^\star(X)-Y\rangle\big] \\
&\qquad =
\mathbb{E}\Big[
\big\langle g(X)-g^\star(X),\,
\mathbb{E}[g^\star(X)-Y\mid X]
\big\rangle
\Big] = 0.
\end{align}
Therefore,
\begin{equation}
\mathcal{L}(g)
=
\mathbb{E}\big[\|g(X)-g^\star(X)\|^2\big]
+
\mathbb{E}\big[\|Y-g^\star(X)\|^2\big].
\end{equation}

Since
\begin{equation}
\mathbb{E}\big[\|Y-g^\star(X)\|^2\big]
=
\mathbb{E}\!\left[
\operatorname{Tr}\!\left(\operatorname{Cov}(Y\mid X)\right)
\right],
\end{equation}
we obtain
\begin{equation}
\mathcal{L}(g)
=
\mathbb{E}\big[\|g(x_t,t)-g^\star(x_t,t)\|^2\big]
+
\mathbb{E}\!\left[
\operatorname{Tr}\!\left(
\operatorname{Cov}\!\left(v_t^\star(x_0,x_1)\mid x_t,t\right)
\right)
\right],
\end{equation}
which proves \Cref{eq:regression_decomposition}. In particular,
\begin{equation}
\mathcal{L}(g)-\mathcal{L}(g^\star)
=
\mathbb{E}\big[\|g(x_t,t)-g^\star(x_t,t)\|^2\big],  
\end{equation}
and therefore $g^\star$ is the unique population minimizer up to almost-sure equivalence.

\subsection{Proof of \Cref{prop:gaussian_exact_target}}
\label{app:gaussian_exact_target_proof}
Assume $x_0 \sim \mathcal{N}(0,I), x_1 \sim \mathcal{N}(0,H), x_t=(1-t)x_0+t x_1$, with $x_0$ and $x_1$ independent, so $\mathbb{E}[x_0 x_1^\top] = \mathbb{E}[x_1 x_0^\top] = 0$.

\paragraph{Step 1: covariance of the intermediate state.}
Since $x_0$ and $x_1$ are independent and zero mean,
\begin{align}
\Sigma_t
&:= \mathbb{E}[x_t x_t^\top] \\
&= \mathbb{E}\!\left[\bigl((1-t)x_0+t x_1\bigr)\bigl((1-t)x_0+t x_1\bigr)^\top\right] \\
&= (1-t)^2 \mathbb{E}[x_0x_0^\top]
   + t^2 \mathbb{E}[x_1x_1^\top]
   + t(1-t)\mathbb{E}[x_0x_1^\top]
   + t(1-t)\mathbb{E}[x_1x_0^\top] \\
&= (1-t)^2 I + t^2 H.
\end{align}

\paragraph{Step 2: conditional mean target.}
Let $y := x_1-x_0$. Because $(y,x_t)$ is jointly Gaussian and zero mean, the conditional expectation is linear:
\begin{equation}
\mathbb{E}[y\mid x_t]
=
\operatorname{Cov}(y,x_t)\,\Sigma_t^{-1}x_t.
\end{equation}
It remains to compute the cross-covariance:
\begin{align}
\operatorname{Cov}(x_1-x_0,x_t)
&= \mathbb{E}\big[(x_1-x_0)\big((1-t)x_0+t x_1\big)^\top\big] \\
&= (1-t)\mathbb{E}[(x_1-x_0)x_0^\top]
   + t\mathbb{E}[(x_1-x_0)x_1^\top].
\end{align}
Using independence and zero mean,
\begin{equation}
\mathbb{E}[(x_1-x_0)x_0^\top] = -I,
\qquad
\mathbb{E}[(x_1-x_0)x_1^\top] = H.
\end{equation}
Hence,
\begin{equation}
\operatorname{Cov}(x_1-x_0,x_t)=tH-(1-t)I.
\end{equation}
Therefore,
\begin{equation}
g_t^\star(x_t)
=
\mathbb{E}[x_1-x_0\mid x_t]
=
A^\star(t)x_t,
\end{equation}
with
\begin{equation}
A^\star(t) = \bigl(tH-(1-t)I\bigr)\bigl((1-t)^2I+t^2H\bigr)^{-1}.
\end{equation}
This proves \Cref{prop:gaussian_exact_target}.

\subsection{Proof of \Cref{lem:gaussian_excess_risk}}
\label{app:gaussian_excess_risk_proof}

Fix $t\in[0,1]$ and define
\begin{equation}
\mathcal{L}_t(A)=\mathbb{E}\big[\|Ax_t-(x_1-x_0)\|^2\big].
\end{equation}
Let
\begin{equation}
y:=x_1-x_0,
\qquad
g_t^\star(x_t)=A^\star(t)x_t=\mathbb{E}[y\mid x_t].
\end{equation}
By \Cref{prop:cfm_regression_decomposition}, since $g_A(x_t)=Ax_t$ is a measurable predictor,
\begin{equation}
\mathcal{L}_t(A)
=
\mathbb{E}\big[\|Ax_t-A^\star(t)x_t\|^2\big]
+
\mathbb{E}\!\left[
\operatorname{Tr}\!\left(
\operatorname{Cov}(y\mid x_t)
\right)
\right].
\label{eq:appendix_risk_decomp_start}
\end{equation}
Because
\begin{equation}
Ax_t-A^\star(t)x_t = (A-A^\star(t))x_t,
\end{equation}
we have
\begin{align}
\mathbb{E}\big[\|(A-A^\star(t))x_t\|^2\big]
&= \mathbb{E}\big[x_t^\top (A-A^\star(t))^\top (A-A^\star(t))x_t\big] \\
&= \mathbb{E}\big[\operatorname{Tr} \big(x_t^\top (A-A^\star(t))^\top (A-A^\star(t))x_t\big)\big] \\
&= \operatorname{Tr}\!\Big(
(A-A^\star(t))^\top (A-A^\star(t))\,\mathbb{E}[x_t x_t^\top]
\Big) \\
&= \operatorname{Tr}\!\Big(
(A-A^\star(t))^\top (A-A^\star(t))\,\Sigma_t
\Big) \\
&= \operatorname{Tr}\!\Big(
(A-A^\star(t))\Sigma_t(A-A^\star(t))^\top
\Big),
\end{align}
where the last equality uses cyclic invariance of the trace. Substituting this into
\Cref{eq:appendix_risk_decomp_start} gives
\begin{equation}
\mathcal{L}_t(A)
=
\mathbb{E}\!\left[
\operatorname{Tr}\!\left(
\operatorname{Cov}(x_1-x_0\mid x_t)
\right)
\right]
+
\operatorname{Tr}\!\Big(
(A-A^\star(t))\Sigma_t(A-A^\star(t))^\top
\Big),
\end{equation}
which proves \Cref{eq:gaussian_risk_decomp}.

Since $\Sigma_t\succ 0$, the quadratic trace term is nonnegative and vanishes only when
$A=A^\star(t)$. Therefore $A^\star(t)$ is the unique minimizer. Evaluating the previous identity at
$A=A^\star(t)$ gives
\begin{equation}
\mathcal{L}_t(A^\star(t))
=
\mathbb{E}\!\left[
\operatorname{Tr}\!\left(
\operatorname{Cov}(x_1-x_0\mid x_t)
\right)
\right].
\end{equation}
Subtracting this identity from the expression for $\mathcal{L}_t(A)$ yields
\begin{equation}
\mathcal{L}_t(A)-\mathcal{L}_t(A^\star(t))
=
\operatorname{Tr}\!\Big(
(A-A^\star(t))\Sigma_t(A-A^\star(t))^\top
\Big),
\end{equation}
which proves \Cref{eq:excess_risk_trace} and \Cref{eq:gaussian_excess_risk_only}.

\subsection{Proof of \Cref{thm:gd_gaussian}}
\label{app:gd_gaussian_proof}

Recall the population loss
\begin{equation}
\mathcal{L}_t(A)=\mathbb{E}\big[\|Ax_t-(x_1-x_0)\|^2\big].
\end{equation}
Let
\begin{equation}
C_t := \operatorname{Cov}(x_1-x_0,x_t)=tH-(1-t)I.
\end{equation}
Using the standard matrix derivative identity,
\begin{equation}
\nabla_A \mathcal{L}_t(A) = 2(A\Sigma_t-C_t).
\end{equation}
Therefore full-batch gradient descent with step size $\eta$ is
\begin{equation}
A_{k+1}=A_k-2\eta(A_k\Sigma_t-C_t).
\end{equation}
Let
\begin{equation}
E_k:=A_k-A^\star(t),
\end{equation}
where $A^\star(t)\Sigma_t=C_t$. Then
\begin{align}
E_{k+1}
&= A_{k+1}-A^\star(t) \\
&= A_k-2\eta(A_k\Sigma_t-C_t)-A^\star(t) \\
&= A_k-A^\star(t)-2\eta(A_k-A^\star(t))\Sigma_t \\
&= E_k(I-2\eta\Sigma_t).
\end{align}
This proves \Cref{eq:error_recursion_gd_refined}. Hence, for a fixed step size $\eta$, the error recursion is
\begin{equation}
E_{k+1}=E_k(I-2\eta\Sigma_t).
\end{equation}
Since $\Sigma_t$ is symmetric positive definite, let
\begin{equation}
\Sigma_t = U\operatorname{diag}(\sigma_1(t),\ldots,\sigma_d(t))U^\top,
\end{equation}
where
\begin{equation}
0<\sigma_{\min}(t)\leq \sigma_i(t)\leq \sigma_{\max}(t).
\end{equation}
Then
\begin{equation}
I-2\eta\Sigma_t
=
U\operatorname{diag}\bigl(1-2\eta\sigma_1(t),\ldots,
1-2\eta\sigma_d(t)\bigr)U^\top.
\end{equation}
Thus each eigendirection contracts independently by a factor
\begin{equation}
|1-2\eta\sigma_i(t)|.
\end{equation}
The worst-case contraction factor for a given $\eta$ is therefore
\begin{equation}
\rho(\eta)
=
\max_i |1-2\eta\sigma_i(t)|
=
\max\Bigl\{
|1-2\eta\sigma_{\min}(t)|,\,
|1-2\eta\sigma_{\max}(t)|
\Bigr\}.
\end{equation}
Convergence requires $\rho(\eta)<1$, which is equivalent to
\begin{equation}
0<\eta<\frac{1}{\sigma_{\max}(t)}.
\end{equation}

The optimal fixed step size minimizes the worst-case contraction factor
\begin{equation}
\rho(\eta)
=
\max\Bigl\{
|1-2\eta\sigma_{\min}(t)|,\,
|1-2\eta\sigma_{\max}(t)|
\Bigr\}.
\end{equation}
At the optimum, the two endpoint contractions are balanced:
\begin{equation}
1-2\eta^\star\sigma_{\min}(t)
=
-\bigl(1-2\eta^\star\sigma_{\max}(t)\bigr).
\end{equation}
Solving gives
\begin{equation}
1-2\eta^\star\sigma_{\min}(t)
=
-1+2\eta^\star\sigma_{\max}(t),
\end{equation}
and hence
\begin{equation}
2
=
2\eta^\star\bigl(\sigma_{\max}(t)+\sigma_{\min}(t)\bigr).
\end{equation}
Therefore,
\begin{equation}
\eta^\star
=
\frac{1}{\sigma_{\max}(t)+\sigma_{\min}(t)}.
\end{equation}

Substituting this step size into the contraction factor gives
\begin{equation}
\rho^\star
=
1-2\eta^\star\sigma_{\min}(t)
=
1-\frac{2\sigma_{\min}(t)}
{\sigma_{\max}(t)+\sigma_{\min}(t)}.
\end{equation}
Thus
\begin{equation}
\rho^\star
=
\frac{\sigma_{\max}(t)-\sigma_{\min}(t)}
{\sigma_{\max}(t)+\sigma_{\min}(t)}.
\end{equation}
Equivalently, using
\begin{equation}
\kappa(\Sigma_t)
=
\frac{\sigma_{\max}(t)}{\sigma_{\min}(t)},
\end{equation}
we obtain
\begin{equation}
\rho^\star
=
\frac{
\frac{\sigma_{\max}(t)}{\sigma_{\min}(t)}-1
}{
\frac{\sigma_{\max}(t)}{\sigma_{\min}(t)}+1
}
=
\frac{\kappa(\Sigma_t)-1}{\kappa(\Sigma_t)+1}.
\end{equation}

After $k$ gradient-descent steps,
\begin{equation}
\|E_k\|
\leq
(\rho^\star)^k\|E_0\|.
\end{equation}
To reduce the error by a factor $\varepsilon\in(0,1)$, it is sufficient that
\begin{equation}
(\rho^\star)^k \leq \varepsilon.
\end{equation}
Taking logarithms gives
\begin{equation}
k\log(\rho^\star)
\leq
\log(\varepsilon).
\end{equation}
Since $\rho^\star\in(0,1)$, we have $\log(\rho^\star)<0$, and therefore
\begin{equation}
k
\geq
\frac{\log(1/\varepsilon)}
{-\log(\rho^\star)}.
\end{equation}

Now
\begin{equation}
\rho^\star
=
\frac{\kappa(\Sigma_t)-1}{\kappa(\Sigma_t)+1}
=
1-\frac{2}{\kappa(\Sigma_t)+1}.
\end{equation}
Using the standard bound $-\log(1-a)\ge a$ for $a\in(0,1)$, with
\begin{equation}
a=\frac{2}{\kappa(\Sigma_t)+1},
\end{equation}
we obtain
\begin{equation}
-\log(\rho^\star)
=
-\log\!\left(1-\frac{2}{\kappa(\Sigma_t)+1}\right)
\ge
\frac{2}{\kappa(\Sigma_t)+1}.
\end{equation}
Hence
\begin{equation}
\frac{\log(1/\varepsilon)}{-\log(\rho^\star)}
\le
\frac{\kappa(\Sigma_t)+1}{2}\log(1/\varepsilon).
\end{equation}
Therefore, it is sufficient to choose
\begin{equation}
k
\ge
\frac{\kappa(\Sigma_t)+1}{2}\log(1/\varepsilon).
\end{equation}
Thus the number of iterations required to reduce the error by a factor $\varepsilon$ scales as
\begin{equation}
k=\mathcal{O}\!\left(\kappa(\Sigma_t)\log(1/\varepsilon)\right).
\end{equation}
This proves \Cref{thm:gd_gaussian}.

\subsection{Proof of \Cref{cor:low_variance_bottleneck}}
\label{app:low_variance_bottleneck_proof}

By \Cref{thm:gd_gaussian}, the GD error evolves modewise according to the spectrum of $\Sigma_t$. The slowest mode corresponds to the smallest eigenvalue $\sigma_{\min}(t)$, so directions associated with small $\sigma_i(t)$ converge most slowly. If $H$ is ill-conditioned and $t$ is not close to zero, then
\begin{equation}
\sigma_i(t)=(1-t)^2+t^2\lambda_i
\end{equation}
inherits the disparity in the eigenvalues of $H$, and therefore the optimization becomes dominated by the low-variance directions of the intermediate distribution. This proves the corollary.

\begin{figure}[h]
    \centering
    \includegraphics[width=0.95\linewidth]{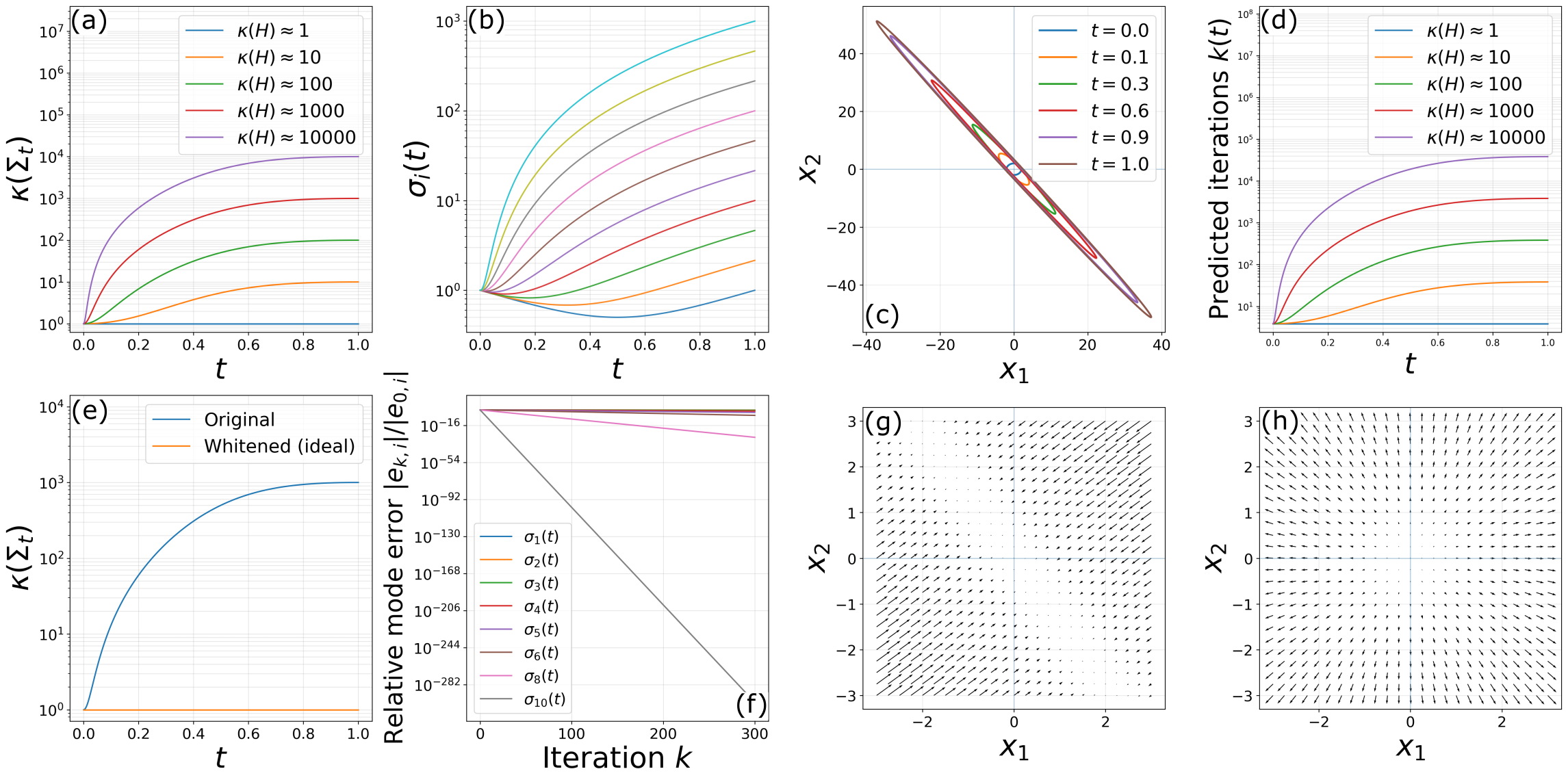}
    \caption{\textbf{Gaussian model illustrating conditioning and optimization effects in flow matching:} \textbf{(a)} Growth of the condition number $\kappa(\Sigma_t$) along the linear interpolation $\Sigma_t = (1-t)^2 I + t^2 H$ showing how intermediate distributions $p_t$ gradually inherit the ill-conditioning of $H$; \textbf{(b)} evolution of the eigenvalues $\sigma_i(t)$ of $\Sigma_t$ for $\kappa(H) = 1000$, highlighting the increasing spectral disparity that creates slow and fast directions; \textbf{(c)} Covariance ellipses of $p_t = \mathcal{N}(0, \Sigma_t)$ at selected times, visualizing the transition from an isotropic distribution to a highly anisotropic one; \textbf{(d)} Predicted number of gradient-descent iterations $k(t)$ required to reach a fixed error tolerance, demonstrating the sharp slowdown at large $t$ due to ill-conditioning; \textbf{(e)} Effect of ideal preconditioning (whitening), which removes conditioning growth and keeps $\kappa(\Sigma_t) \approx 1$ for all $t$; \textbf{(f)} Per-mode gradient-descent error decay at $t=0.9$, showing that large-variance modes converge rapidly while small-variance modes decay extremely slowly; \textbf{(g)} Optimal score field $s^*(x, t) = -\Sigma_t^{-1} x$, exhibiting strong anisotropy and stiffness along low-variance directions; \textbf{(h)} Optimal flow matching velocity field $v^\star(x, t) = A^\star(t) x$, which is smoother than the score but still inherits anisotropy from $\Sigma_t$.}
    \label{fig:gaussian_model_all_plots}
\end{figure}

\subsection{Proof of \Cref{cor:sgd_gaussian}}
\label{app:sgd_gaussian_proof}

We derive the stochastic-gradient dynamics for the Gaussian FM regression problem at a fixed time $t$. 
For a single sample $(x_0^{(m)},x_1^{(m)})$, define
\begin{equation}
x_t^{(m)}=(1-t)x_0^{(m)}+t x_1^{(m)},
\qquad
y^{(m)}=x_1^{(m)}-x_0^{(m)}.
\end{equation}
The single-sample loss is
\begin{equation}
\ell_m(A)=\|Ax_t^{(m)}-y^{(m)}\|^2,
\end{equation}
with gradient
\begin{equation}
\nabla_A \ell_m(A)
=
2(Ax_t^{(m)}-y^{(m)})(x_t^{(m)})^\top.
\end{equation}
Thus SGD with step size $\eta$ gives
\begin{equation}
A_{m+1}
=
A_m
-
2\eta(A_mx_t^{(m)}-y^{(m)})(x_t^{(m)})^\top.
\end{equation}

Let $A^\star=A^\star(t)$ denote the population minimizer and define
\begin{equation}
E_m=A_m-A^\star.
\end{equation}
Then
\begin{align}
E_{m+1}
&=
A_{m+1}-A^\star \\
&=
E_m
-
2\eta E_m x_t^{(m)}(x_t^{(m)})^\top
-
2\eta(A^\star x_t^{(m)}-y^{(m)})(x_t^{(m)})^\top .
\end{align}
Taking conditional expectation given $E_m$ gives
\begin{align}
\mathbb{E}[E_{m+1}\mid E_m]
&=
E_m
-
2\eta E_m \mathbb{E}[x_t x_t^\top]
-
2\eta\mathbb{E}\big[(A^\star x_t-y)x_t^\top\big].
\end{align}
Using
\begin{equation}
\mathbb{E}[x_t x_t^\top]=\Sigma_t
\end{equation}
and the population normal equation
\begin{equation}
A^\star\Sigma_t
=
\operatorname{Cov}(y,x_t)
=
\mathbb{E}[y x_t^\top],
\end{equation}
we obtain
\begin{equation}
\mathbb{E}\big[(A^\star x_t-y)x_t^\top\big]
=
A^\star\Sigma_t-\mathbb{E}[yx_t^\top]
=
0.
\end{equation}
Therefore,
\begin{equation}
\mathbb{E}[E_{m+1}\mid E_m]
=
E_m(I-2\eta\Sigma_t).
\end{equation}
Thus the conditional mean SGD error follows the same contraction matrix as full-batch gradient descent.

Now diagonalize
\begin{equation}
\Sigma_t
=
U\operatorname{diag}(\sigma_1(t),\ldots,\sigma_d(t))U^\top.
\end{equation}
Define the error in the eigenbasis of $\Sigma_t$ by
\begin{equation}
\widetilde{E}_m := E_m U.
\end{equation}
Multiplying the mean recursion on the right by $U$ gives
\begin{equation}
\mathbb{E}[\widetilde{E}_{m+1}\mid E_m]
=
\widetilde{E}_m
\operatorname{diag}
\bigl(
1-2\eta\sigma_1(t),\ldots,1-2\eta\sigma_d(t)
\bigr).
\end{equation}
Hence, for any output coordinate $k$ and eigendirection $i$, the scalar component
\begin{equation}
e_{k i,m}:=(\widetilde{E}_m)_{k i}
\end{equation}
satisfies
\begin{equation}
\mathbb{E}[e_{k i,m+1}\mid E_m]
=
(1-2\eta\sigma_i(t))e_{k i,m}.
\end{equation}
Equivalently, the stochastic recursion can be written as
\begin{equation}
e_{k i,m+1}
=
(1-2\eta\sigma_i(t))e_{k i,m}
+
\zeta_{k i,m},
\end{equation}
where $\zeta_{k i,m}$ is a zero-mean stochastic-gradient noise term satisfying
\begin{equation}
\mathbb{E}[\zeta_{k i,m}\mid E_m]=0.
\end{equation}

Therefore, directions with smaller $\sigma_i(t)$ contract more slowly in the mean SGD dynamics, matching the full-batch GD behavior. In particular, when $\Sigma_t$ is ill-conditioned, the mean dynamics of SGD also make slow progress along low-variance eigendirections. This proves \Cref{cor:sgd_gaussian}.

The preceding Gaussian proofs isolate the same spectral mechanism from several angles: 
the covariance $\Sigma_t$ becomes increasingly anisotropic with $t$, the GD contraction factors become mode-dependent, and low-variance eigendirections converge slowly. \Cref{fig:gaussian_model_all_plots} visualizes these consequences, together with the corresponding score and flow fields, and the effect of ideal whitening.

\section{Proofs for Gaussian-Mixture Flow Matching}
\label{app:gmm_fm_proofs}

\subsection{Proof of \Cref{prop:gmm_conditional_mean}}
\label{app:gmm_conditional_mean_proof}

Assume
\begin{equation}
x_1 \sim \sum_{k=1}^K \pi_k\,\mathcal{N}(0,H_k),
\qquad
x_0\sim\mathcal{N}(0,I),
\qquad
x_t=(1-t)x_0+t x_1.
\end{equation}
Let $Z\in\{1,\dots,K\}$ denote the latent mixture component of $x_1$, with
\begin{equation}
\mathbb{P}(Z=k)=\pi_k.
\end{equation}
Then by the law of total expectation,
\begin{align}
g_t^\star(x_t)
&= \mathbb{E}[x_1-x_0\mid x_t] \\
&= \sum_{k=1}^K \mathbb{P}(Z=k\mid x_t,t)\,
   \mathbb{E}[x_1-x_0\mid x_t, Z=k, t].
\end{align}
Define
\begin{equation}
w_k(x_t,t):=\mathbb{P}(Z=k\mid x_t,t).
\end{equation}
Conditioned on $Z=k$, the pair $(x_0,x_1)$ is Gaussian with
\begin{equation}
x_0\sim\mathcal{N}(0,I),
\qquad
x_1\sim\mathcal{N}(0,H_k),
\qquad
x_t=(1-t)x_0+t x_1.
\end{equation}
Therefore, by the single-Gaussian result already proved in \Cref{app:gaussian_exact_target_proof},
\begin{equation}
\mathbb{E}[x_1-x_0\mid x_t, Z=k, t]
=
A_k^\star(t)x_t,
\end{equation}
where
\begin{equation}
A_k^\star(t)
=
\bigl(tH_k-(1-t)I\bigr)\bigl((1-t)^2I+t^2H_k\bigr)^{-1}.
\end{equation}
Substituting into the law-of-total-expectation formula yields
\begin{equation}
g_t^\star(x_t)
=
\sum_{k=1}^K w_k(x_t,t)\,A_k^\star(t)\,x_t,
\end{equation}
which proves \Cref{prop:gmm_conditional_mean}.

\subsection{Proof of \Cref{thm:gmm_bottleneck}}
\label{app:gmm_bottleneck_proof}

Under \Cref{ass:gmm_perfect_gating}, the mixture responsibilities are known exactly, so the Gaussian-mixture FM regression problem decomposes into independent componentwise Gaussian regression subproblems. For component $k$, the intermediate covariance is
\begin{equation}
\Sigma_{t,k}=(1-t)^2I+t^2H_k.
\end{equation}
If $\lambda_{k,i}$ is the $i$th eigenvalue of $H_k$, then the corresponding eigenvalues of $\Sigma_{t,k}$ are
\begin{equation}
\sigma_{k,i}(t)=(1-t)^2+t^2\lambda_{k,i}.
\end{equation}

For each component $k$, the same argument as in the Gaussian case applies. In particular, full-batch GD applied to the componentwise quadratic loss has error recursion
\begin{equation}
E_{k,m+1}
=
E_{k,m}(I-2\eta\Sigma_{t,k}),
\end{equation}
so along eigendirection $i$,
\begin{equation}
e_{k,i,m+1}
=
(1-2\eta\sigma_{k,i}(t))e_{k,i,m}.
\end{equation}
Thus directions with smaller $\sigma_{k,i}(t)$ contract more slowly.

Similarly, for single-sample SGD, the conditional mean error obeys
\begin{equation}
\mathbb{E}[E_{k,m+1}\mid E_{k,m}]
=
E_{k,m}(I-2\eta\Sigma_{t,k}),
\end{equation}
and therefore each scalar mode satisfies
\begin{equation}
\mathbb{E}[e_{k,i,m+1}\mid E_{k,m}]
=
(1-2\eta\sigma_{k,i}(t))e_{k,i,m}.
\end{equation}
Thus the mean SGD dynamics inherit the same componentwise mode-dependent contraction as full-batch GD.

Consequently, if one seeks comparable convergence across all mixture components and eigendirections, the overall rate is limited by the slowest componentwise mode. This mode is determined by
\begin{equation}
\min_{k,i}\sigma_{k,i}(t).
\end{equation}
Therefore, even if most mixture components are well-conditioned, a single component with a small eigenvalue can dominate the global optimization time scale. This proves \Cref{thm:gmm_bottleneck}.

\subsection{Proof of \Cref{cor:gmm_multimodal_takeaway}}
\label{app:gmm_multimodal_takeaway_proof}

The corollary follows immediately from \Cref{thm:gmm_bottleneck}. Since the optimization rate is controlled by the smallest eigenvalue across all components, multimodality does not average away the bottleneck; instead, the hardest component determines the overall difficulty.

\subsection{Extension to nonzero-mean Gaussian and Gaussian-mixture FM}
\label{app:nonzero_mean_gaussian_and_gmm}

\begin{corollary}[Extension to nonzero-mean Gaussian and Gaussian-mixture FM case]
\label{cor:nonzero_mean}
Let $x_0\sim\mathcal{N}(0,I)$ and $x_1\sim\mathcal{N}(\mu,H)$ be independent,
with $H\succ0$, and let $x_t=(1-t)x_0+tx_1$. Define
\[
m_t=t\mu,\qquad
\Sigma_t=(1-t)^2I+t^2H,\qquad
C_t=tH-(1-t)I.
\]
Then the population conditional FM target is affine:
\[
g_t^\star(x)
=\mathbb{E}[x_1-x_0\mid x_t=x]
=A^\star(t)x+b^\star(t),
\]
where
\[
A^\star(t)=C_t\Sigma_t^{-1},
\qquad
b^\star(t)=\mu-A^\star(t)m_t
=(1-t)\Sigma_t^{-1}\mu.
\]

For
\[
L_t(A,b)
=\mathbb{E}\!\left[\|Ax_t+b-(x_1-x_0)\|_2^2\right],
\]
the excess risk is
\[
\begin{aligned}
L_t(A,b)-L_t(A^\star,b^\star)
&=\operatorname{Tr}\!\left(
(A-A^\star)\Sigma_t(A-A^\star)^\top
\right)\\
&\quad+
\left\|(A-A^\star)m_t+(b-b^\star)\right\|_2^2.
\end{aligned}
\]
Equivalently, parameterizing the centered predictor as
$A(x_t-m_t)+\beta$, with $\beta^\star=\mu$, gives
\[
\widetilde L_t(A,\beta)-\widetilde L_t(A^\star,\mu)
=
\operatorname{Tr}\!\left(
(A-A^\star)\Sigma_t(A-A^\star)^\top
\right)
+\|\beta-\mu\|_2^2.
\]
Thus the slope and intercept decouple. Gradient descent on the slope satisfies
\[
E_{\ell+1}=E_\ell(I-2\eta\Sigma_t),
\qquad
E_\ell=A_\ell-A^\star(t),
\]
while the intercept error contracts by $1-2\eta$. Hence, after centering or
profiling out the intercept, nonzero means do not alter the
covariance-controlled directional bottleneck.

More generally, if
\[
x_1\mid Z=k\sim\mathcal{N}(\mu_k,H_k),
\]
then
\[
\Sigma_{t,k}=(1-t)^2I+t^2H_k
\]
is independent of $\mu_k$, and
\[
\mathbb{E}[x_1-x_0\mid x_t=x,Z=k]
=A_k^\star(t)x+b_k^\star(t),
\]
where
\[
A_k^\star(t)
=\bigl(tH_k-(1-t)I\bigr)\Sigma_{t,k}^{-1},
\qquad
b_k^\star(t)
=(1-t)\Sigma_{t,k}^{-1}\mu_k.
\]
Therefore, under Assumption~3.7, the componentwise rates and the
worst-component result of Theorem~3.9 remain unchanged with nonzero component
means.
\end{corollary}

\begin{remark}[Effect of mean shifts]
\label{rem:nonzero_mean}
In the uncentered $(A,b)$ parameterization, the Hessian is governed by
\[
M_t=
\begin{pmatrix}
\Sigma_t+m_tm_t^\top & m_t\\
m_t^\top & 1
\end{pmatrix},
\]
so a large mean can worsen raw optimization conditioning. This is a removable
parameterization effect: explicit centering eliminates it exactly, whereas
the directional stiffness caused by $\kappa(\Sigma_t)$ remains. For mixtures,
distinct means affect the posterior responsibilities and mode locations, but
not the within-component covariances. Moreover, increasing separation makes
the perfect-assignment approximation in Assumption~3.7 more accurate.
\end{remark}

\begin{proof}
Let $y=x_1-x_0$. The pair $(y,x_t)$ is jointly Gaussian, with
\[
\mathbb{E}[y]=\mu,\qquad
\mathbb{E}[x_t]=m_t,\qquad
\operatorname{Cov}(x_t)=\Sigma_t,\qquad
\operatorname{Cov}(y,x_t)=C_t.
\]
Therefore,
\[
\mathbb{E}[y\mid x_t=x]
=\mu+C_t\Sigma_t^{-1}(x-m_t),
\]
which gives $A^\star=C_t\Sigma_t^{-1}$ and
$b^\star=\mu-A^\star m_t$. Since
\[
\Sigma_t-tC_t=(1-t)I,
\]
we also obtain $b^\star=(1-t)\Sigma_t^{-1}\mu$.

Let $\Delta A=A-A^\star$, $\Delta b=b-b^\star$, and
$z_t=x_t-m_t$. By the conditional-mean decomposition,
\[
L_t(A,b)-L_t(A^\star,b^\star)
=
\mathbb{E}\!\left[
\|\Delta A x_t+\Delta b\|_2^2
\right].
\]
Using $\mathbb{E}[z_t]=0$ and
$\operatorname{Cov}(z_t)=\Sigma_t$ yields
\[
L_t(A,b)-L_t(A^\star,b^\star)
=
\operatorname{Tr}(\Delta A\Sigma_t\Delta A^\top)
+\|\Delta A m_t+\Delta b\|_2^2.
\]
The centered form and gradient-descent recursion follow immediately.
Conditioning on $Z=k$ gives the same identities with
$(\mu,H,\Sigma_t)$ replaced by $(\mu_k,H_k,\Sigma_{t,k})$. Since
$\Sigma_{t,k}$ is independent of $\mu_k$, Theorem~3.9 continues to apply.
\end{proof}

\section{Proofs for Preconditioning in Flow Matching}
\label{app:preconditioning_fm_proofs}

\subsection{Proof of \Cref{thm:fm_preconditioning}}
\label{app:fm_preconditioning_proof}

\begin{equation}
x_0\sim\mathcal{N}(0,I),
\qquad
x_1\sim\mathcal{N}(0,H),
\qquad
x_t=(1-t)x_0+t x_1.
\end{equation}
Let $\mathcal{P}\in\mathbb{R}^{d\times d}$ be invertible and define
\begin{equation}
\tilde{x}_1=\mathcal{P}x_1,
\qquad
\tilde{x}_t=(1-t)x_0+t\tilde{x}_1.
\end{equation}

\paragraph{Step 1: transformed covariance.}
Since $x_0$ and $x_1$ are independent and zero mean,
\begin{align}
\tilde{\Sigma}_t
&:= \mathbb{E}[\tilde{x}_t\tilde{x}_t^\top] \\
&= \mathbb{E}\!\left[\bigl((1-t)x_0+t\mathcal{P}x_1\bigr)\bigl((1-t)x_0+t\mathcal{P}x_1\bigr)^\top\right] \\
&= (1-t)^2 \mathbb{E}[x_0x_0^\top]
   + t^2 \mathbb{E}[\mathcal{P}x_1x_1^\top \mathcal{P}^\top] \\
&\quad + t(1-t)\mathbb{E}[x_0x_1^\top]\mathcal{P}^\top
   + t(1-t)\mathcal{P}\mathbb{E}[x_1x_0^\top] \\
&= (1-t)^2I+t^2\mathcal{P}H\mathcal{P}^\top.
\end{align}

\paragraph{Step 2: exact whitening.}
If $\mathcal{P}=H^{-1/2}$, then
\begin{equation}
\mathcal{P}H\mathcal{P}^\top = I,
\end{equation}
and therefore
\begin{equation}
\tilde{\Sigma}_t = (1-t)^2I+t^2I = \bigl((1-t)^2+t^2\bigr)I.
\end{equation}
Hence every eigenvalue of $\tilde{\Sigma}_t$ is equal to $(1-t)^2+t^2$, so
\begin{equation}
\kappa(\tilde{\Sigma}_t)=1
\qquad \text{for all } t\in[0,1].
\end{equation}

\paragraph{Step 3: approximate isotropization.}
Assume
\begin{equation}
mI \preceq \mathcal{P}H\mathcal{P}^\top \preceq MI
\end{equation}
for some $0<m\le M$. Then
\begin{equation}
(1-t)^2I+t^2 mI
\preceq
\tilde{\Sigma}_t
\preceq
(1-t)^2I+t^2 MI.
\end{equation}
Therefore every eigenvalue of $\tilde{\Sigma}_t$ satisfies
\begin{equation}
(1-t)^2+t^2m
\le
\lambda_i(\tilde{\Sigma}_t)
\le
(1-t)^2+t^2M,
\end{equation}
and so
\begin{equation}
\kappa(\tilde{\Sigma}_t)
\le
\frac{(1-t)^2+t^2M}{(1-t)^2+t^2m}.
\end{equation}
This proves \Cref{thm:fm_preconditioning}.

\subsection{Additional note on Gaussian-mixture preconditioning}
\label{app:gmm_preconditioning_note}

The same idea extends componentwise to the Gaussian-mixture setting. If one could apply a componentwise whitening map
\begin{equation}
T_k = H_k^{-1/2}
\end{equation}
to each mixture component, then the transformed component covariances would satisfy
\begin{equation}
\tilde{\Sigma}_{t,k} = \bigl((1-t)^2+t^2\bigr)I,
\end{equation}
so all components would become perfectly conditioned and no single mode could dominate optimization. In the main paper, however, Section~4 focuses on the single-Gaussian FM theorem because it is the cleanest exact statement. The practical learned preconditioners in the experiments should be viewed as approximate nonlinear analogues of this isotropization principle.

\numberwithin{theorem}{subsection}

\subsection{FM-specific consequences of the conditioning analysis}
\label{app:fm_specific_implications}

We derive several consequences of the Gaussian FM analysis that are specific to the flow matching
objective and its time-indexed regression geometry.

\begin{corollary}[Small FM loss can hide directional velocity error]
\label{cor:small_loss_hides_velocity_error}
Under the setting of \Cref{lem:gaussian_excess_risk}, let $H=U\Lambda U^\top$, and let $u_i$
denote the $i^{\mathrm{th}}$ column of $U$. Since $\Sigma_t$ and $H$ share eigenvectors, with
eigenvalues $\sigma_i(t)=(1-t)^2+t^2\lambda_i$, the fixed-time excess risk decomposes as
\begin{equation}
\mathcal{L}_t(A)-\mathcal{L}_t(A^\star(t))
=
\sum_{i=1}^d
\sigma_i(t)
\bigl\|(A-A^\star(t))u_i\bigr\|^2 .
\label{eq:app_directional_excess_risk}
\end{equation}
Consequently, if
$\mathcal{L}_t(A)-\mathcal{L}_t(A^\star(t))\le \varepsilon$, then for every eigendirection $u_i$,
\begin{equation}
\bigl\|(A-A^\star(t))u_i\bigr\|^2
\le
\frac{\varepsilon}{\sigma_i(t)} .
\label{eq:app_directional_velocity_bound}
\end{equation}
Thus, in low-variance directions where $\sigma_i(t)$ is small, a small FM excess loss can still
permit large velocity-field error.
\end{corollary}

\begin{proof}
By \Cref{eq:gaussian_excess_risk_only}, the fixed-time excess risk is
\begin{equation}
\mathcal{L}_t(A)-\mathcal{L}_t(A^\star(t))
=
\operatorname{Tr}\!\Big((A-A^\star(t))\Sigma_t(A-A^\star(t))^\top\Big).
\end{equation}
Let
\begin{equation}
\Delta_t := A-A^\star(t).
\end{equation}
Since $H=U\Lambda U^\top$ and $\Sigma_t=(1-t)^2I+t^2H$, the matrices $H$ and $\Sigma_t$ share the same eigenvectors. Hence
\begin{equation}
\Sigma_t
=
U\operatorname{diag}(\sigma_1(t),\dots,\sigma_d(t))U^\top,
\qquad
\sigma_i(t)=(1-t)^2+t^2\lambda_i.
\end{equation}
Substituting this eigendecomposition into the excess-risk expression gives
\begin{equation}
\operatorname{Tr}\!\big(\Delta_t\Sigma_t\Delta_t^\top\big)
=
\operatorname{Tr}\!\Big(
\Delta_t U
\operatorname{diag}(\sigma_1(t),\dots,\sigma_d(t))
U^\top \Delta_t^\top
\Big).
\end{equation}
Let $u_i$ denote the $i^{\mathrm{th}}$ column of $U$. Since
\begin{equation}
U\operatorname{diag}(\sigma_1(t),\dots,\sigma_d(t))U^\top
=
\sum_{i=1}^d \sigma_i(t) u_i u_i^\top,
\end{equation}
we obtain
\begin{equation}
\operatorname{Tr}\!\big(\Delta_t\Sigma_t\Delta_t^\top\big)
=
\sum_{i=1}^d
\sigma_i(t)
\operatorname{Tr}\!\big(\Delta_t u_i u_i^\top \Delta_t^\top\big).
\end{equation}
For each $i$,
\begin{equation}
\operatorname{Tr}\!\big(\Delta_t u_i u_i^\top \Delta_t^\top\big)
=
\operatorname{Tr}\!\big((\Delta_t u_i)(\Delta_t u_i)^\top\big)
=
\|\Delta_t u_i\|^2.
\end{equation}
Therefore,
\begin{equation}
\mathcal{L}_t(A)-\mathcal{L}_t(A^\star(t))
=
\sum_{i=1}^d
\sigma_i(t)
\bigl\|(A-A^\star(t))u_i\bigr\|^2,
\end{equation}
which proves \Cref{eq:app_directional_excess_risk}.

Now suppose
\begin{equation}
\mathcal{L}_t(A)-\mathcal{L}_t(A^\star(t))\le \varepsilon.
\end{equation}
Since $\Sigma_t$ is positive definite, each eigenvalue satisfies $\sigma_i(t)>0$, and every term
\begin{equation}
\sigma_i(t)\bigl\|(A-A^\star(t))u_i\bigr\|^2
\end{equation}
in the sum is nonnegative. Hence no single term can exceed the whole sum. Therefore, for every $i$,
\begin{equation}
\sigma_i(t)\bigl\|(A-A^\star(t))u_i\bigr\|^2
\le
\sum_{j=1}^d
\sigma_j(t)
\bigl\|(A-A^\star(t))u_j\bigr\|^2
=
\mathcal{L}_t(A)-\mathcal{L}_t(A^\star(t))
\le
\varepsilon.
\end{equation}
Dividing by $\sigma_i(t)>0$ gives
\begin{equation}
\bigl\|(A-A^\star(t))u_i\bigr\|^2
\le
\frac{\varepsilon}{\sigma_i(t)}.
\end{equation}
This proves \Cref{eq:app_directional_velocity_bound}.
\end{proof}

\begin{corollary}[Scalar time reweighting does not remove fixed-time directional stiffness]
\label{cor:time_reweighting_no_directional_fix}
Fix $t\in[0,1]$ and let $\alpha(t)>0$. Consider the scalar-reweighted fixed-time objective
\begin{equation}
\mathcal{L}_{t,\alpha}(A)
=
\alpha(t)\,
\mathbb{E}\bigl[\|Ax_t-(x_1-x_0)\|^2\bigr].
\label{eq:weighted_fixed_time_objective}
\end{equation}
Then the curvature of this objective with respect to $A$ is scaled by $\alpha(t)$, but its directional
condition number is unchanged. In particular, the convergence of full-batch gradient descent remains
controlled by $\kappa(\Sigma_t)$ at fixed $t$.
\end{corollary}

\begin{proof}
Fix $t\in[0,1]$. Since $\alpha(t)>0$ is a scalar that does not depend on $A$, the reweighted
objective is simply
\begin{equation}
\mathcal{L}_{t,\alpha}(A)
=
\alpha(t)\mathcal{L}_t(A),
\qquad
\mathcal{L}_t(A)
=
\mathbb{E}\bigl[\|Ax_t-(x_1-x_0)\|^2\bigr].
\end{equation}
From the Gaussian fixed-time regression problem, the gradient of the unweighted objective is
\begin{equation}
\nabla_A \mathcal{L}_t(A)
=
2(A\Sigma_t-C_t),
\end{equation}
where
\begin{equation}
\Sigma_t=\mathbb{E}[x_tx_t^\top],
\qquad
C_t=\operatorname{Cov}(x_1-x_0,x_t).
\end{equation}
Therefore the gradient of the reweighted objective is
\begin{equation}
\nabla_A \mathcal{L}_{t,\alpha}(A)
=
\alpha(t)\nabla_A\mathcal{L}_t(A)
=
2\alpha(t)(A\Sigma_t-C_t).
\end{equation}

To see the effect on curvature, consider a perturbation $\Delta\in\mathbb{R}^{d\times d}$ around
$A$. The gradient changes as
\begin{equation}
\nabla_A \mathcal{L}_{t,\alpha}(A+\Delta)
-
\nabla_A \mathcal{L}_{t,\alpha}(A)
=
2\alpha(t)\Delta\Sigma_t.
\end{equation}
Thus the Hessian is the linear operator
\begin{equation}
\Delta \mapsto 2\alpha(t)\Delta\Sigma_t.
\end{equation}
Since $\Sigma_t$ is symmetric positive definite, let
\begin{equation}
\Sigma_t=U\operatorname{diag}(\sigma_1(t),\dots,\sigma_d(t))U^\top.
\end{equation}
In the eigenbasis of $\Sigma_t$, the curvature in direction $u_i$ is scaled by
\begin{equation}
2\alpha(t)\sigma_i(t).
\end{equation}
Hence the largest and smallest directional curvatures are
\begin{equation}
2\alpha(t)\sigma_{\max}(t)
\qquad\text{and}\qquad
2\alpha(t)\sigma_{\min}(t),
\end{equation}
respectively. Their ratio is therefore
\begin{equation}
\frac{2\alpha(t)\sigma_{\max}(t)}
     {2\alpha(t)\sigma_{\min}(t)}
=
\frac{\sigma_{\max}(t)}{\sigma_{\min}(t)}
=
\kappa(\Sigma_t).
\end{equation}
Thus scalar time reweighting changes the overall scale of the curvature, but not the fixed-time
directional condition number.

Equivalently, full-batch gradient descent on $\mathcal{L}_{t,\alpha}$ gives
\begin{equation}
A_{k+1}
=
A_k-2\eta\alpha(t)(A_k\Sigma_t-C_t).
\end{equation}
Let $A^\star(t)$ be the minimizer, so that $A^\star(t)\Sigma_t=C_t$, and define
\begin{equation}
E_k:=A_k-A^\star(t).
\end{equation}
Then
\begin{equation}
E_{k+1}
=
E_k\bigl(I-2\eta\alpha(t)\Sigma_t\bigr).
\end{equation}
Projecting onto an eigendirection $u_i$ of $\Sigma_t$ gives the mode-wise contraction factor
\begin{equation}
1-2\eta\alpha(t)\sigma_i(t).
\end{equation}
Thus $\alpha(t)$ rescales the effective step size from $\eta$ to $\eta\alpha(t)$, but the relative
separation between fast and slow directions is still governed by the spectrum of $\Sigma_t$. With
the optimal fixed step size, the iteration complexity therefore remains controlled by
$\kappa(\Sigma_t)$ at fixed $t$.
\end{proof}

\section{Additional Experiments and Results}
\label{appsec:additional_exp_and_results}
\subsection{Visualizing preconditioned transport for the Swiss roll data using a normalizing flow and a flow-based preconditioner} 
\label{subsec:point_clouds_2d}
\begin{wrapfigure}[13]{r}{0.40\textwidth}
    \vspace{-7pt}
    \centering
    \setlength{\tabcolsep}{2pt} 
    \begin{tabular}{@{}cc@{}}
        \includegraphics[width=0.18\textwidth]{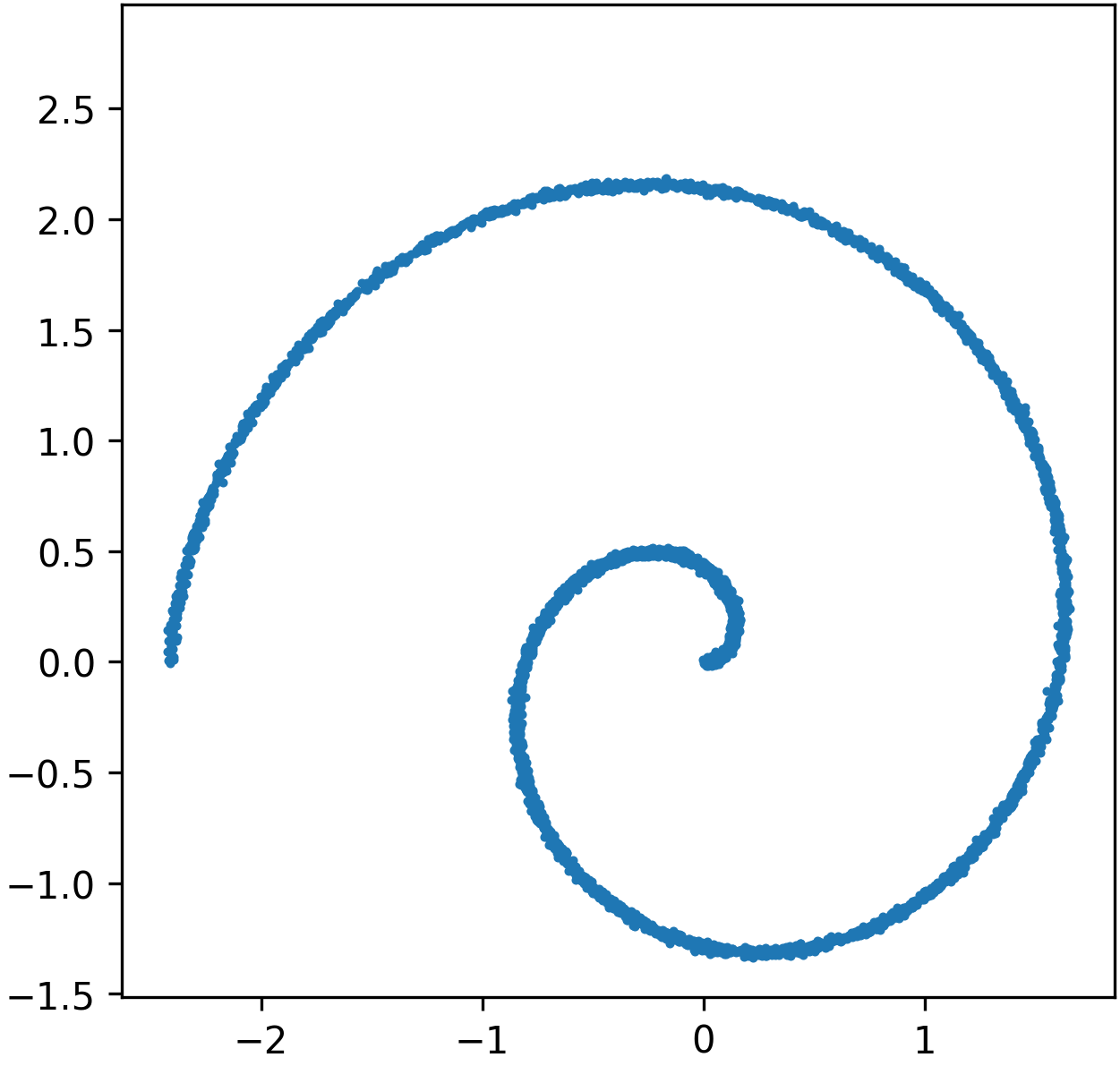} &
        \includegraphics[width=0.18\textwidth]{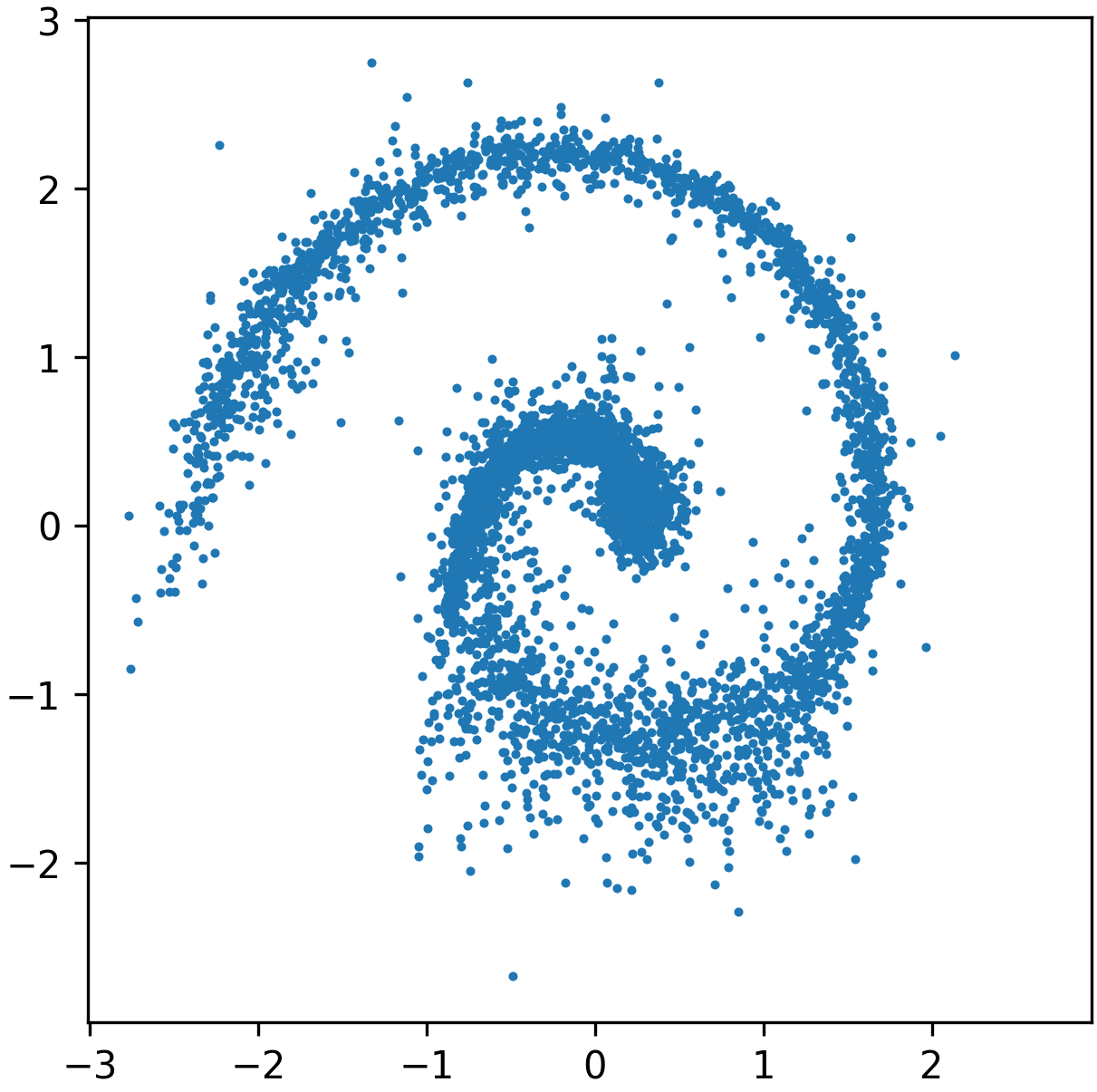} \\
        {\small (a) $x_1 \sim p_1$} &
        {\small (b) $\hat{x}_1 = \mathrm{Flow}(z)$}
    \end{tabular}
    \caption{\textbf{Flow matching without preconditioning.} \textbf{(a)} Ground-truth Swiss roll samples. \textbf{(b)} Generated samples obtained by transporting Gaussian source $z$ through a learned flow.}
    \label{fig:flow_matching}
\end{wrapfigure}

To build intuition and isolate algorithmic behavior of the proposed preconditioning strategy, we begin with 2D experiments since low-dimensional settings allow for direct visualization of densities, transport maps, and trajectories.

We study the Swiss roll, a curved, anisotropic, non-Gaussian manifold with severe scale separation, as a testbed for evaluating preconditioning methods. Even with small isotropic noise, its geometry is challenging: elongated along the manifold and tightly concentrated transversely. Preconditioning here serves as geometric isotropization, aiming to reduce anisotropy and simplify the structure to facilitate easier transport toward a Gaussian. As a baseline, we use standard flow matching to map a Gaussian to the Swiss roll, the results for which are presented in \Cref{fig:flow_matching}. We then use both a normalizing flow and a flow matching-based model to precondition the problem.

\begin{wraptable}[10]{r}{0.5\textwidth}
\vspace{-0.8em}
\centering
\caption{Sliced-Wasserstein distances for the Swiss roll dataset. The top row compares generated samples with ground truth; the bottom row measures how Gaussian the pushed data is. Lower is better.}
\resizebox{0.5\textwidth}{!}{
\begin{tabular}{lccc}
\toprule
Direction & \makecell{No\\preconditioner} & \makecell{NF\\preconditioner} & \makecell{FM\\preconditioner} \\
\midrule
$z \rightarrow x_1$ & $0.11$ & $0.06 $ & $0.07$ \\
$x_1 \rightarrow z$ & $0.81 $ & $0.31$ & $0.34$ \\
\bottomrule
\end{tabular}
}
\label{tab:2dresults}
\vspace{-1em}
\end{wraptable}

\paragraph{Normalizing flow preconditioner.} \Cref{fig:precond_comparison} (top row) shows that normalizing flow (NF) significantly improves generation quality over the standard non-preconditioned flow matching (\Cref{fig:flow_matching}). While the pushforward $\tilde{x} = \mathcal{P}_\theta (x_1)$ in \Cref{fig:precond_comparison} (a) isn't perfectly Gaussian either, it is far more isotropic (and closer to Gaussian), allowing the second-stage flow (b) to operate on a simpler geometry. In inference, Gaussian samples are pulled back via the learned flow (c) and the inverse NF (d), recovering the original distribution more faithfully, as compared to \Cref{fig:flow_matching} (b).

\begin{figure}[h]
    \centering
    \setlength{\tabcolsep}{3pt}
    \begin{tabular}{cccc}
        \includegraphics[width=0.20\linewidth]{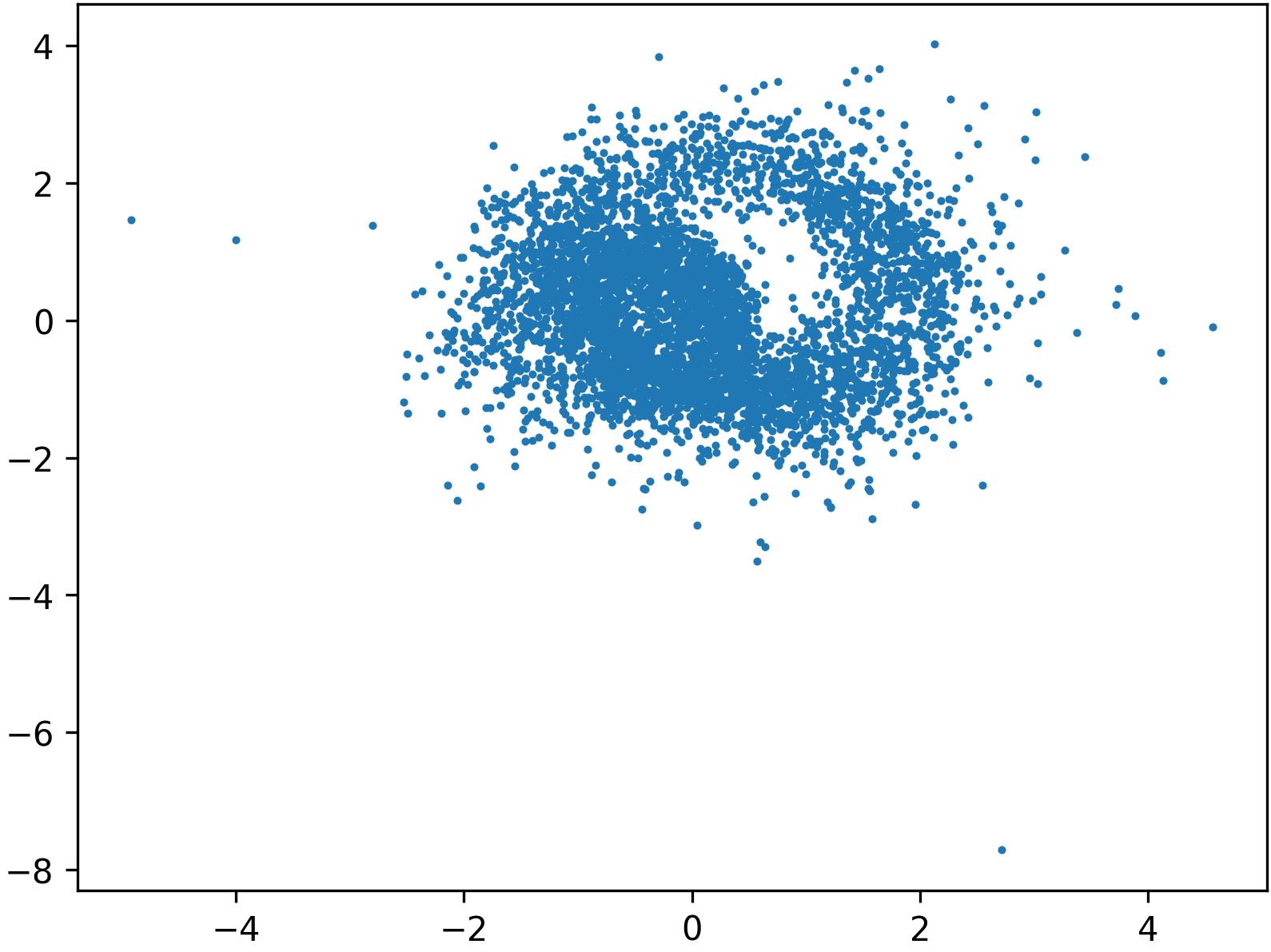} &
        \includegraphics[width=0.20\linewidth]{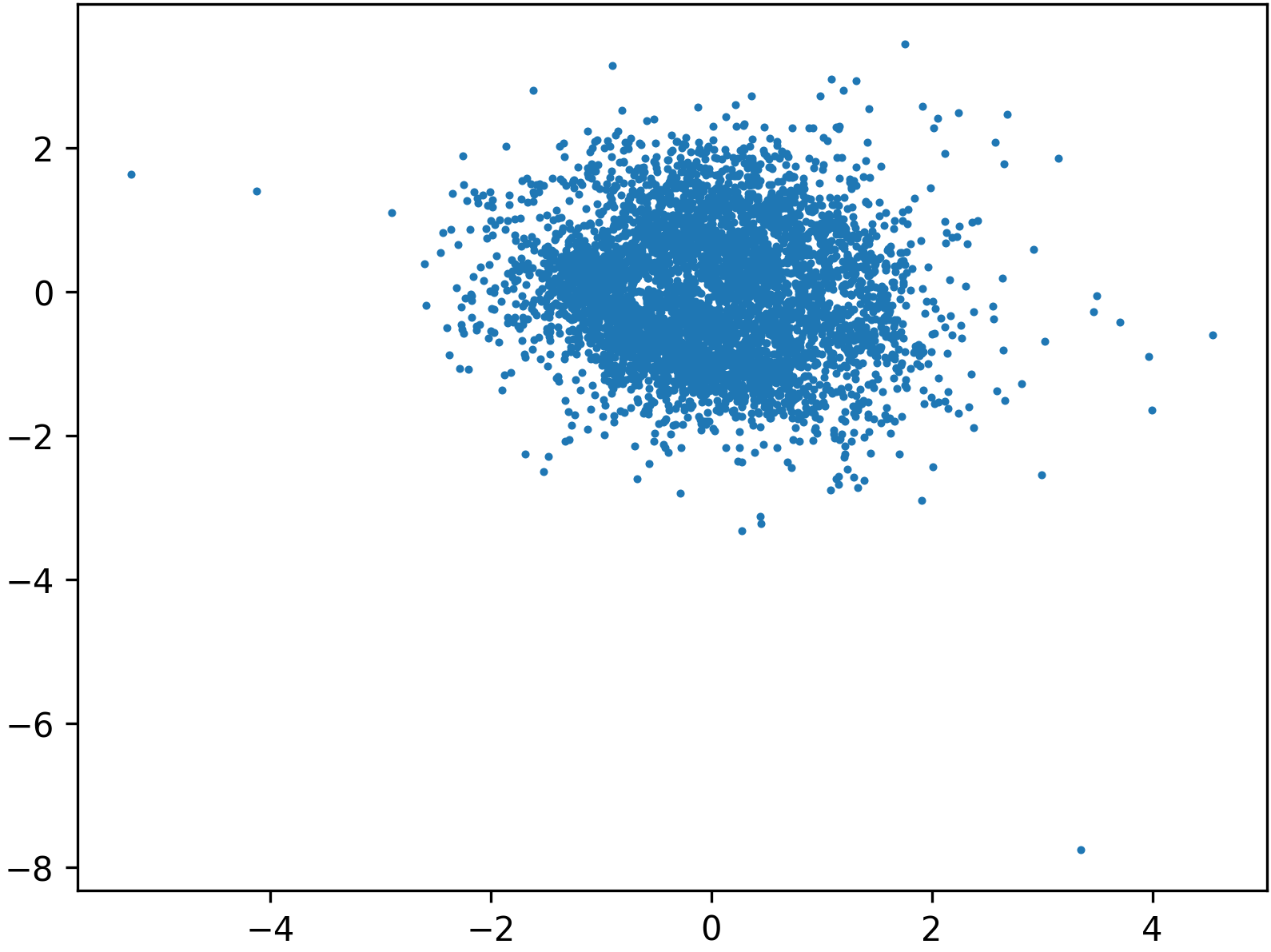} &
        \includegraphics[width=0.20\linewidth]{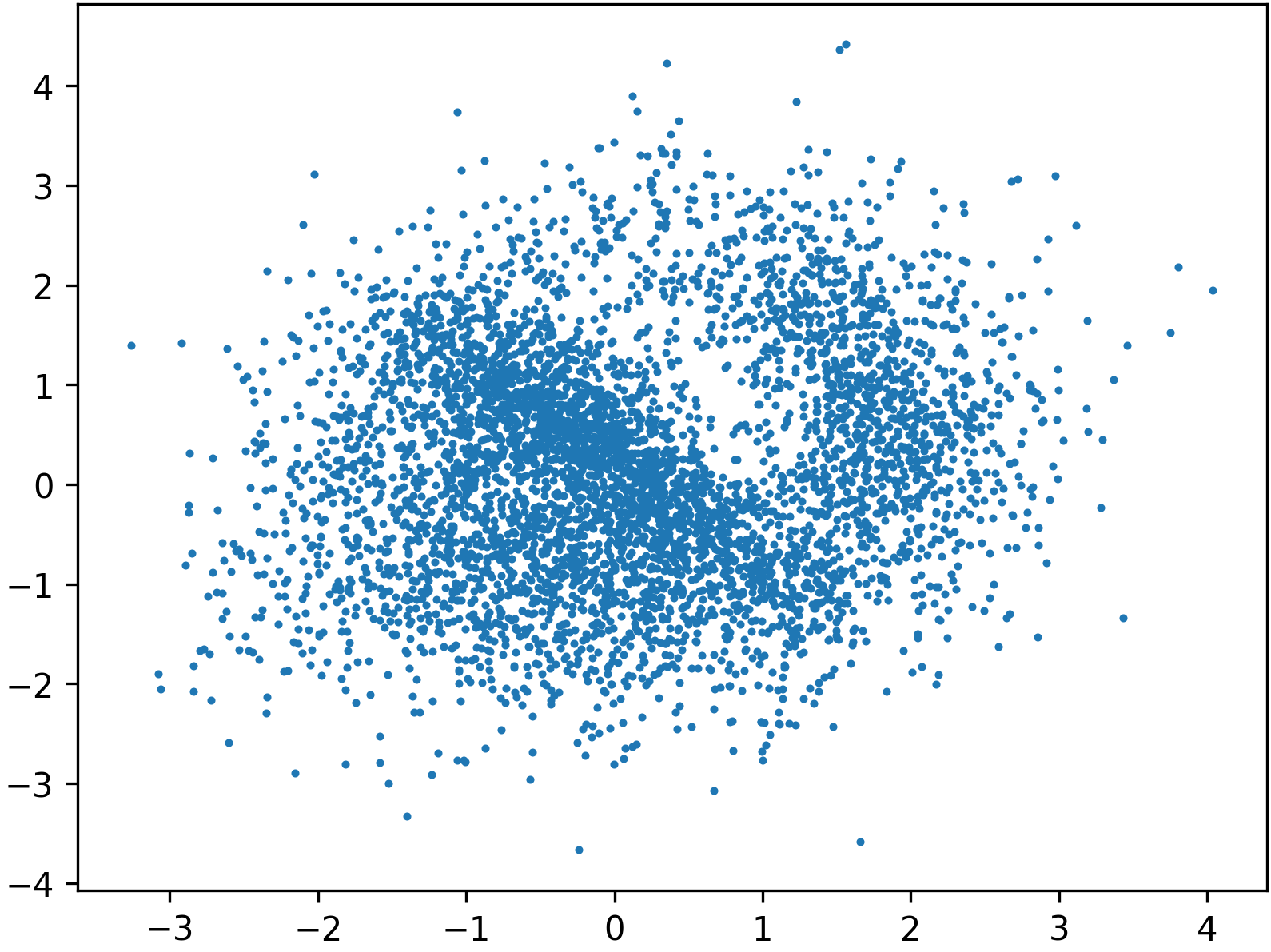} &
        \includegraphics[width=0.20\linewidth]{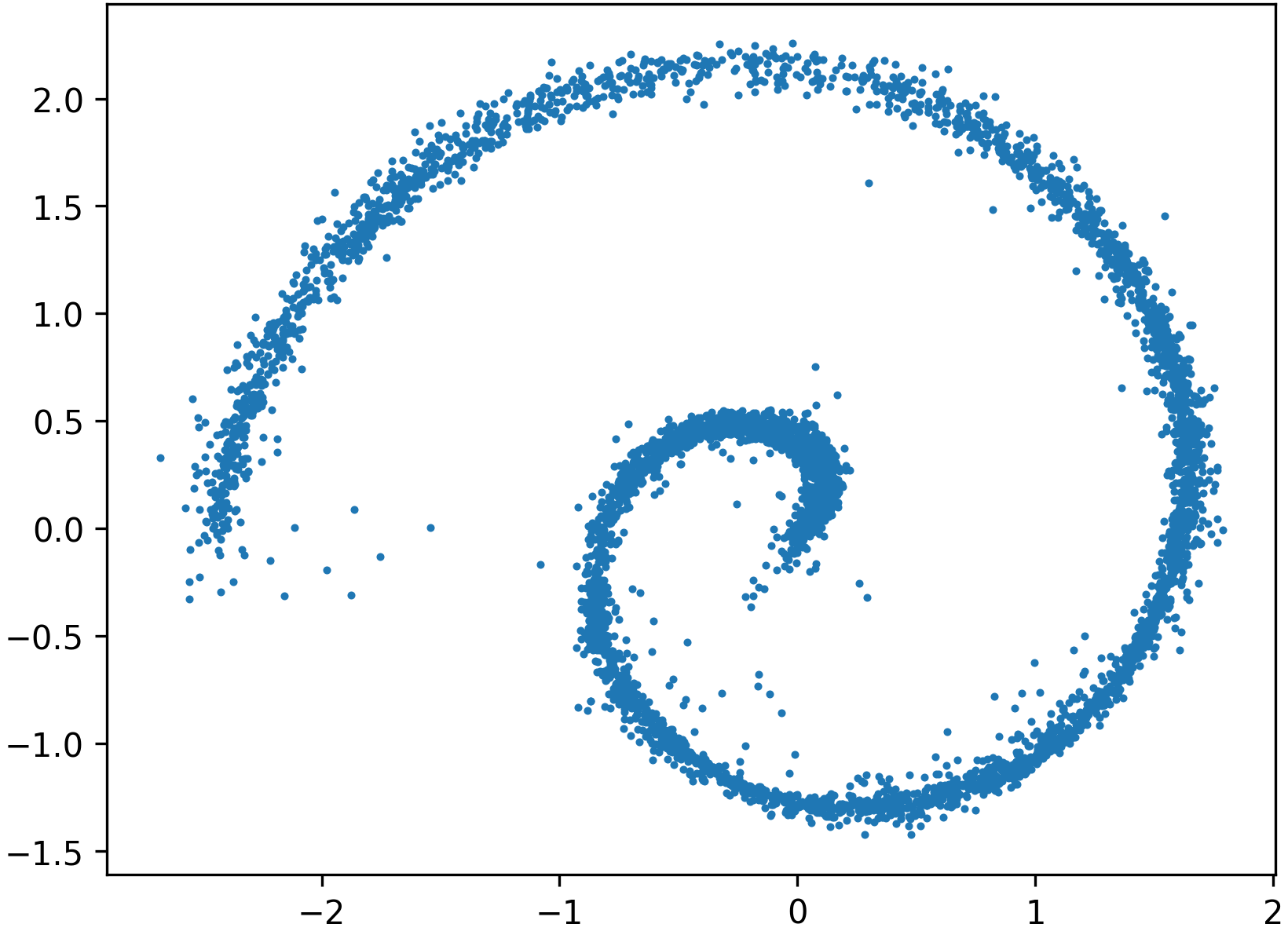} \\
        {\small (a) $\tilde x = {\cal P}_{\theta}(x_1)$} &
        {\small (b) $\hat{\tilde{x}}=\mathrm{Flow}(z)$} &
        {\small (c) $\tilde z=\mathrm{Flow}^{-1}(\tilde{x})$} &
        {\small (d) $\hat x = {\cal P}_{\theta}^{-1}(\hat{\tilde x})$} \\
        \includegraphics[width=0.20\linewidth]{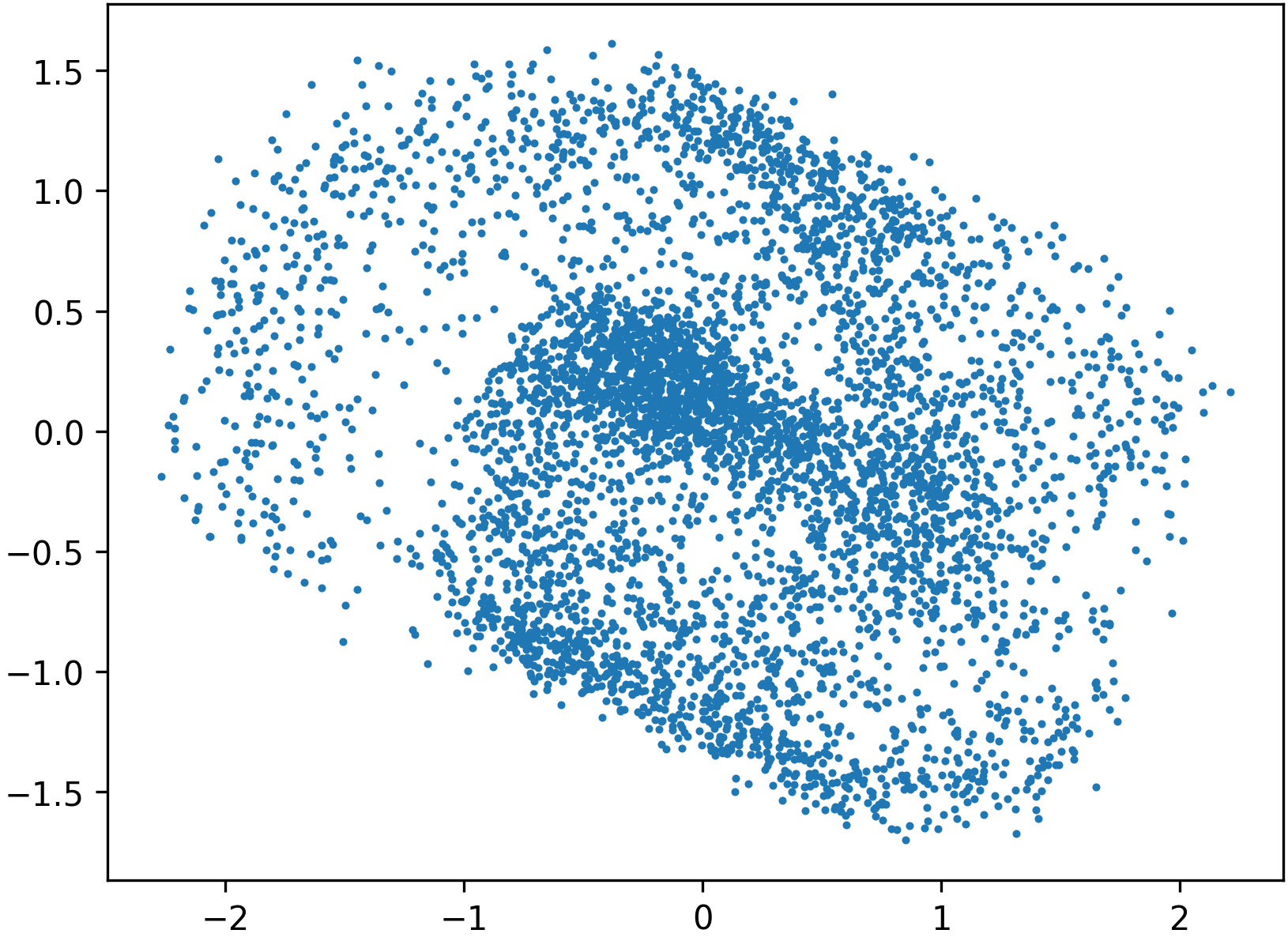} &
        \includegraphics[width=0.20\linewidth]{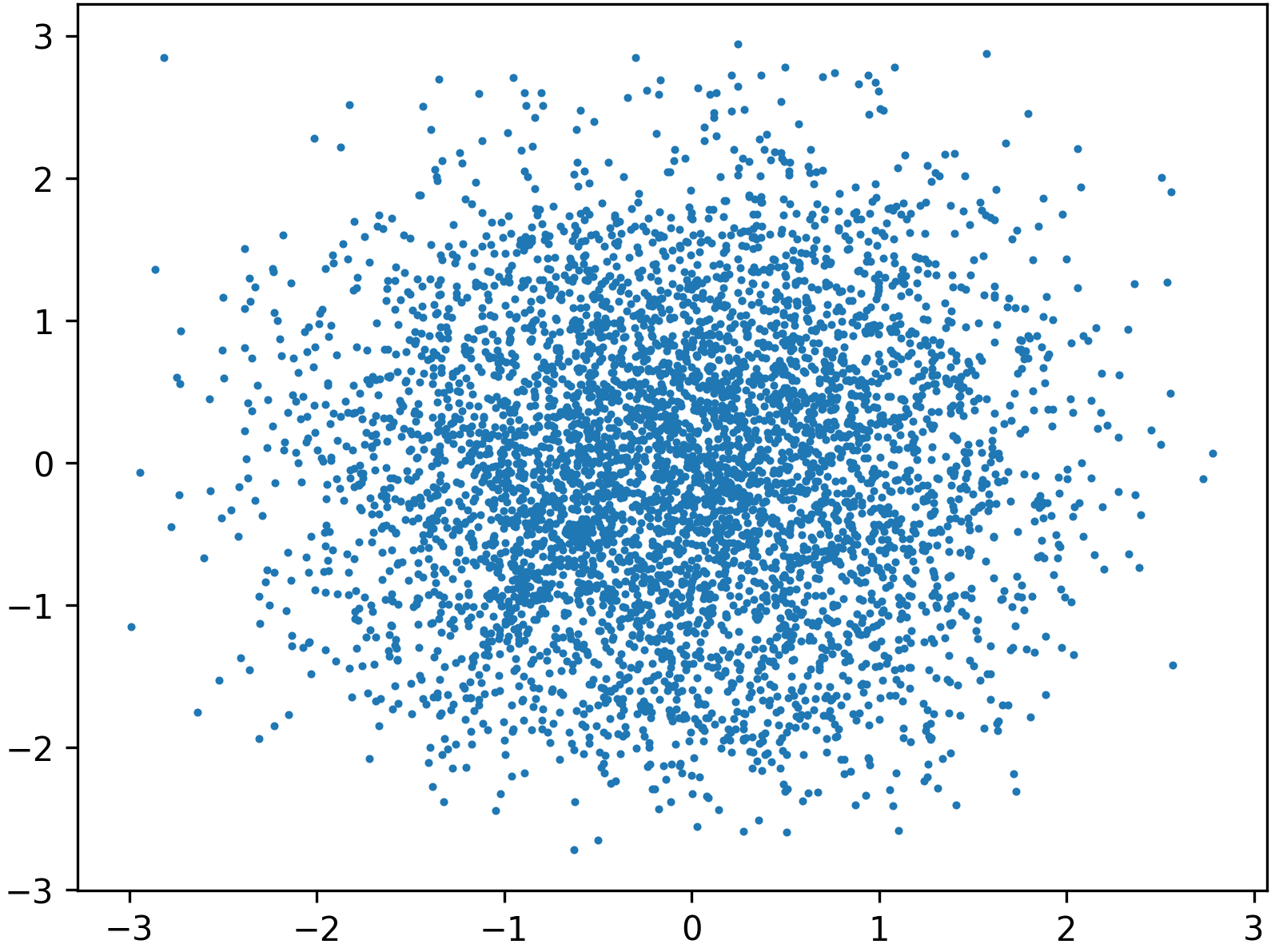} &
        \includegraphics[width=0.20\linewidth]{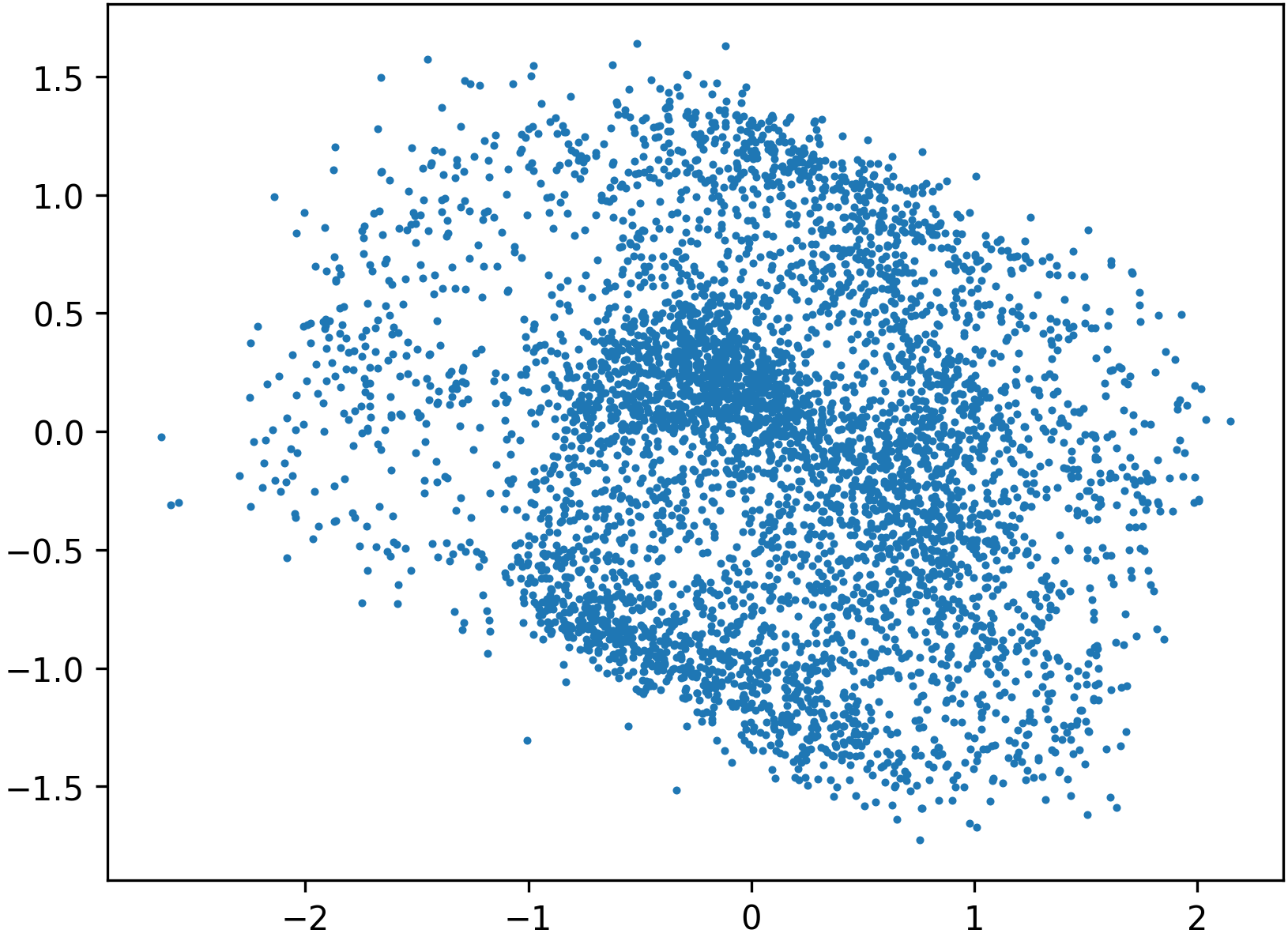} &
        \includegraphics[width=0.20\linewidth]{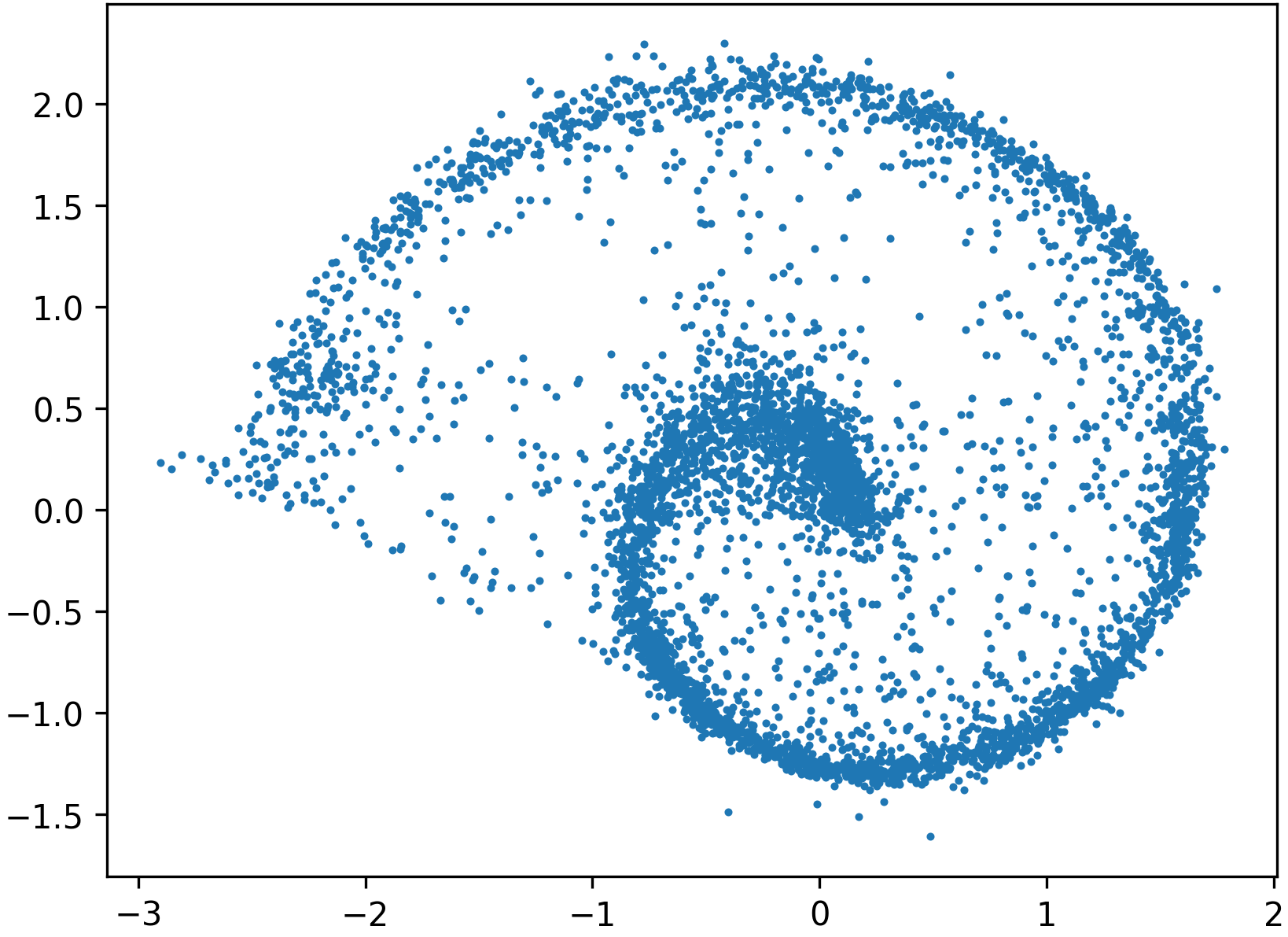} \\
        {\small (e) $\tilde x = \mathrm{Flow}_{\mathrm{lc}}(x_1)$} &
        {\small (f) $\hat{\tilde{x}}=\mathrm{Flow}(z)$} &
        {\small (g) $\tilde z=\mathrm{Flow}^{-1}(\tilde{x})$} &
        {\small (h) $\hat x = \mathrm{Flow}_{\mathrm{lc}}^{-1}(\hat{\tilde x})$}
    \end{tabular}
    \caption{
    \textbf{Flow matching with preconditioning.}
    \textbf{Top row: Normalizing flow preconditioning.}
    \textbf{(a)} The data $x_1$ (Swiss roll) is pushed forward by the normalizing flow preconditioner to $\tilde x$.
    \textbf{(b)} Random Gaussian samples are mapped by the learned downstream flow to the preconditioned data space.
    \textbf{(c)} For visualization, data in the preconditioned space can also be mapped backward through the learned flow toward Gaussian samples.
    \textbf{(d)} The points in the preconditioned space are mapped back to the original space using ${\cal P}_{\theta}^{-1}$.
    \textbf{Bottom row: Low-capacity flow preconditioning.}
    \textbf{(e)} The data $x_1$ is pushed forward by a low-capacity flow preconditioner to $\tilde x$.
    \textbf{(f)} Random Gaussian samples are mapped by the learned downstream flow to the preconditioned data space.
    \textbf{(g)} For visualization, data in the preconditioned space can also be mapped backward through the learned flow toward Gaussian samples.
    \textbf{(h)} Generated samples in the preconditioned space are mapped back using the inverse low-capacity flow.}
    \label{fig:precond_comparison}
\end{figure}

\paragraph{Flow matching preconditioner.} In our second experiment, we use a low-capacity flow $\text{Flow}_\text{lc}$ with just 168 parameters as a preconditioner. As shown in \Cref{fig:precond_comparison} (bottom row), even such a lightweight model improves the performance of a downstream high-capacity flow, similar to the normalizing flow case in \Cref{fig:precond_comparison} (top row).

To evaluate the effect of preconditioning quantitatively, we compute the Sliced-Wasserstein distance between point clouds in both directions: $x_1 \rightarrow z$ (data to Gaussian) and $z \rightarrow x_1$ (Gaussian to data). Results in \Cref{tab:2dresults} show that both preconditioners improve upon the baseline, with normalizing flow performing the best. Its distinct transformation is likely to complement flow matching by reducing the effective null space of the standard flow.

\subsection{Controlled Gaussian-Mixture experiments: conditioning and low-eigenvalue recovery}
\label{appsec:gmm_experiments}

\begin{figure}[h]
    \centering
    \includegraphics[width=\linewidth]{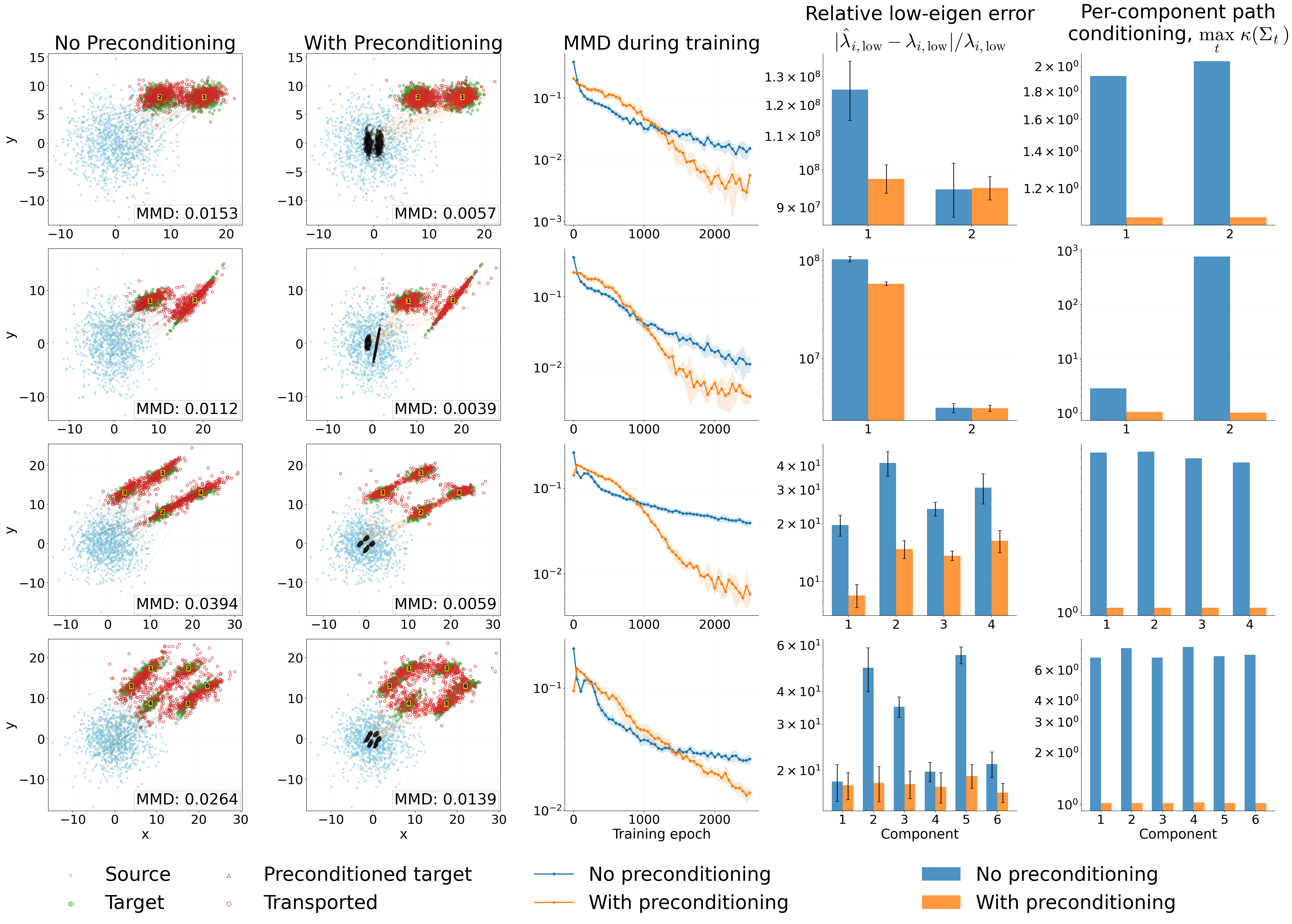}
    \caption{
    \textbf{Preconditioning improves Gaussian-mixture flow matching across controlled 2D configurations.} We compare standard flow matching and preconditioned flow matching on GMM targets with varying numbers of components, anisotropy levels, and transport distances.
    Each row shows one GMM configuration, while columns show: transported samples without preconditioning, transported samples with preconditioning, MMD during training, relative low-eigenvalue recovery error, and maximum per-component path condition number.
    Across all configurations, preconditioning produces better-aligned transported samples, lower final MMD, smaller low-eigenvalue error, and substantially reduced path conditioning.
    These diagnostics support the theoretical prediction that preconditioning improves flow matching optimization by reducing anisotropy along the interpolation path, especially for the hardest mixture components. Plots show the mean over 10 independent runs with different random seeds; error bars / shaded regions denote one standard deviation.}
    \label{fig:gmm_experiments_all_configs}
\end{figure}

Across all example GMM configurations shown in \Cref{fig:gmm_experiments_all_configs}, preconditioning improves the learned transport relative to the unpreconditioned baseline. In these cases, the unpreconditioned model often captures the high-variance direction of a component but fails to accurately recover its low-variance direction, producing visibly distorted samples and larger low-eigenvalue error.

The MMD curves show the same behavior from an optimization perspective (see column 3 in \Cref{fig:gmm_experiments_all_configs}). The unpreconditioned baseline often decreases the discrepancy early in training but then plateaus at a higher value. This matches the theoretical prediction that gradient-based training can make rapid progress along high-variance directions while making much slower progress along suppressed low-variance directions. Preconditioning reduces this imbalance, allowing training to continue improving the distributional fit.

The path-conditioning diagnostic provides the most direct link to the theory (see columns 4--5 in \Cref{fig:gmm_experiments_all_configs}). In the unpreconditioned case, the interpolation path inherits the anisotropy of the hardest mixture components, leading to large values of $\max_t \kappa(\Sigma_{t,k})$. After preconditioning, the corresponding path condition numbers are reduced across components. This explains why the preconditioned model better recovers low-eigenvalue directions and achieves lower final MMD. These results support the claim that preconditioning is not merely changing the visual shape of the generated samples, but is directly improving the geometry of the regression problems encountered during flow matching training.

\subsection{Optimal flow matching reduces stochastic-target error; preconditioning improves excess-risk optimization}
\label{appsec:comparison_to_ofm_and_iste}

\begin{figure}[h]
\centering
\includegraphics[width=0.9\textwidth]{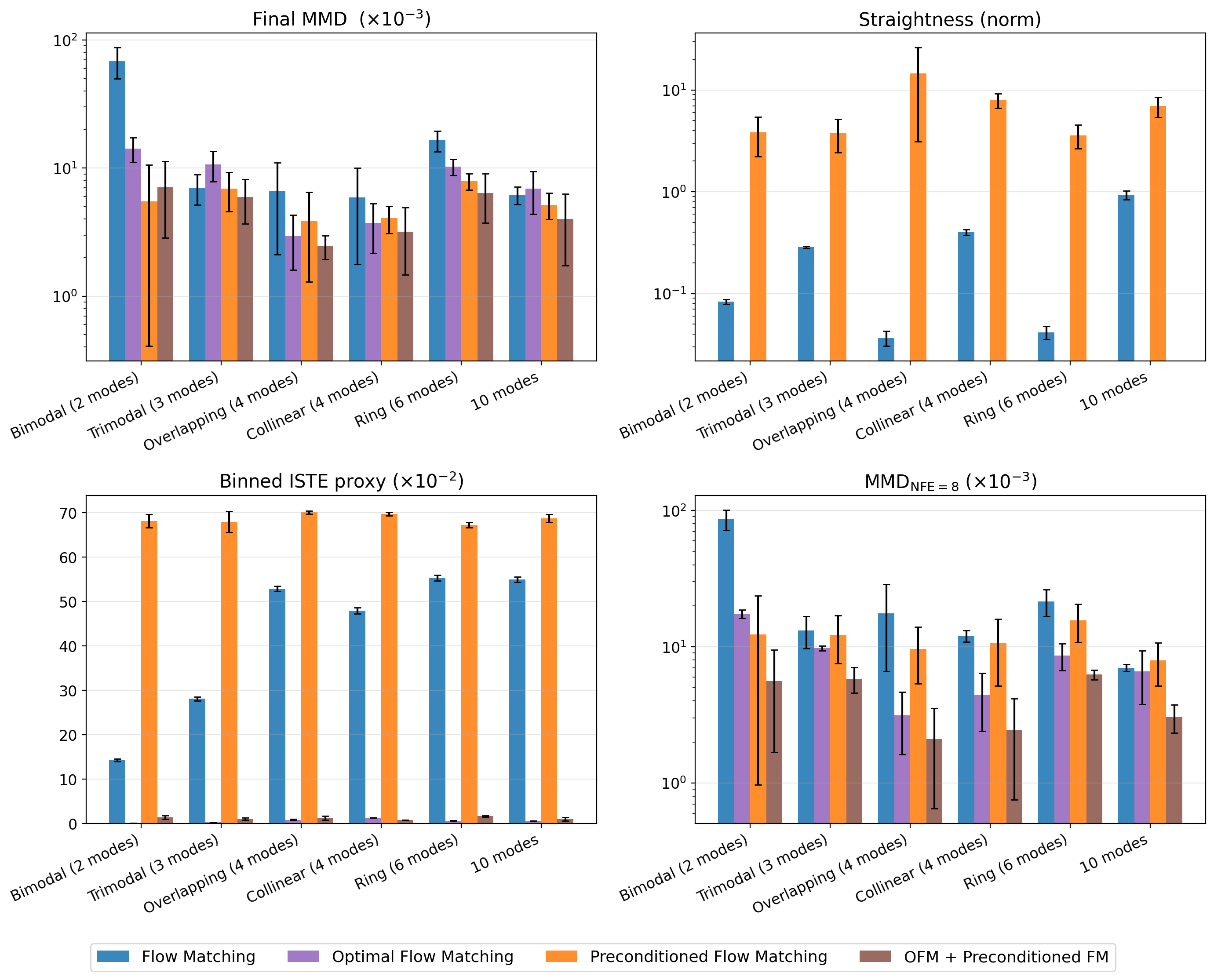}
\caption{\textbf{OFM and preconditioning across six Gaussian-mixture configurations.} Final MMD, trajectory straightness, normalized binned ISTE proxy, and MMD at the reported NFE~$=8$ setting. Endpoint MMD is evaluated after mapping preconditioned samples back to the original target space. Error bars show variability across 10 runs. The corresponding source and target distributions, together with samples transported by each method, are shown in \Cref{fig:figure_scatter_grid}.}
\label{fig:combined_metrics_comparison}
\end{figure}

\begin{figure}[]
\centering
\includegraphics[width=\linewidth]{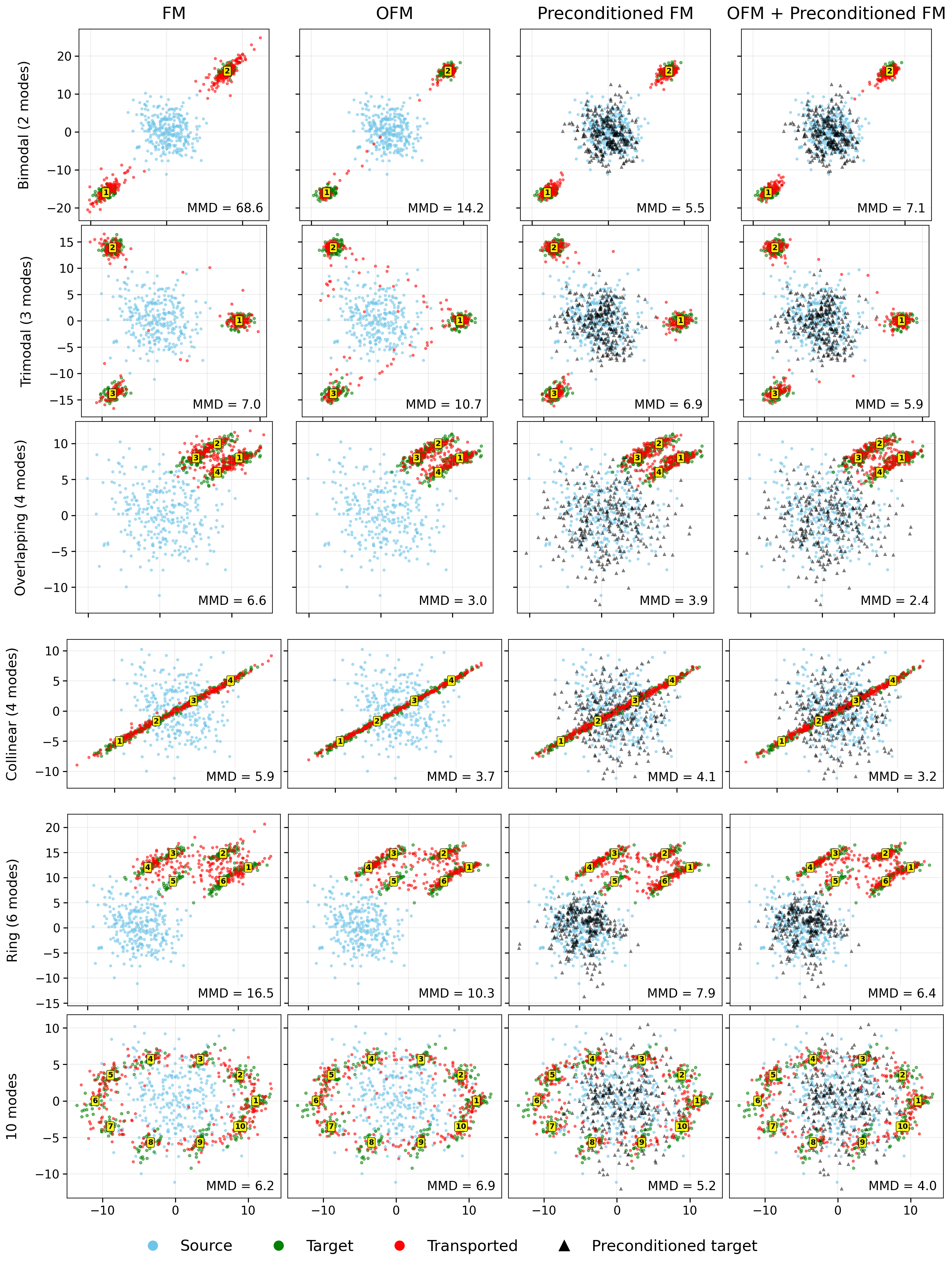}
\caption{\textbf{Generated samples for FM, OFM, preconditioned FM, and their combination.} Rows show Gaussian-mixture configurations; columns show methods. Blue points denote source samples, green points target samples, red points transported samples, and black triangles the transformed target for preconditioned methods. The displayed MMD values are in units of $10^{-3}$.}
\label{fig:figure_scatter_grid}
\end{figure}

The FM risk decomposition in \Cref{eq:regression_decomposition} separates irreducible stochastic-target error (ISTE), caused by ambiguity in the target velocity at an intermediate point, from excess risk, which measures error in the learned conditional-mean velocity. Optimal Flow Matching (OFM) \citep{kornilov2024optimal} addresses the first term through a deterministic Monge-form coupling. Preconditioning instead aims to reduce the second by improving the conditioning of the intermediate distributions. Because transforming the target can also increase ISTE, we examine this trade-off on six Gaussian-mixture configurations by comparing standard FM, Optimal Flow Matching (OFM) \citep{kornilov2024optimal}, preconditioned FM, and OFM applied in the preconditioned space. Generated samples from the latter two methods are mapped back to the original space for MMD evaluation. We estimate a normalized, binned ISTE proxy by averaging the conditional covariance of the target displacement within bins of intermediate points and dividing by the overall displacement covariance. This proxy can remain nonzero under a deterministic Monge coupling because each bin contains distinct intermediate points.

\Cref{fig:combined_metrics_comparison} show that preconditioning alone increases the binned ISTE proxy and the measured deviation from straight trajectories in all six configurations, even though it reduces final MMD relative to standard FM. Thus, the improvement in distributional fit does not, per say, require a reduction in stochastic-target ambiguity. OFM and its preconditioned variant retain a low binned ISTE proxy. The combined method has the lowest \emph{mean} final MMD in five of six configurations and the lowest MMD at the reported NFE~$=8$ setting in all six. These results in \Cref{fig:combined_metrics_comparison,fig:figure_scatter_grid} suggest that target-space conditioning and Monge-form coupling can provide complementary benefits. The binned proxy does not, however, establish whether individual trajectories cross, and straightness in preconditioned coordinates need not persist after applying a nonlinear inverse preconditioner.

\subsection{Condition number dynamics for MNIST and ImageNet-1k training}
\label{appsubsec:condition_number_mnist}
To understand how preconditioning impacts the geometry of the transport process, we measure the condition number $\kappa(\Sigma_t)$ of the covariance matrix of samples $x_t$ from the intermediate distributions $p_t$ at various time steps $t \in (0, 1)$ during the flow matching training for the MNIST and ImageNet-1k datasets. A higher condition number indicates stronger anisotropy, which typically hinders optimization during flow matching.

\paragraph{Condition number estimation.} During training, for each time step $t$, we draw $8192$ samples from the intermediate distribution by forming
$x_t = s(t)x_1 + c(t)x_0$, where $x_0 \sim \mathcal{N}(0, I)$ and $x_1$ is drawn uniformly
from the training set (normalized latents for the baseline; preconditioner-transformed latents
$\tilde{x}_1 = \mathcal{P}(x_1)$ for the preconditioned variants).
We then compute the empirical covariance
$\hat{\Sigma}_t = \frac{1}{N-1}\sum_{i=1}^N (x_t^{(i)} - \bar{x}_t)(x_t^{(i)} - \bar{x}_t)^\top$
and report the condition number
$\kappa(\hat{\Sigma}_t) = \lambda_{\max}(\hat{\Sigma}_t) / \lambda_{\min}(\hat{\Sigma}_t)$,
regularized by adding $\epsilon I$ with $\epsilon = 10^{-6} \cdot (\operatorname{tr}(\hat{\Sigma}_t)/D)$
to avoid numerical degeneracies.
All covariances are computed on flattened latent vectors.

\Cref{fig:kappa_vs_t_compare_datasets} shows that the intermediate FM distributions become increasingly ill-conditioned as $t \to 1$ for both MNIST and ImageNet-1k latent representations. This indicates that the downstream regression problem inherits anisotropy from the target latent distribution. Both normalizing flow (NF) and flow matching (FM) preconditioners substantially reduce the condition numbers, especially near the data endpoint. On MNIST, the NF preconditioner gives the most isotropic path, consistent with its strong quantitative performance in \Cref{tab:mnist_table}. The ImageNet-1k results show that this conditioning improvement also holds in a higher-dimensional and more complex latent space, supporting the broader role of preconditioning as a geometric optimization tool for flow matching.

\begin{figure*}[h]
    \centering

    \begin{subfigure}[t]{0.48\linewidth}
        \centering
        \includegraphics[width=\linewidth]{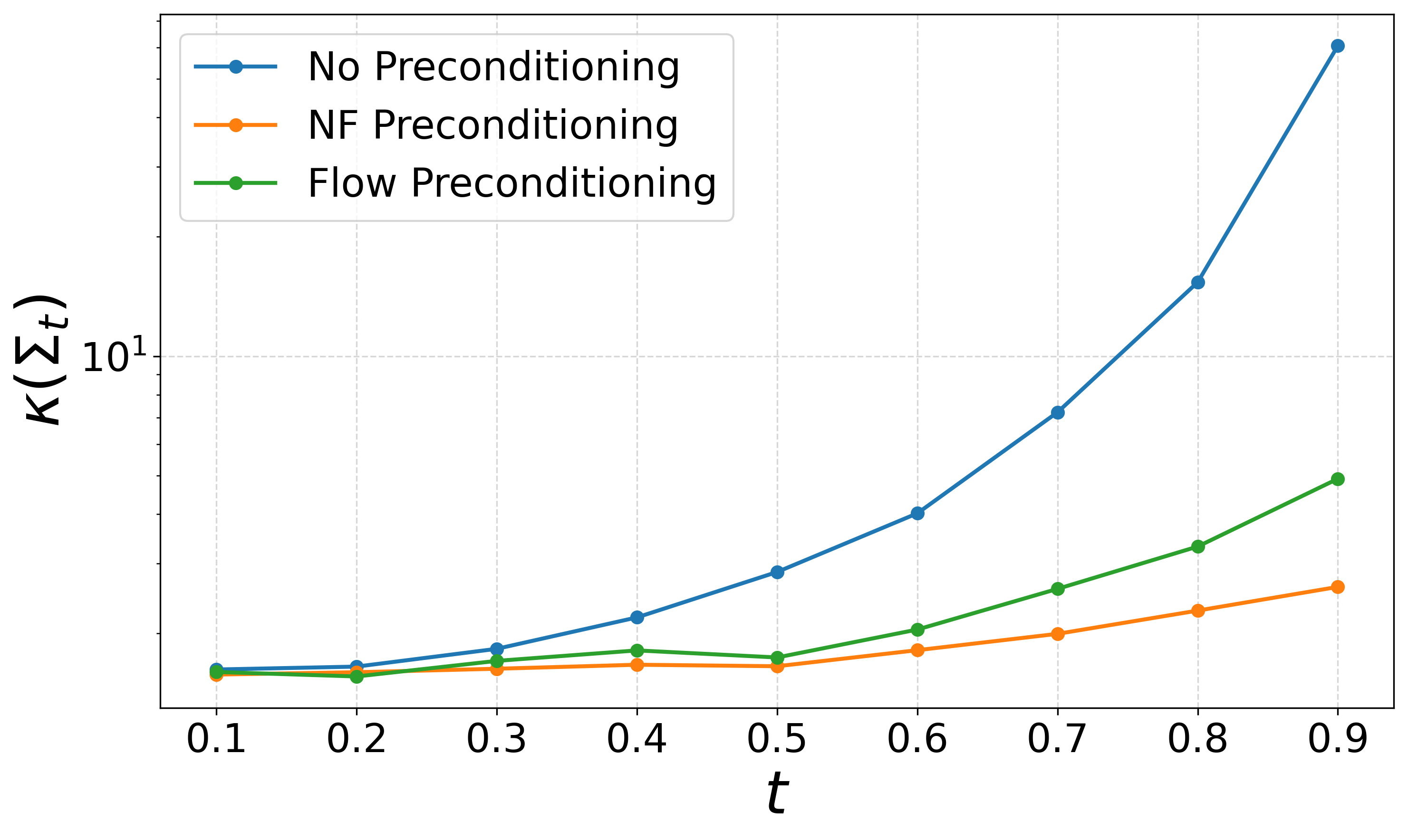}
        \caption{MNIST latent space}
        \label{fig:kappa_vs_t_mnist}
    \end{subfigure}
    \hfill
    \begin{subfigure}[t]{0.48\linewidth}
        \centering
        \includegraphics[width=\linewidth]{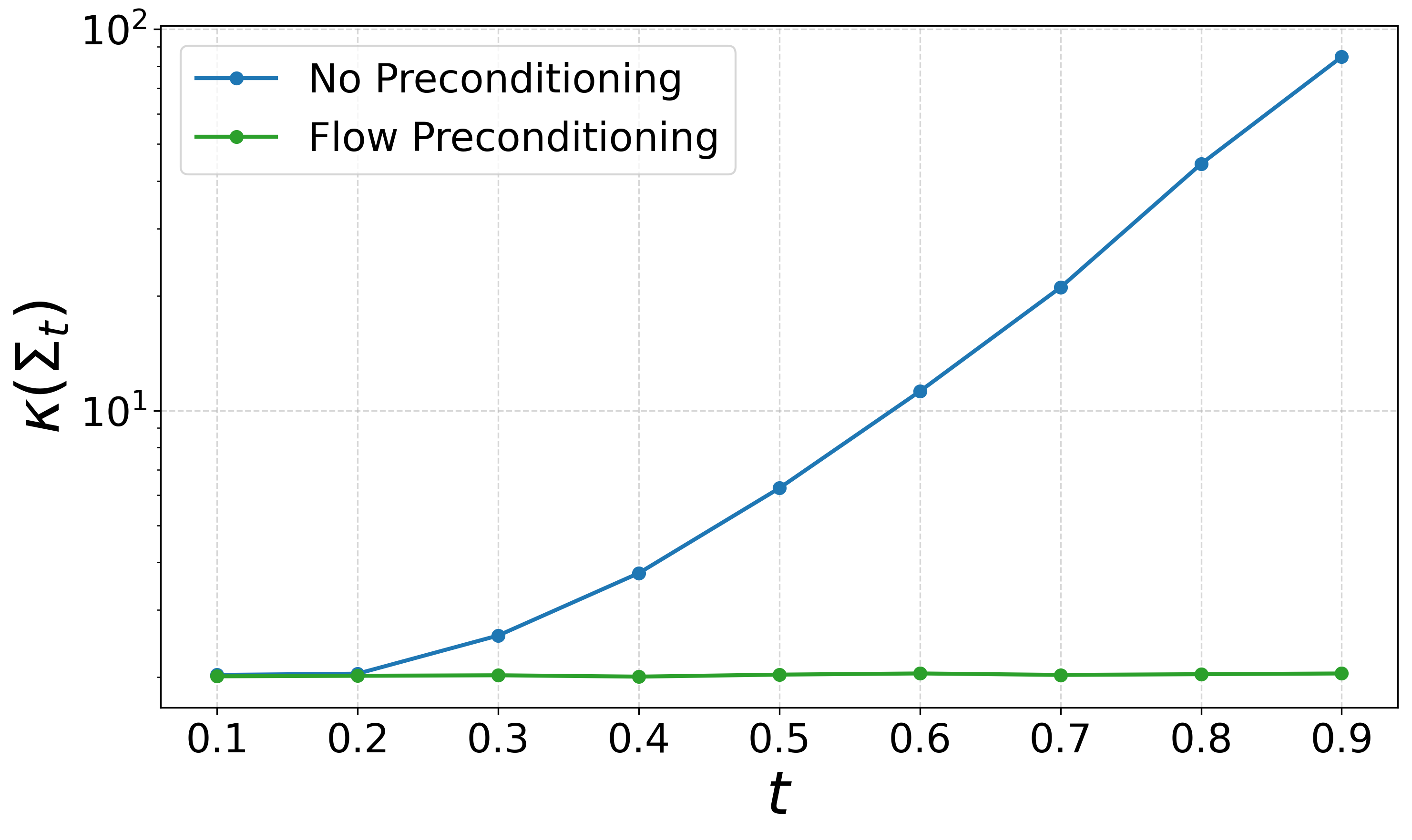}
        \caption{ImageNet-1k latent space}
        \label{fig:kappa_vs_t_imagenet1k}
    \end{subfigure}

    \caption{
    \textbf{Conditioning of intermediate flow matching distributions with and without preconditioning.}
    We plot the condition number $\kappa(\Sigma_t)$ of the intermediate covariance matrices $\Sigma_t$
    over time $t$ for \textbf{(a)} MNIST and \textbf{(b)} ImageNet-1k latent representations.
    Without preconditioning, the intermediate distributions become increasingly ill-conditioned as
    $t \to 1$, indicating that the downstream flow matching regression problem inherits strong
    anisotropy from the target representation. Both normalizing flow (NF) and flow matching (FM)
    preconditioning substantially reduce $\kappa(\Sigma_t)$, especially near the data endpoint,
    supporting the claim that preconditioning improves the geometry of the FM path rather than merely
    adding model capacity.
    }
    \label{fig:kappa_vs_t_compare_datasets}
\end{figure*}

\subsection{Mechanistic analysis of directional learning on latent MNIST}
\label{appsec:mnist_directional_mechanism}

\Cref{appsubsec:condition_number_mnist} showed that preconditioning reduces the condition number of the empirical flow-matching path on both MNIST and ImageNet-1k. We next investigate whether this improvement in path geometry is accompanied by the directional optimization behavior predicted by the theory. In particular, we examine whether the learned nonlinear velocity field fits high-variance directions more rapidly than low-variance directions, and whether preconditioning reduces this learning imbalance. We perform this analysis post hoc using checkpoints saved throughout the training of the latent-MNIST models.

\paragraph{Empirical latent geometry.}
Let $z_1\in\mathbb{R}^{64}$ denote the latent representation produced by the pretrained MNIST VAE. For the unpreconditioned model, the flow-matching target is $z_1$. For the NF-preconditioned model, the corresponding target representation is $\widetilde{z}_1 = P_{\phi}(z_1,y)$, where $P_{\phi}$ is the class-conditional normalizing-flow preconditioner and $y$ is the MNIST class label. All quantities below are computed separately in the representation space used to train each model.

For each representation, we estimate the target mean and covariance from the MNIST training set:
\begin{equation}
    \widehat{\mu}
    =
    \frac{1}{M}\sum_{m=1}^{M}z_1^{(m)},
    \qquad
    \widehat{H}
    =
    \frac{1}{M-1}
    \sum_{m=1}^{M}
    \left(z_1^{(m)}-\widehat{\mu}\right)
    \left(z_1^{(m)}-\widehat{\mu}\right)^\top.
\end{equation}
Here, $z_1^{(m)}$ denotes either the original or preconditioned latent, depending on the model under consideration. We eigendecompose $\widehat{H}=U\Lambda U^\top$, and order the eigenvectors $u_i$ by increasing eigenvalue. The eight eigenvectors associated with the largest eigenvalues define the high-variance subspace, while those associated with the smallest eigenvalues define the low-variance subspace.

The MNIST models use the trigonometric interpolation, $x_t = s(t)z_1+c(t)z_0$, with $s(t)=\sin\left(\frac{\pi t}{2}\right)$ and $c(t)=\cos\left(\frac{\pi t}{2}\right)$, where $z_0\sim\mathcal{N}(0,I)$. The corresponding conditional flow-matching target is
\begin{equation}
    v_t^\star
    =
    \dot{s}(t)z_1+\dot{c}(t)z_0
    =
    \frac{\pi}{2}
    \left(
        c(t)z_1-s(t)z_0
    \right).
\end{equation}
Under the moment-matched Gaussian approximation
$z_1\sim\mathcal{N}(\widehat{\mu},\widehat{H})$, the intermediate covariance is
\begin{equation}
    \widehat{\Sigma}_t
    =
    c(t)^2I+s(t)^2\widehat{H}.
\end{equation}
Because $\widehat{\Sigma}_t$ is a polynomial in $\widehat{H}$, the two matrices share the same eigenvectors. The associated Gaussian-optimal linear velocity operator is
\begin{equation}
\label{eq:mnist_gaussian_reference}
    \widehat{A}_t^\star
    =
    \frac{\pi}{2}s(t)c(t)
    \left(\widehat{H}-I\right)
    \left(
        c(t)^2I+s(t)^2\widehat{H}
    \right)^{-1}.
\end{equation}

\paragraph{Jacobian-based directional error.}
The trained MNIST velocity field is nonlinear, so its local action at an intermediate sample $x_t$ is characterized by its input Jacobian
\begin{equation}
    J_{\theta_e}(x_t,y,t)
    =
    \frac{\partial v_{\theta_e}(x_t,y,t)}
         {\partial x_t},
\end{equation}
where $e$ denotes the training epoch. We evaluate the action of this Jacobian along each covariance eigendirection using exact Jacobian-vector products.

For eigendirection $u_i$, we define the covariance-unweighted directional operator error as
\begin{equation}
\label{eq:mnist_per_direction_error}
    e_i(t,e)
    =
    \frac{1}{N_{\mathrm{probe}}}
    \sum_{n=1}^{N_{\mathrm{probe}}}
    \left\|
        J_{\theta_e}
        \left(x_t^{(n)},y^{(n)},t\right)u_i
        -
        \widehat{A}_t^\star u_i
    \right\|_2^2.
\end{equation}
The same fixed set of probe samples is used for every checkpoint, so changes in this quantity reflect changes in the learned velocity field rather than sampling variation. We evaluate Equation~\eqref{eq:mnist_per_direction_error} at
$t\in\{0.50,0.75,0.90\}$.

Let $\mathcal{I}_{\mathrm{top}\text{-}8}$ and
$\mathcal{I}_{\mathrm{bottom}\text{-}8}$ denote the indices of the eight highest- and eight lowest-variance eigendirections, respectively. We compute
\begin{align}
    \operatorname{err}_{\mathrm{top}\text{-}8}(t,e)
    &=
    \frac{1}{8}
    \sum_{i\in\mathcal{I}_{\mathrm{top}\text{-}8}}
    e_i(t,e),
    \\
    \operatorname{err}_{\mathrm{bottom}\text{-}8}(t,e)
    &=
    \frac{1}{8}
    \sum_{i\in\mathcal{I}_{\mathrm{bottom}\text{-}8}}
    e_i(t,e).
\end{align}
Unlike the FM excess risk, these errors are not weighted by the corresponding covariance eigenvalues. Consequently, an error is not down-weighted merely because the data variance along that direction is small.

We summarize the directional learning imbalance using
\begin{equation}
\label{eq:mnist_directional_ratio}
    R(t,e)
    =
    \frac{
        \operatorname{err}_{\mathrm{bottom}\text{-}8}(t,e)
    }{
        \operatorname{err}_{\mathrm{top}\text{-}8}(t,e)
    }.
\end{equation}
A value $R(t,e)\approx 1$ indicates that the high- and low-variance subspaces are learned with comparable accuracy, whereas $R(t,e)\gg 1$ indicates preferential learning of the high-variance directions.

\begin{figure}[t]
    \centering
    \includegraphics[width=0.9\textwidth]
    {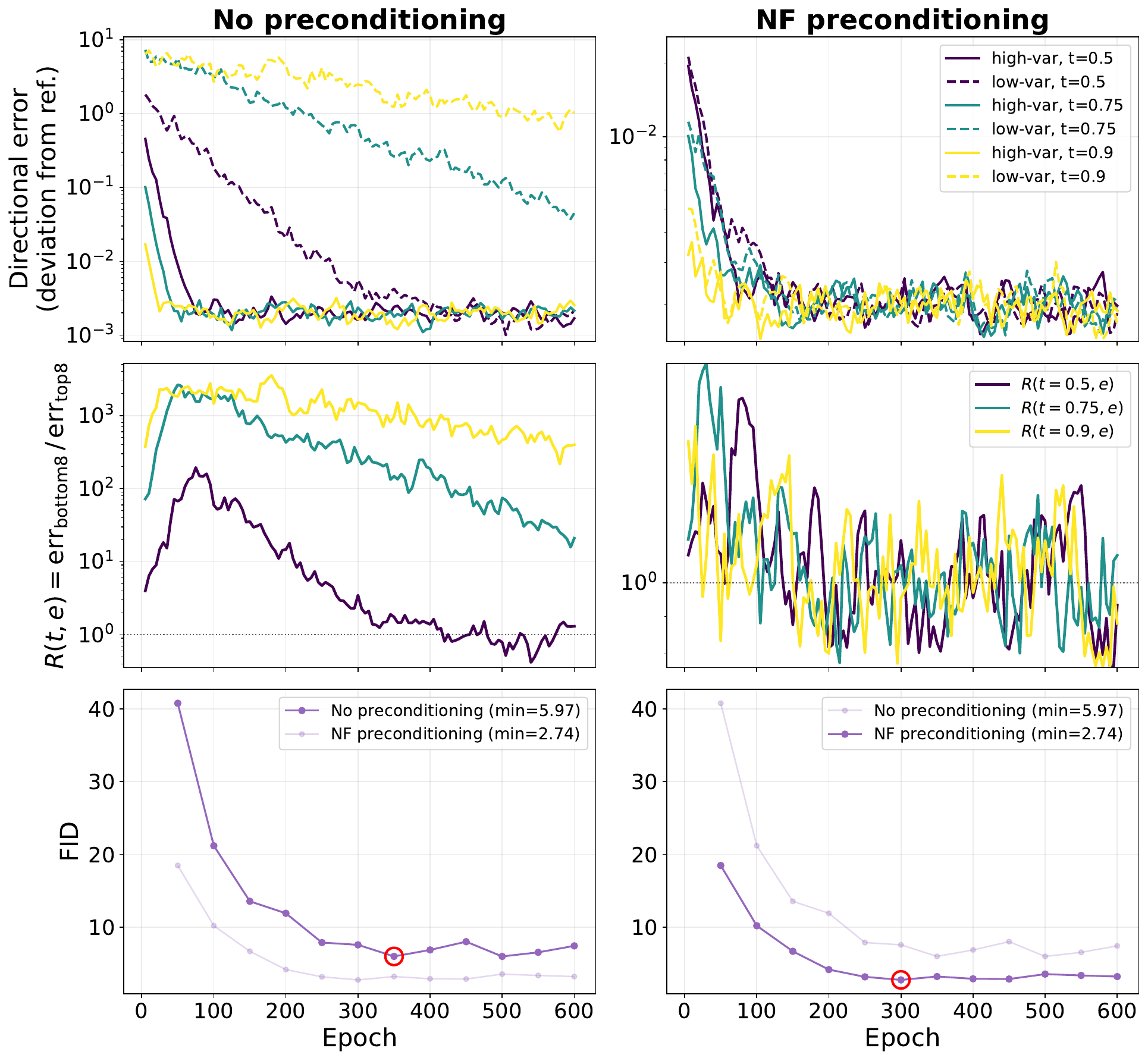}
    \caption{
        \textbf{Directional learning dynamics and generation quality on latent MNIST.}
        The left and right columns show standard flow matching and NF-preconditioned flow matching, respectively. Both models use the sine--cosine interpolation $x_t=\sin(\pi t/2)x_1+\cos(\pi t/2)x_0$, where $x_0\sim\mathcal{N}(0,I)$ and $x_1$ denotes the target in the latent representation used by each model.
        \textbf{Top:} Covariance-unweighted directional operator errors along the eight highest-variance (solid) and eight lowest-variance (dashed) eigendirections at $t\in\{0.50,0.75,0.90\}$.
        \textbf{Middle:} The directional imbalance ratio
        $R(t,e)=
        \operatorname{err}_{\mathrm{bottom}\text{-}8}(t,e)/
        \operatorname{err}_{\mathrm{top}\text{-}8}(t,e)$;
        the dotted line denotes $R(t,e)=1$.
        \textbf{Bottom:} FID measured throughout the corresponding MNIST training trajectories.
        The directional diagnostics are computed within the latent representation used by each model, whereas FID is evaluated in the common decoded image space after applying the inverse preconditioner where necessary.
    }
    \label{fig:mnist_directional_diagnostics}
\end{figure}

\paragraph{Preconditioning reduces directional learning imbalance.}
\Cref{fig:mnist_directional_diagnostics} reveals a consistent difference between the directional learning dynamics with and without preconditioning. Without preconditioning, the errors along the high-variance eigendirections decrease rapidly and reach a low plateau, whereas the errors along the low-variance eigendirections decay much more slowly. This separation is present across the tested values of $t$ and becomes progressively more pronounced as $t$ approaches the data endpoint, where the intermediate distribution inherits more of the anisotropy of the target latent distribution. Consequently, the directional imbalance ratio $R(t,e)$ can remain substantially larger than one, especially at larger values of $t$, indicating that good recovery of the high-variance subspace does not necessarily imply equally accurate recovery of the low-variance subspace.

With NF preconditioning, the high- and low-variance directional errors rapidly approach comparable scales for all tested values of $t$. Accordingly, the imbalance ratios remain close to one throughout most of training and exhibit no persistent growth as the interpolation approaches the data endpoint. Importantly, this improved balance is not obtained by degrading the high-variance directions: their errors remain low, while the errors along the low-variance directions are reduced substantially. These results are consistent with the proposed mechanism that preconditioning improves the geometry of the FM path and enables directions receiving comparatively weak gradient signal in the original representation to be learned at rates comparable to the dominant high-variance directions.

\paragraph{Relation to generation quality.}
The bottom row of \Cref{fig:mnist_directional_diagnostics} reports FID along the same training trajectories. NF preconditioning achieves a best FID of
$2.74$, compared with
$5.97$ without preconditioning. NF preconditioning also achieves lower FID throughout training. 

We do not directly compare raw FM losses between the two models because they regress velocity targets defined in different representation spaces. Instead, the directional error measures recovery relative to the appropriate moment-matched Gaussian reference within each representation, while the imbalance ratio measures the relative discrepancy between low- and high-variance directions within the same run. FID is computed after mapping generated samples back to the original latent space and decoding them, making it directly comparable across the two models.

Together with the path-conditioning measurements in \Cref{appsubsec:condition_number_mnist}, these results provide an empirical connection between the linear theory and nonlinear flow matching on real data. Preconditioning reduces the anisotropy of the empirical FM path, makes learning more balanced across covariance eigendirections, and is accompanied by improved generation quality. This experiment is intended as a mechanistic diagnostic rather than a global convergence result for arbitrary nonlinear velocity fields.

\subsection{Preconditioner quality ablation: capacity and training iterations}
\label{appsec:preconditioner_quality_ablation}
We ablate the quality of the learned preconditioner on MNIST by varying two factors: 
(i) the capacity of the preconditioner, measured by its number of trainable parameters, and 
(ii) the amount of preconditioner training, measured by the number of training epochs. 
All evaluations are performed using 50k generated MNIST samples. 
The goal of this ablation is to test whether the gains from preconditioning require a highly accurate or highly overparameterized preconditioner, or whether a moderately trained, moderate-capacity preconditioner is already sufficient to improve downstream flow matching generation.

\begin{figure}[h]
    \centering
    \begin{minipage}{0.49\linewidth}
        \centering
        \includegraphics[width=\linewidth]{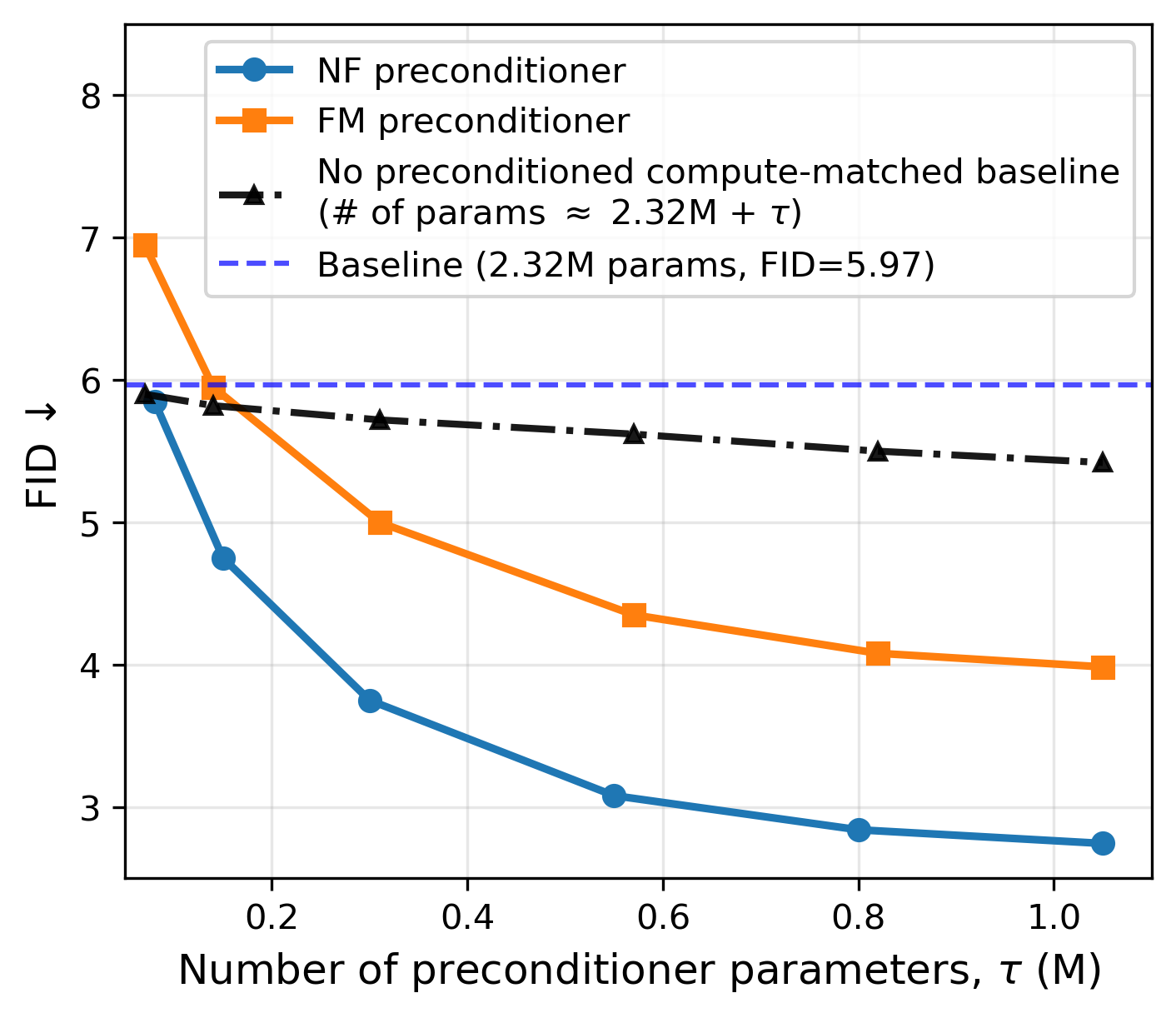}
        \vspace{-0.5em}
        \caption*{\small (a) Effect of preconditioner capacity.}
    \end{minipage}
    \hfill
    \begin{minipage}{0.49\linewidth}
        \centering
        \includegraphics[width=\linewidth]{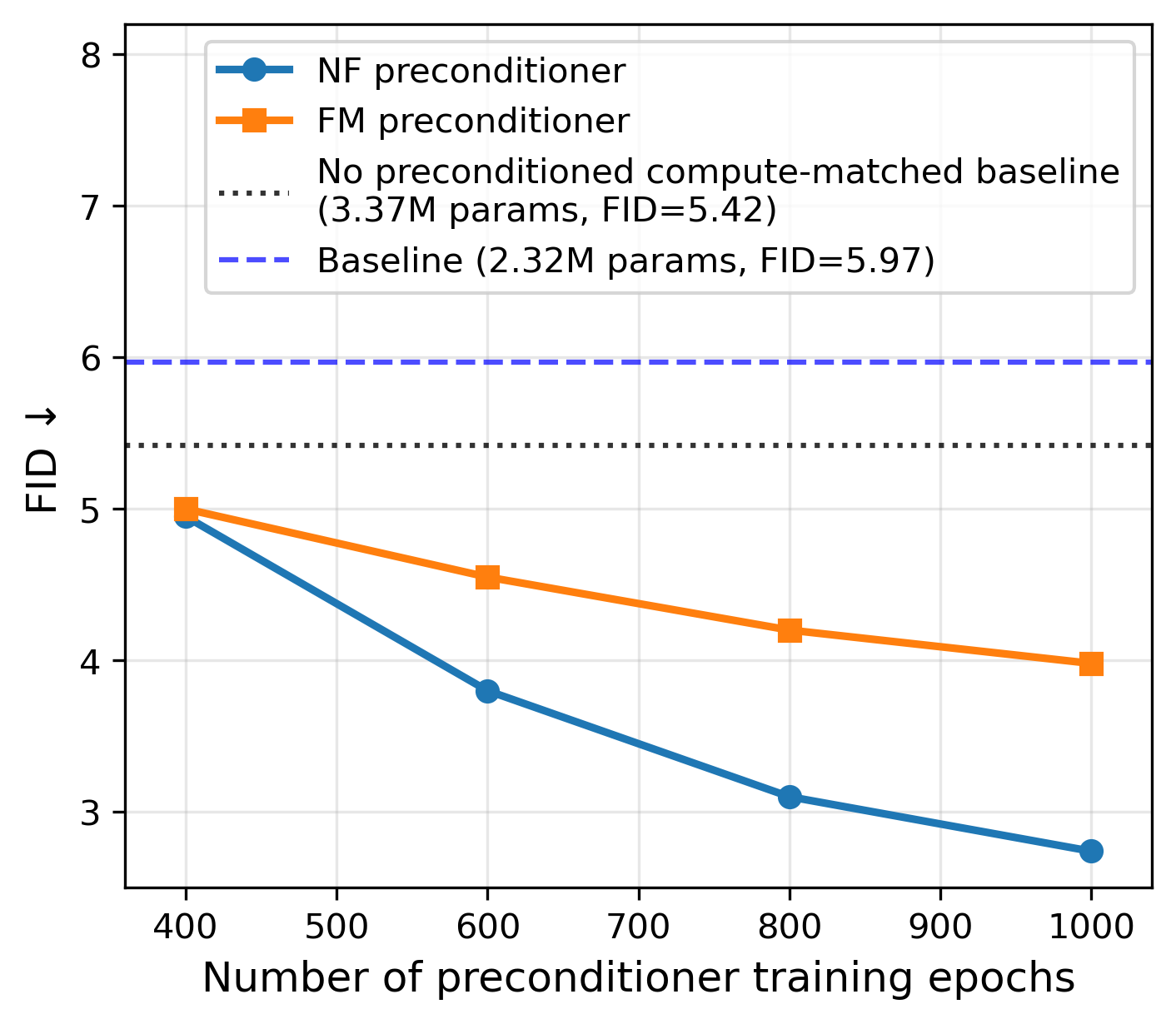}
        \vspace{-0.5em}
        \caption*{\small (b) Effect of preconditioner training iterations.}
    \end{minipage}
    \caption{
    \textbf{Preconditioner quality ablation on MNIST.}
    We evaluate image generation quality as a function of preconditioner capacity and training progress using FID computed on 50k generated samples.
    \textbf{(a)} Increasing the number of preconditioner parameters improves generation quality for both normalizing flow (NF) and flow matching (FM) preconditioners. The blue dashed line denotes the original unpreconditioned FM baseline with 2.32M parameters, and the black dash-dotted curve shows a compute-matched unpreconditioned baseline whose parameter count is increased by approximately the same number of parameters as the preconditioner. Preconditioning improves FID more substantially than simply allocating the same additional parameter budget to the unpreconditioned generator.
    \textbf{(b)} Increasing the number of preconditioner training epochs also improves FID, with the gains eventually beginning to saturate. The dotted horizontal line shows the compute-matched unpreconditioned baseline with 3.37M parameters, and the dashed horizontal line shows the original 2.32M-parameter baseline. Across both ablations, NF preconditioning yields the strongest gains, while FM preconditioning also consistently improves over the unpreconditioned baselines.
    }
    \label{fig:preconditioner_quality_ablation}
\end{figure}

\Cref{fig:preconditioner_quality_ablation} shows that the effectiveness of preconditioning depends on both the capacity and the training quality of the preconditioner. 
In \Cref{fig:preconditioner_quality_ablation} (a), increasing the number of preconditioner parameters steadily improves FID for both NF and FM preconditioners. For each preconditioner-capacity point, the preconditioner was trained to convergence using periodically evaluated FID10k as the stopping/selection criterion. We used the checkpoint with the best converged FID10k for the downstream FM model, so the ablation reflects the effect of preconditioner capacity rather than differences in training duration or undertraining. This suggests that a more expressive preconditioner can better reshape the data distribution into a geometry that is easier for the downstream flow matching model to learn. 
However, the improvement is not simply due to adding parameters: the compute-matched unpreconditioned baseline, which receives a comparable increase in parameter count, improves only marginally relative to the original baseline. 
Thus, the benefit comes primarily from the geometric effect of preconditioning rather than from model capacity alone.

\Cref{fig:preconditioner_quality_ablation} (b) shows a complementary trend when the preconditioner capacity is fixed (1.05M parameters) and the number of preconditioner training epochs is varied. 
Poorly trained preconditioners provide only moderate improvement, while better-trained preconditioners lead to substantially lower FID. 
The gains eventually begin to plateau, indicating that the downstream flow matching model does not require a perfectly converged preconditioner to benefit from the transformed geometry. 
This behavior is desirable in practice: preconditioning can improve generation even when the preconditioner is trained only to moderate accuracy.

Across both ablations, NF preconditioning achieves the lowest FID, while FM preconditioning still consistently improves over the unpreconditioned and compute-matched baselines. 
These results support the view that preconditioning is not merely an architectural enlargement of the generator. 
Instead, the preconditioner changes the geometry of the learning problem in a way that improves downstream flow matching optimization and sample quality.

\subsection{Comparison with Rectified flow and Reflow}
\label{appsec:reflow-comparison}
Preconditioning and Reflow \citep{RectifiedFlow_liu2022flowstraightfastlearning} modify different aspects of the transport
problem. Preconditioning reshapes the target distribution before training
the main flow, whereas Reflow generates model-induced source -- target pairs
and retrains a Rectified Flow (RF) to obtain straighter trajectories
\citep{RectifiedFlow_liu2022flowstraightfastlearning}. We compare both approaches on MNIST using one Reflow round and 60 function evaluations (NFE) for the main flow.

\begin{table*}[h]
    \centering
    \setlength{\tabcolsep}{3.5pt}
    \renewcommand{\arraystretch}{1.15}
    \caption{
        Comparison with Rectified flow and Reflow on MNIST.
        NFE counts ODE vector-field evaluations.
        $^\dagger$The NF inverse is a direct network evaluation and is
        therefore not included in the ODE NFE.
    }
    \label{tab:reflow-comparison}

    \resizebox{\textwidth}{!}{%
        \begin{tabular}{
            @{}
            L{3.5cm}
            L{2.5cm}
            L{3.1cm}
            C{1.3cm}
            C{1.6cm}
            C{1.1cm}
            C{0.9cm}
            @{}
        }
            \toprule
            Method
            & Additional training
            & Additional inference operation
            & \shortstack{Main-flow\\NFE}
            & \shortstack{Inverse-prec.\\NFE}
            & \shortstack{Total\\NFE}
            & FID $\downarrow$ \\
            \midrule

            FM
            & --
            & --
            & 60
            & 0
            & 60
            & 5.97 \\
            \midrule

            NF preconditioning (Ours)
            & Train NF preconditioner
            & One direct NF inverse pass
            & 60
            & N/A
            & $60^\dagger$
            & \textbf{2.74} \\

            FM preconditioning (Ours)
            & Train FM preconditioner
            & Integrate inverse preconditioner
            & 60
            & 60
            & 120
            & 3.98 \\
            \midrule

            RF
            & --
            & --
            & 60
            & 0
            & 60
            & 6.86 \\

            RF + Reflow
            & Pair generation and RF retraining
            & --
            & 60
            & 0
            & 60
            & 5.15 \\

            \bottomrule
        \end{tabular}%
    }
\end{table*}

Reflow improves the RF baseline from an FID of $6.86$ to $5.15$. Nevertheless, both preconditioned variants achieve lower FID. NF preconditioning obtains the best quality--inference trade-off, reaching an FID of $2.74$ with 60 ODE evaluations and one direct inverse pass. FM preconditioning reaches an FID of $3.98$, but requires an additional 60 evaluations to integrate the inverse preconditioning flow. Reflow avoids this additional inference integration by placing its extra cost in pair generation and retraining. Thus, the comparison demonstrates better sample quality for target-space preconditioning, although it does not by itself establish lower total training cost because the additional training stages were not normalized by wall-clock time or training FLOPs.

\subsection{Operational preconditioner selection and limited-benefit regimes}
\label{appsec:preconditioner_selection}
We propose a practical preconditioner selection protocol. A practical selection procedure is to first estimate $\max_t \kappa(\Sigma_t)$ along the baseline FM path, including class- or
component-specific estimates when global covariance may conceal local anisotropy.
An affine or ZCA transformation is appropriate when the anisotropy is predominantly
global and approximately elliptical, whereas a nonlinear invertible map is preferable
when substantial local or componentwise anisotropy remains. In either case, we recommend
using the weakest or lowest-capacity preconditioner that materially reduces path
conditioning while keeping the inverse sensitivity
$\lVert J_{\mathcal{P}^{-1}}\rVert_2$ controlled.

\begin{table*}[h]
    \centering
    \setlength{\tabcolsep}{3pt}
    \renewcommand{\arraystretch}{1.15}
    \caption{
        Relationship between path-conditioning reduction and
        generation quality across all image datasets. Conditioning reduction
        is the ratio of the baseline maximum path condition number to its value
        after preconditioning.
    }
    \label{tab:preconditioner_selection}

    \resizebox{0.95\textwidth}{!}{%
    \begin{tabular}{@{}
        >{\raggedright\arraybackslash}m{2.05cm}
        >{\centering\arraybackslash}m{1.25cm}
        >{\centering\arraybackslash}m{2.05cm}
        >{\centering\arraybackslash}m{2.15cm}
        >{\centering\arraybackslash}m{1.65cm}
        >{\centering\arraybackslash}m{1.45cm}
        >{\centering\arraybackslash}m{1.45cm}
        >{\centering\arraybackslash}m{1.80cm}
        >{\centering\arraybackslash}m{1.90cm}
        @{}}
        \toprule
        \textbf{Dataset}
        &
        \shortstack{\textbf{Training}\\\textbf{space}}
        &
        \shortstack{\textbf{Baseline}\\
        $\max_t \kappa(\Sigma_t)\downarrow$}
        &
        \shortstack{\textbf{Precond.}\\
        $\max_t \kappa(\widetilde{\Sigma}_t)\downarrow$}
        &
        \shortstack{\textbf{Cond.}\\
        \textbf{reduction} $\uparrow$}
        &
        \shortstack{\textbf{FID: no}\\
        \textbf{precond.} $\downarrow$}
        &
        \shortstack{\textbf{FID: FM}\\
        \textbf{precond.} $\downarrow$}
        &
        \shortstack{\textbf{Relative FID}\\
        \textbf{gain} $\uparrow$}
        &
        \textbf{Diagnostic}
        \\
        \midrule

        MNIST
        & Latent
        & 52.83
        & 2.02
        & $26.2\times$
        & 5.97
        & 3.98
        & 33.3\%
        & \shortstack{Preconditioning\\beneficial}
        \\
        \midrule

        LSUN Churches
        & Latent
        & 102.63
        & 5.54
        & $18.53\times$
        & 19.53
        & 14.47
        & 25.9\%
        & \shortstack{Preconditioning\\beneficial}
        \\
        \midrule

        Oxford Flowers-102
        & Latent
        & 83.47
        & 47.92
        & $1.74\times$
        & 23.82
        & 22.64
        & 5.0\%
        & Limited benefit
        \\
        \midrule

        AFHQ Cats
        & Latent
        & 64.75
        & 2.36
        & $27.44\times$
        & 7.11
        & 4.86
        & 31.6\%
        & \shortstack{Preconditioning\\beneficial}
        \\
        \midrule

        CIFAR-10
        & Pixel
        & 755.92
        & 453.68
        & $1.67\times$
        & 4.15
        & 4.02
        & 3.1\%
        & Limited benefit
        \\
        \midrule

        ImageNet-1k
        & Latent
        & 84.37
        & 1.90
        & $44.4\times$
        & 8.66
        & 4.98
        & 42.5\%
        & \shortstack{Preconditioning\\beneficial}
        \\

        \bottomrule
    \end{tabular}%
    }
\end{table*}

\Cref{tab:preconditioner_selection} indicates that the achieved reduction in path
conditioning is more predictive of the FID benefit than the baseline condition number
alone. Across our datasets, reductions of approximately $18.5\times$--$44.4\times$
correspond to relative FID improvements of $25.9\%$--$42.5\%$, whereas reductions below
$2\times$ yield improvements of at most $5\%$. These ranges are empirical diagnostics
rather than universal thresholds. Preconditioning is therefore expected to provide
limited benefit when the baseline path is already well-conditioned, when a pilot
preconditioner produces only a small conditioning reduction, or when further reduction
requires an excessively sensitive inverse map.

In either case, we recommend using the weakest or lowest-capacity
preconditioner that materially reduces path conditioning while keeping
the inverse sensitivity $\|J_{\mathcal{P}^{-1}}\|_2$ controlled;
in \Cref{appsec:inverse-sensitivity}, we evaluate this trade-off directly.

\subsection{Inverse-preconditioner Jacobian sensitivity}
\label{appsec:inverse-sensitivity}

Let $z$ denote an endpoint in the preconditioned space. For a sufficiently
small endpoint error $\delta z$, the corresponding reconstruction error
satisfies
\begin{equation}
    \mathcal{P}^{-1}(z+\delta z)-\mathcal{P}^{-1}(z)
    =
    J_{\mathcal{P}^{-1}}(z)\,\delta z
    + \mathcal{O}\!\left(\|\delta z\|_2^2\right).
\end{equation}
Thus, the local worst-case absolute amplification is governed by
$\|J_{\mathcal{P}^{-1}}(z)\|_2
=\sigma_{\max}(J_{\mathcal{P}^{-1}}(z))$.
In contrast,
$\kappa(J_{\mathcal{P}^{-1}})
=\sigma_{\max}(J_{\mathcal{P}^{-1}})
 /\sigma_{\min}(J_{\mathcal{P}^{-1}})$
measures the variation in amplification across directions and does not,
by itself, determine the absolute amplification of an endpoint error.

\paragraph{Controlled GMM ablation.}
We first vary the affine preconditioning strength $\gamma$ on the bimodal
and anisotropic-ring GMMs. Here, $\gamma=0$ denotes no preconditioning,
whereas $\gamma=1$ denotes full whitening. \Cref{tab:gmm-inverse-sensitivity}
reports the resulting path conditioning, inverse-map sensitivity, and
generation quality.

\begin{table*}[h]
    \centering
    \small
    \setlength{\tabcolsep}{5pt}
    \caption{
        Trade-off between path conditioning, inverse sensitivity, and
        generation quality under partial affine preconditioning.
        Path conditioning and MMD are reported as mean $\pm$ standard
        deviation over three seeds. Since the preconditioner is affine,
        its inverse-Jacobian quantities are constant over space.
        Bold values of $\gamma$ identify the lowest final MMD for each target.
    }
    \label{tab:gmm-inverse-sensitivity}
    \resizebox{0.8\textwidth}{!}{%
    \begin{tabular}{lccccc}
        \toprule
        Distribution
        & $\gamma$
        & \shortstack{Peak path\\$\max_t\kappa(\Sigma_t)\downarrow$}
        & $\|J_{\mathcal{P}_{\gamma}^{-1}}\|_2\downarrow$
        & $\kappa(J_{\mathcal{P}_{\gamma}^{-1}})\downarrow$
        & \shortstack{Final MMD\\$(\times 10^{-3})\downarrow$} \\
        \midrule
        Bimodal
        & $0.00$ & $849.75 \pm 11.14$ & $1.00$  & $1.00$  & $114.69 \pm 4.26$ \\
        & $0.25$ & $160.02 \pm 2.05$  & $2.18$  & $2.31$  & $202.96 \pm 10.04$ \\
        & $\mathbf{0.50}$ & $30.14 \pm 0.37$ & $4.76$ & $5.34$ & $55.74 \pm 17.32$ \\
        & $0.75$ & $5.70 \pm 0.06$ & $10.39$ & $12.35$ & $66.88 \pm 17.59$ \\
        & $1.00$ & $1.21 \pm 0.02$ & $22.66$ & $28.53$ & $324.64 \pm 5.02$ \\
        \midrule
        Anisotropic ring
        & $0.00$ & $7.95 \pm 0.17$ & $1.00$ & $1.00$ & $20.46 \pm 2.90$ \\
        & $\mathbf{0.25}$ & $4.71 \pm 0.09$ & $1.63$ & $1.29$ & $14.01 \pm 4.78$ \\
        & $0.50$ & $2.80 \pm 0.03$ & $2.66$ & $1.67$ & $20.26 \pm 5.72$ \\
        & $0.75$ & $1.68 \pm 0.02$ & $4.34$ & $2.16$ & $24.48 \pm 10.49$ \\
        & $1.00$ & $1.19 \pm 0.01$ & $7.08$ & $2.79$ & $53.78 \pm 15.70$ \\
        \bottomrule
    \end{tabular}
}
\end{table*}

Full whitening minimizes the path condition number but also produces the
largest inverse amplification and worse final MMD. The best performance is
instead obtained with partial preconditioning. For the bimodal target,
$\gamma=0.50$ reduces peak path conditioning by $28.2\times$ and MMD by
$51.4\%$, while limiting the inverse gain to $4.76$. For the anisotropic
ring, $\gamma=0.25$ reduces MMD by $31.5\%$ with an inverse gain of $1.63$.
Controlled perturbations of magnitude $10^{-3}$ relative to the endpoint
RMS produce negligible changes in MMD.

\paragraph{Learned nonlinear preconditioner on MNIST.}
We next evaluate the learned nonlinear MNIST preconditioner at both
preconditioned real samples and FM-generated endpoints. At each point, we
compute the singular values of the exact inverse Jacobian and measure
finite-difference amplification along random directions and along the
leading right singular vector of the Jacobian.

\begin{table*}[h]
    \centering
    \small
    \setlength{\tabcolsep}{6pt}
    \caption{
        Local sensitivity of the learned nonlinear MNIST inverse.
        Amplification is measured as
        $\|\mathcal{P}^{-1}(z+\delta z)-\mathcal{P}^{-1}(z)\|_2/
        \|\delta z\|_2$.
        The worst direction is the leading right singular vector of the
        inverse Jacobian; p95 denotes the $95$th percentile.
    }
    \label{tab:mnist-inverse-sensitivity}
    \resizebox{\textwidth}{!}{%
    \begin{tabular}{lcccc}
        \toprule
        \shortstack{Points\\evaluated}
        & \shortstack{$\|J_{\mathcal{P}^{-1}}\|_2$\\median / p95 / max}
        & \shortstack{$\kappa(J_{\mathcal{P}^{-1}})$\\median / p95}
        & \shortstack{Random amplification\\median / p95}
        & \shortstack{Worst-direction amplification\\median / p95} \\
        \midrule
        Real preconditioned samples
        & $2.41\,/\,2.99\,/\,3.33$
        & $275\,/\,388$
        & $0.60\,/\,0.88$
        & $2.41\,/\,2.99$ \\
        Generated endpoints
        & $2.32\,/\,3.01\,/\,4.65$
        & $273\,/\,379$
        & $0.66\,/\,0.95$
        & $2.32\,/\,3.01$ \\
        \bottomrule
    \end{tabular}
}
\end{table*}

Across the evaluated MNIST points, the median and p95 worst-direction gains
are approximately $2.3$--$3.0$, with a maximum observed gain of $4.65$.
Random perturbations are typically attenuated, and the sensitivities at
real and generated endpoints are similar. Notably, the large Jacobian
condition numbers coexist with moderate spectral norms, illustrating that
directional anisotropy does not imply equally large absolute error
amplification. These results support selecting a partial preconditioner
that balances improved path conditioning against controlled inverse-map
sensitivity.

\subsection{Preconditioning for conditional image restoration tasks}
\label{appsec:conditional_restoration}

We investigate whether preconditioning extends to conditional flow-matching tasks in
which the model must preserve measurement information. We consider two lightweight
inverse problems on MNIST, both defined by
\begin{equation}
    b = A x + \epsilon,
    \qquad
    \epsilon \sim \mathcal{N}(0,s^2 I),
\end{equation}
where the noise level $s$ is sampled uniformly between $0\%$ and $10\%$ of the
measurement dynamic range. We evaluate Gaussian deblurring with blur standard deviation
$\sigma=3$ pixels and $4\times$ super-resolution, for which $A$ is a
factor-$4$ downsampling operator.

\paragraph{Conditional flow-matching formulation.}
The baseline is a conditional velocity field, $v_\theta(x_t,t,c)$ with $c=(A^\top b,s)$, implemented using a multiscale UNet. We used the formulation from \citet{ahamed2024dawn} for solving inverse problems using flow matching. The estimated endpoint,
\begin{equation}
    \widehat{x}_1
    =
    x_t+(1-t)v_\theta(x_t,t,c)
\end{equation}
is regularized using a measurement-consistency penalty
$\|A\widehat{x}_1-b\|_2^2$. The preconditioned model follows Algorithm~1 and provides
the same condition $c$ to both the preconditioner and the main FM model. The main model
is trained with the analogous latent-space consistency penalty
$\|\widetilde{z}_1-\widetilde{z}_1^{\parallel}\|_2^2$, where
$\widetilde{z}_1^{\parallel}$ denotes the measurement-consistent latent endpoint.

We compare against the original conditional FM baseline and unpreconditioned controls
matched to either the total parameter count or the total training FLOPs of the
preconditioned pipeline. Reconstruction quality is measured using per-sample SSIM and
PSNR of the posterior mean. We additionally report the posterior standard deviation as
a measure of sample diversity and the measurement misfit.

\begin{table*}[h]
    \centering
    \caption{
        Conditional image restoration on MNIST. ``Params/FLOPs'' reports the
        parameter count and training FLOPs relative to the baseline. For our method,
        the two parameter counts correspond to the main FM model and preconditioner.
        PSNR is computed from the posterior mean, while posterior standard deviation
        summarizes sample diversity. Higher SSIM and PSNR are better; lower
        measurement misfit is better.
    }
    \label{tab:conditional_restoration}
    \resizebox{\textwidth}{!}{%
    \begin{tabular}{@{}lccccc@{}}
        \toprule
        \multicolumn{6}{c}{\textbf{Deblurring}} \\
        \midrule
        \textbf{Method}
        & \textbf{Params/FLOPs}
        & \textbf{SSIM} $\boldsymbol{\uparrow}$
        & \shortstack{\textbf{PSNR}\\\textbf{(posterior mean)} $\boldsymbol{\uparrow}$}
        & \shortstack{\textbf{Posterior}\\\textbf{std.}}
        & \shortstack{\textbf{Misfit}\\$(10^{-3})$ $\boldsymbol{\downarrow}$}
        \\
        \midrule
        Baseline FM
        & $1.07\mathrm{M}/1.0\times$
        & 0.803
        & 22.4
        & 0.135
        & \textbf{5.8}
        \\
        Parameter-matched
        & $1.50\mathrm{M}/1.0\times$
        & 0.756
        & \textbf{22.6}
        & 0.156
        & \textbf{5.8}
        \\
        FLOPs-matched
        & $1.07\mathrm{M}/1.27\times$
        & 0.817
        & 22.5
        & 0.128
        & \textbf{5.8}
        \\
        \textbf{Preconditioned (Ours)}
        & $(1.07+0.44)\mathrm{M}/1.27\times$
        & \textbf{0.834}
        & 21.7
        & 0.128
        & 6.2
        \\
        \midrule
        \multicolumn{6}{c}{\textbf{$4\times$ super-resolution}} \\
        \midrule
        \textbf{Method}
        & \textbf{Params/FLOPs}
        & \textbf{SSIM} $\boldsymbol{\uparrow}$
        & \shortstack{\textbf{PSNR}\\\textbf{(posterior mean)} $\boldsymbol{\uparrow}$}
        & \shortstack{\textbf{Posterior}\\\textbf{std.}}
        & \shortstack{\textbf{Misfit}\\$(10^{-3})$ $\boldsymbol{\downarrow}$}
        \\
        \midrule
        Baseline FM
        & $1.07\mathrm{M}/1.0\times$
        & 0.685
        & \textbf{19.8}
        & 0.195
        & \textbf{9.6}
        \\
        \textbf{Preconditioned (Ours)}
        & $(1.07+0.44)\mathrm{M}/1.27\times$
        & \textbf{0.708}
        & 19.2
        & 0.203
        & 15.4
        \\
        \bottomrule
    \end{tabular}%
    }
\end{table*}

\paragraph{Restoration performance.}
As shown in \Cref{tab:conditional_restoration}, preconditioning achieves the best
SSIM on deblurring, outperforming both the parameter-matched and FLOPs-matched
unpreconditioned controls. The posterior standard deviation is preserved and the
measurement misfit remains within approximately $7\%$ of the baseline. The
preconditioned model also reaches comparable early reconstruction quality in
approximately $2$--$3\times$ fewer epochs.

The super-resolution experiment reproduces the SSIM improvement at the same
$1.27\times$ training-FLOPs ratio, while maintaining comparable posterior diversity.
The improvement in structural similarity is accompanied by a modest reduction in
posterior-mean PSNR and a larger measurement misfit, indicating that measurement
consistency does not transfer as cleanly as in the deblurring experiment. Overall,
the results demonstrate that preconditioning can improve conditional reconstruction
quality without removing the diversity of the learned conditional distribution, while
also revealing a task-dependent trade-off between perceptual fidelity and strict data
consistency.

\paragraph{Conditioning diagnostics under conditional generation.}
For the unpreconditioned model, the path condition number increases from
$\kappa(\Sigma_t)=3.9$ at $t=0.1$ to $7.0\times10^3$ at $t=0.9$.
Preconditioning reduces these values to $3.6$ and $1.7\times10^3$, respectively.
The marginal target condition number is also reduced from $8.0\times10^7$ to
$7.8\times10^4$ for deblurring and to $6.4\times10^4$ for super-resolution.

Conditioning the preconditioner itself is important. In an ablation, removing the
condition produces a slightly more isotropic marginal distribution
($5.7\times10^4$ versus $7.5\times10^4$) but substantially poorer reconstructions.
Using the digit class as a proxy for the measurement condition, the worst
condition-specific condition number is also larger for the unconditional
preconditioner than for the conditional one
($9.5\times10^5$ versus $4.3\times10^5$). Thus, marginal isotropy alone can conceal
poor geometry within individual conditional distributions.

We further measure the sensitivity of the inverse preconditioner using
\begin{equation}
    G
    =
    \frac{
        \left\|
            P^{-1}(\widetilde{z}+\delta,c)
            -
            P^{-1}(\widetilde{z},c)
        \right\|_2
    }{
        \|\delta\|_2
    }.
\end{equation}
The unconditional preconditioner amplifies perturbations, with $G=2.3$, whereas the
conditional preconditioner is contractive, with $G=0.37$, and remains stable for
perturbations between $1\%$ and $10\%$ of $\|\widetilde{z}\|_2$. These results show
that conditioning information does not undermine the benefits of preconditioning;
rather, supplying the condition to the preconditioner is important for preserving
condition-specific geometry and inverse stability. Consequently, conditional
preconditioners should be selected using worst-case condition-specific conditioning
and inverse-map sensitivity, rather than marginal isotropy alone.

\section{Qualitative Comparison for Flow-Based Preconditioning}
\label{qualititive_comparison_images}
\subsection{MNIST}
\label{qualitative_comparison_mnist}
We provide qualitative comparison for the flow matching in the latent space of an MNIST-trained VAE (latent dimension 64, see \Cref{appsubsec:mnist} for details) when the latent space was not preconditioned (left) (FID = 5.97), preconditioned using a normalizing flow (middle) (FID = 2.74), and using a flow matching model (right) (FID = 3.98) in \Cref{fig:mnist_grid_comparison}. The visual quality of the samples is consistent with the FID scores, with sharper and more coherent digits observed in the preconditioned cases. We also analyzed the condition number dynamics as a function of time $t$ during the training of these flow models, which are presented in \Cref{fig:kappa_vs_t_mnist}.

\begin{figure}[h]
    \centering
    \includegraphics[width=\textwidth]{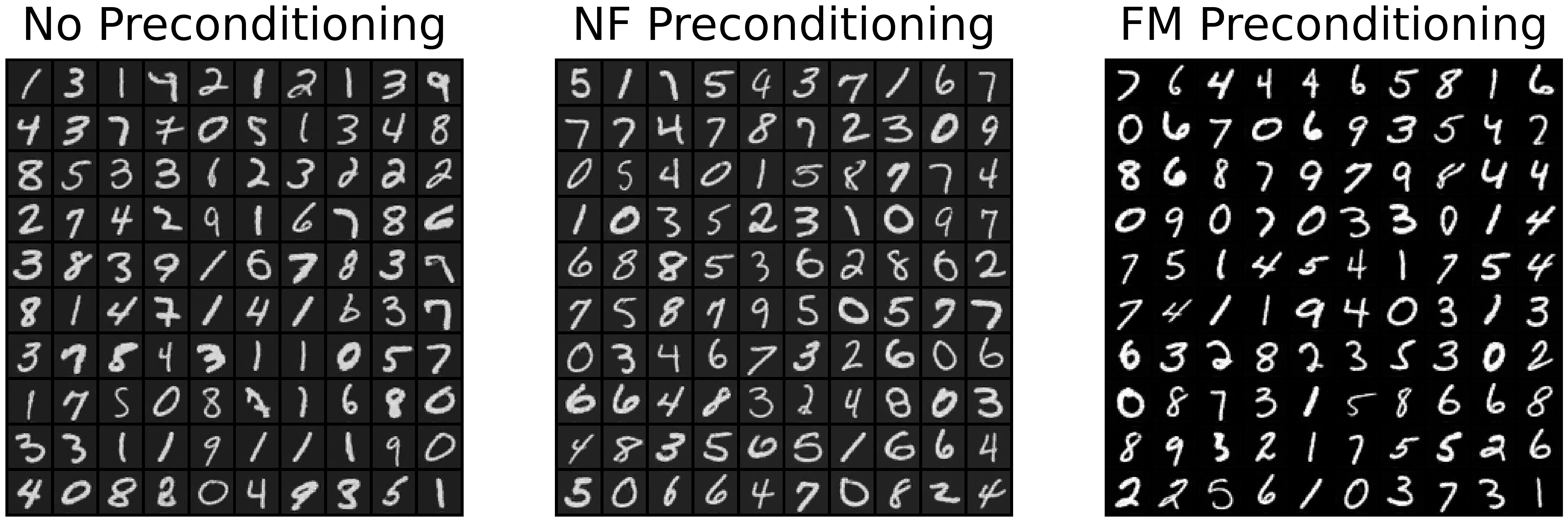}
    \caption{\textbf{MNIST qualitative comparison.} Samples generated by latent-space flow matching under different preconditioning strategies. \textbf{Left:} The no-preconditioning baseline produces recognizable digits but shows occasional shape inconsistencies and reduced sharpness (FID = 5.97). \textbf{Middle:} NF preconditioning improves sample fidelity and yields cleaner digit shapes (FID = 2.74). \textbf{Right:} FM preconditioning also improves over the baseline, producing more consistent and diverse digits (FID = 3.98). Overall, the qualitative trends are consistent with the quantitative FID improvements.}
    \label{fig:mnist_grid_comparison}
\end{figure}

\subsection{Image samples across resolutions}
\label{qualitative_comparison_highres_images}

We provide qualitative comparisons illustrating the effect of flow-based preconditioning on sample quality for LSUN Churches \citep{yu2015lsun}, Oxford Flowers-102 \citep{nilsback2008automated}, AFHQ Cats \citep{choi2020starganv2}, CIFAR-10 \citep{krizhevsky2009learning} and ImageNet-1k \citep{chrabaszcz2017downsampled}. As shown in \Cref{fig:flowers_comparison,fig:churches_comparison,fig:cats_comparison,fig:cifar10_comparison,fig:imagenet64_qualitative_comparison_1,fig:imagenet64_qualitative_comparison_2}, these examples complement the quantitative improvements reported in \Cref{tab:high_res_image_synthesis_results_table} and illustrate how the additional flow improves sample consistency in cases where standard flow matching often remains visually imperfect.

\begin{figure*}[h]
    \centering
    \includegraphics[width=\textwidth]{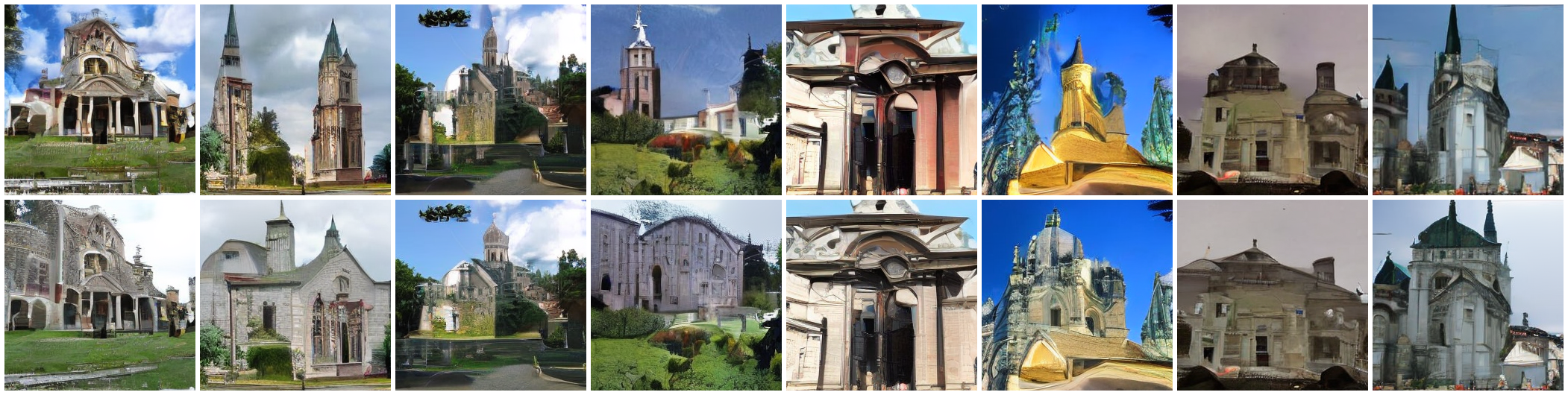}
    \caption{
        \textbf{LSUN Churches qualitative comparison.}
        Samples generated by standard flow matching (\emph{top}) and by flow-based preconditioning
        using an additional flow (\emph{bottom}). Preconditioning produces a sharper global structure
        and more stable spatial layouts.
    }
    \label{fig:churches_comparison}
\end{figure*}

\begin{figure*}[h]
    \centering
    \includegraphics[width=\textwidth]{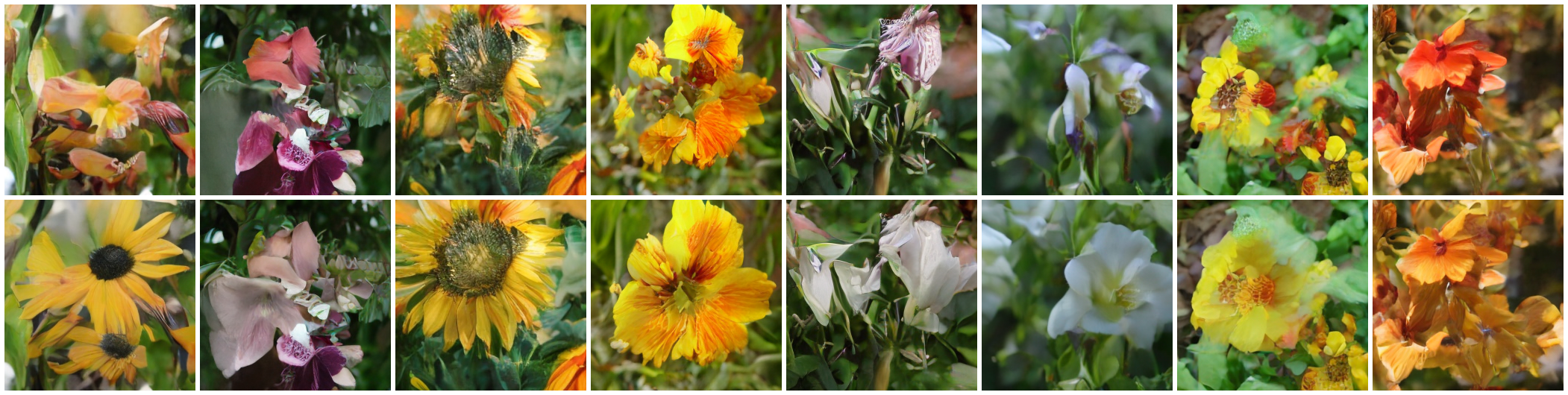}
    \caption{
        \textbf{Oxford Flowers-102 qualitative comparison.}
        Samples generated by standard flow matching (\emph{top}) and by flow-based preconditioning
        using an additional flow (\emph{bottom}). Preconditioning yields more coherent structure
        and improved visual consistency.
    }
    \label{fig:flowers_comparison}
\end{figure*}

\begin{figure*}[t]
    \centering
    \includegraphics[width=\textwidth]{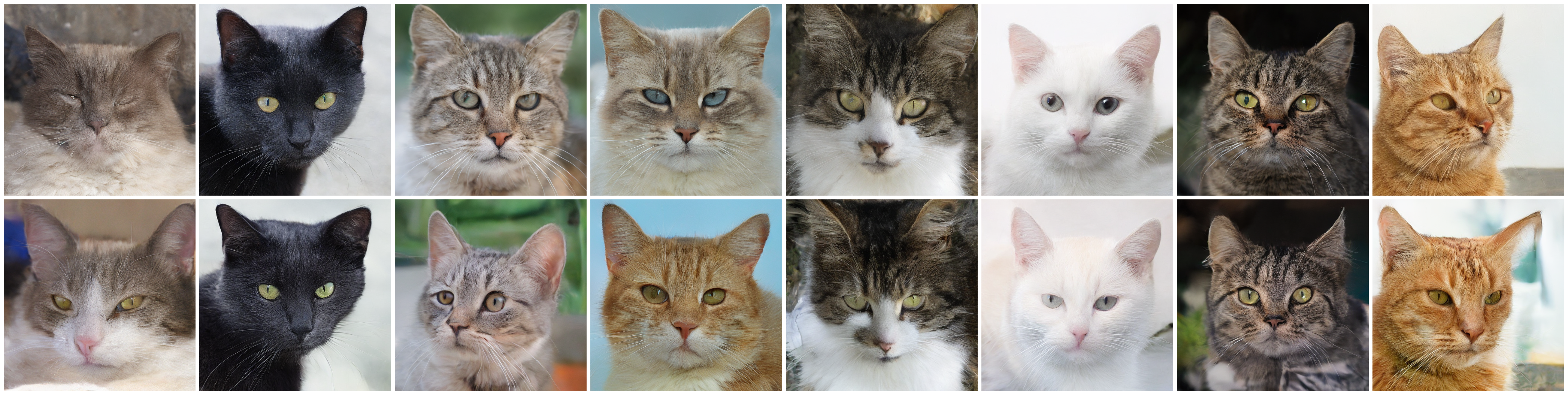}
    \caption{
        \textbf{AFHQ Cats qualitative comparison.}
        Samples generated by standard flow matching (\emph{top}) and by flow-based preconditioning
        using an additional flow (\emph{bottom}).
    }
    \label{fig:cats_comparison}
\end{figure*}

\begin{figure*}[]
    \centering
    \includegraphics[width=\textwidth]{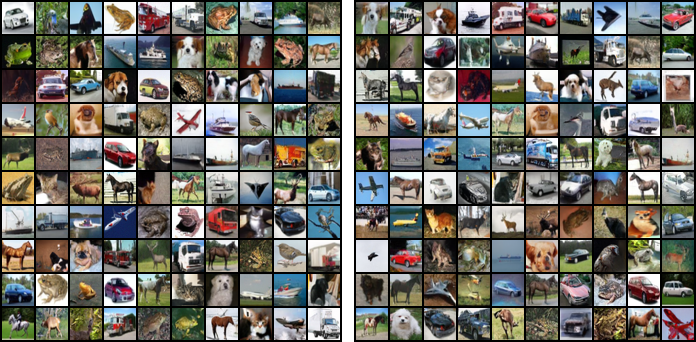}
    \caption{
        \textbf{CIFAR-10 qualitative results.}
        Samples generated by standard flow matching (\emph{left}) and by flow-based preconditioning using an additional flow (\emph{right}).
    }
    \label{fig:cifar10_comparison}
\end{figure*}

\begin{figure}[tbh!]
    \centering

    \begin{subfigure}{0.70\linewidth}
        \centering
        \includegraphics[width=\linewidth]{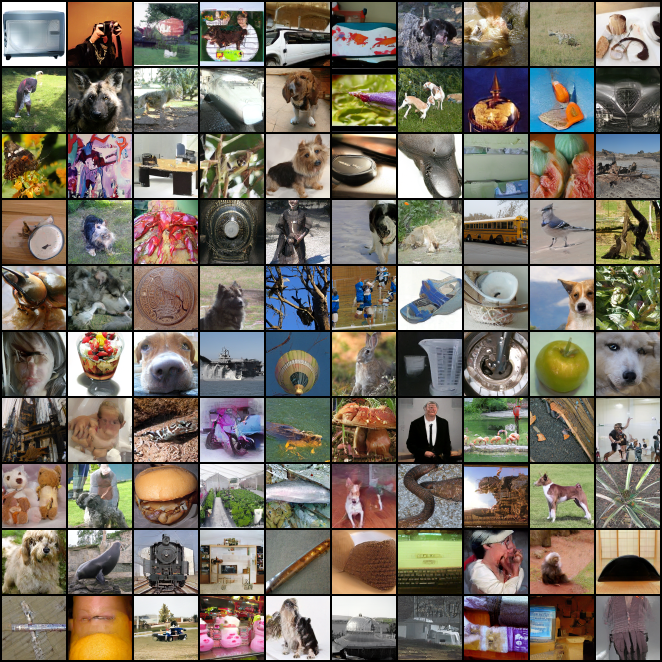}
        \label{fig:imagenet64_baseline_samples_1}
    \end{subfigure}

    \vspace{0.2em}

    \begin{subfigure}{0.70\linewidth}
        \centering
        \includegraphics[width=\linewidth]{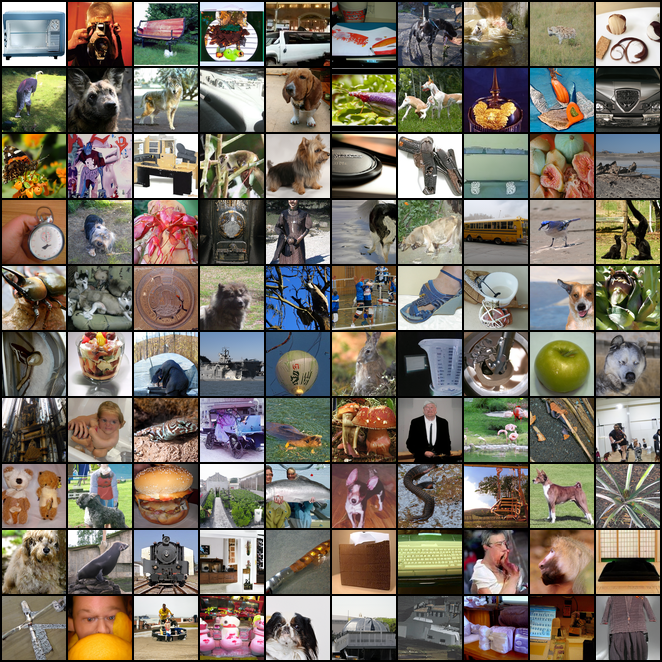}
        \label{fig:imagenet64_preconditioned_samples_1}
    \end{subfigure}

    \caption{\textbf{ImageNet-1k qualitative comparison.}
    Samples generated by standard flow matching \textit{(top)} and by flow-based preconditioning using an additional flow \textit{(bottom)}. Preconditioning improves visual consistency and sample fidelity while using the same downstream generation setting.}
    \label{fig:imagenet64_qualitative_comparison_1}
\end{figure}

\begin{figure}[tbh!]
    \centering

    \begin{subfigure}{0.70\linewidth}
        \centering
        \includegraphics[width=\linewidth]{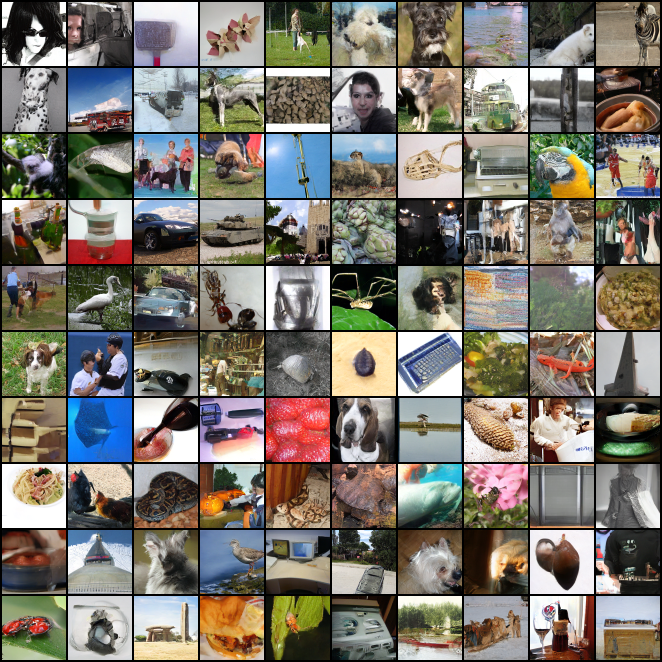}
        \label{fig:imagenet64_baseline_samples_2}
    \end{subfigure}

    \vspace{0.2em}

    \begin{subfigure}{0.70\linewidth}
        \centering
        \includegraphics[width=\linewidth]{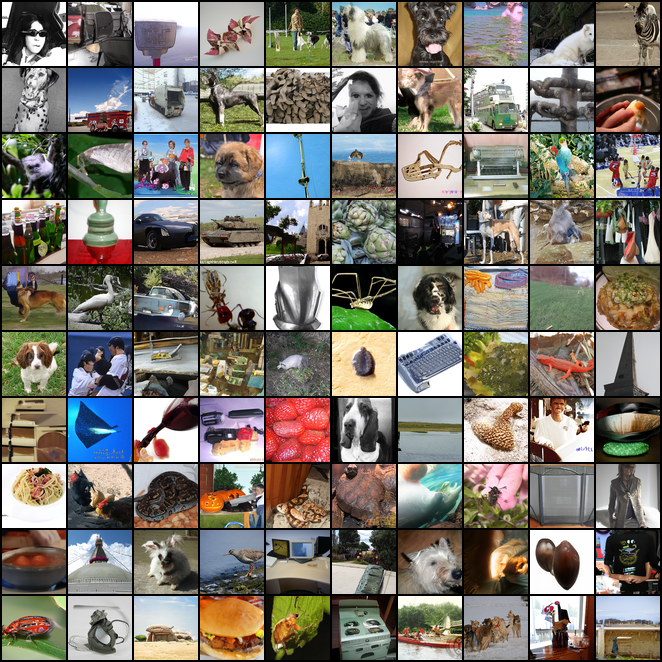}
        \label{fig:imagenet64_preconditioned_samples_2}
    \end{subfigure}

    \caption{\textbf{ImageNet-1k qualitative comparison.}
    Samples generated by standard flow matching \textit{(top)} and by flow-based preconditioning using an additional flow \textit{(bottom)}. Preconditioning improves visual consistency and sample fidelity while using the same downstream generation setting.}
    \label{fig:imagenet64_qualitative_comparison_2}
\end{figure}

\clearpage 

\section{Experimental Details}
\label{experimental_settings}

\subsection{MNIST}
\label{appsubsec:mnist}
We evaluate preconditioned flow matching on MNIST dataset. All MNIST experiments are performed in the latent space of a trained VAE. This VAE consisted of a deep ResNet-based architecture with group normalization, SiLU activations, and was trained using a combination of MSE and perceptual loss. The encoder and decoder consist of multiple downsampling and upsampling ResBlocks, with a latent dimensionality of 64. VAE training used cosine learning rate decay, KL regularization with low $\beta = 0.05$. Generated latent samples are decoded back to image space using the VAE decoder for evaluation.

We compare three variants: standard latent flow matching with no preconditioning, normalizing flow (NF) preconditioning, and flow matching (FM) preconditioning. In all three cases, the main generator is a class-conditional CNN-based latent vector field with 256 hidden channels, 7 convolutional layers, and one-hot class conditioning through a label MLP. The main generator has approximately 2.32M parameters across all variants, so differences in performance are not caused by changing the main FM architecture.

The two preconditioned variants differ in how the auxiliary transformation is learned. For NF preconditioning, we use a conditional RealNVP model with affine coupling layers, hidden MLP width 256, and roughly 1.05M parameters. For FM preconditioning, we first train a smaller class-conditional CNN vector field with 169 hidden channels and roughly 1.05M parameters, then train the main FM generator on the transformed latent distribution. Thus, the preconditioning stage is substantially cheaper than the main generator while still reshaping the geometry of the downstream flow matching problem.

The exact architectures, training spaces, interpolation functions, latent dimensions, learning rates, batch sizes, training epochs, and model sizes are summarized in \Cref{tab:mnist_experimental_table}.

\begin{table}[h]
\centering
\caption{Experimental settings and hyperparameters for MNIST latent flow matching experiments.}
\label{tab:mnist_experimental_table}
\scriptsize
\resizebox{\linewidth}{!}{%
\begin{tabular}{@{}lccc@{}}
\toprule
\textbf{Setting}
& \begin{tabular}[c]{@{}c@{}}\textbf{No Preconditioning} \end{tabular}
& \begin{tabular}[c]{@{}c@{}}\textbf{NF Preconditioning}\\ \end{tabular}
& \begin{tabular}[c]{@{}c@{}}\textbf{FM Preconditioning}\\ \end{tabular} \\
\midrule

Training space
& Latent
& Latent
& Latent \\

Latent dim
& 64 
& 64 
& 64  \\

\multirow{2}{*}{Interpolation}
& $s(t)=\sin(\pi t / 2)$
& $s(t)=\sin(\pi t / 2)$
& $s(t)=\sin(\pi t / 2)$\\

& $c(t)=\cos(\pi t / 2)$
& $c(t)=\cos(\pi t / 2)$
& $c(t)=\cos(\pi t / 2)$\\

\midrule
\multicolumn{4}{@{}l}{\textbf{FM generator model}} \\
\midrule

Network
& CNN 
& CNN 
& CNN \\

Hidden channels
& 256
& 256
& 256 \\

Depth
& 7 conv layers
& 7 conv layers
& 7 conv layers \\

Class conditioning
& one-hot(10)+MLP
& one-hot(10)+MLP
& one-hot(10)+MLP \\

Approx. parameters
& 2.32M
& 2.32M
& 2.32M \\

Max epochs
& 1000 
& 1000 
& 1000 \\

Learning rate
& $2\times10^{-4}$
& $2\times10^{-4}$
& $1\times10^{-3}$ \\

Optimizer
& Adam
& Adam
& Adam \\

Batch size
& 4096
& 4096
& 1024 \\

\midrule
\multicolumn{4}{@{}l}{\textbf{Preconditioner model}} \\
\midrule

Network
& --
& RealNVP
& CNN \\

\multirow{2}{*}{Key architecture}
& --
& hidden MLP width = 256
& hidden channels = 169\\

& --
& coupling layers = 8
& depth = 7 conv layers\\

Class conditioning
& --
& one-hot(10)+MLP
& one-hot(10)+MLP \\

Approx. parameters
& --
& 1.05M
& 1.05M \\

Max epochs
& --
& 450
& 650 \\

Learning rate
& --
& $2\times10^{-4}$
& $1\times10^{-3}$ \\

Optimizer
& --
& Adam
& Adam \\

Batch size
& --
& 4096
& 1024 \\
\bottomrule
\end{tabular}%
}
\end{table}

\subsection{Image synthesis across resolutions}
\label{appsubsec:floweers_and_lsun_churches}

We evaluate preconditioned flow matching on five image datasets: LSUN Churches ($256 \times 256$) \citep{yu2015lsun}, Oxford Flowers-102 ($256 \times 256$) \citep{nilsback2008automated}, AFHQ Cats ($512 \times 512$) \citep{choi2020starganv2}, CIFAR-10 ($32 \times 32$) \citep{krizhevsky2009learning}, and ImageNet-1k ($64\times 64$) \citep{chrabaszcz2017downsampled}. For LSUN Churches, Oxford Flowers-102, AFHQ Cats, and ImageNet-1k, we train in the latent space of a pretrained Stable Diffusion VAE encoder--decoder \citep{rombach2022high}. This compresses images to $4\times32\times32$ latents for LSUN Churches and Oxford Flowers-102, $4\times64\times64$ latents for AFHQ Cats, and $4\times8\times8$ latents for ImageNet-1k. CIFAR-10 is trained directly in pixel space at its native $32\times32$ resolution.

The generator and preconditioner architectures are chosen according to the scale and structure of each dataset. For LSUN Churches and Oxford Flowers-102, both the main generator and the preconditioner use UNet backbones in the VAE latent space, with the preconditioner trained for fewer epochs than the main generator. For AFHQ Cats, both models operate in the VAE latent space: we use a UNet preconditioner and a DiT-L/2 backbone for the main flow matching generator. CIFAR-10 is modeled in pixel space using class-conditional UNets, with a smaller preconditioner obtained by reducing the base channel count. For ImageNet-1k, we use a DiT-B/2 generator and a smaller DiT-S/2 preconditioner in the VAE latent space.

For each dataset, we compare a standard flow matching generator against a preconditioned variant. The preconditioned variant first trains an auxiliary flow model, which is then used to transform the data distribution into a better-conditioned representation. The main flow matching generator is subsequently trained on this preconditioned distribution, and generated samples are mapped back through the inverse preconditioning flow at inference time. This design makes the preconditioning stage cheaper than the main generator, either by using a lower-capacity model or by training it to lower accuracy, while still improving the geometry of the downstream flow matching problem.

The exact architectures, training spaces, interpolation functions, latent dimensions, learning rates, batch sizes, training epochs, and model sizes are summarized in \Cref{tab:experimental_settings}.

\begin{table}[h]
\centering
\caption{Experimental settings and hyperparameters for LSUN Churches, Oxford Flowers-102, AFHQ Cats, CIFAR-10, and ImageNet-1k.}
\label{tab:experimental_settings}
\scriptsize
\resizebox{\linewidth}{!}{
\begin{tabular}{@{}lccccc@{}}
\toprule
\textbf{Setting}
& \textbf{LSUN Churches}
& \textbf{Oxford Flowers-102}
& \textbf{AFHQ Cats}
& \textbf{CIFAR-10}
& \textbf{ImageNet-1k} \\
\midrule

Image resolution 
& $256{\times}256$ 
& $256{\times}256$ 
& $512{\times}512$ 
& $32{\times}32$ 
& $64{\times}64$ \\

Training space 
& Latent 
& Latent 
& Latent 
& Pixel 
& Latent \\

VAE / latent encoder
& Stable Diffusion VAE 
& Stable Diffusion VAE 
& Stable Diffusion VAE 
& -- 
& Stable Diffusion VAE \\

Latent dim
& $4 \times 32 \times 32$
& $4 \times 32 \times 32$
& $4 \times 64 \times 64$
& -- 
&   $4 \times 8\times 8$\\

\multirow{2}{*}{Interpolation}
& $s(t)=\sin(\pi t / 2)$
& $s(t)=\sin(\pi t / 2)$
& $s(t)=\sin(\pi t / 2)$
& $s(t)=t$
& $s(t)=t$ \\

& $c(t)= \cos(\pi t / 2)$
& $c(t)= \cos(\pi t / 2)$
& $c(t)= \cos(\pi t / 2)$
& $c(t)=(1-t)$
& $c(t)=(1-t)$ \\

\midrule
\multicolumn{6}{@{}l}{\textbf{FM generator model}} \\
\midrule

Network 
& UNet 
& UNet 
& DiT-L/2 
& UNet
& DiT-B/2 \\

Base / hidden channels 
& 256 
& 256 
& 1024
& 256 
& 768 \\

Depth / residual blocks 
& 3 blocks 
& 3 blocks 
& 24 blocks 
& 3 blocks
& 12 blocks \\

Attention heads 
& 4 
& 4 
& 16 
& 4 
& 12 \\

Attention resolutions 
& $16{\times}16$   
& $16{\times}16$   
& --
& $16{\times}16$, $8{\times}8$ 
& -- \\

Patch size 
& -- 
& -- 
& $2{\times}2$ 
& -- 
& $2{\times}2$ \\

Class embedding dim
& --  
& --  
& -- 
& 128 
& 768 \\

CFG dropout 
& -- 
& --  
& -- 
& 0.1 
& 0.1 \\

Approx. parameters 
& 189M 
& 189M 
& 458M
& 214M
& 130M \\

Max epochs 
& 150 
& 300 
& 50  
& 1275 
& 350 \\

Learning rate 
& $5{\times}10^{-5}$ 
& $5{\times}10^{-5}$ 
& $1{\times}10^{-4}$  
& $2{\times}10^{-4}$  
& $1{\times}10^{-4}$ \\

Optimizer 
& AdamW 
& AdamW 
& AdamW 
& AdamW 
& AdamW \\

Batch size 
& 32 
& 32 
& 32 $(16{\times}2)$ 
& 512 
& 1536 \\

\midrule
\multicolumn{6}{@{}l}{\textbf{Preconditioner model}} \\
\midrule

Network 
& UNet 
& UNet 
& UNet 
& UNet
& DiT-S/2 \\

Base / hidden channels 
& 256 
& 256 
& 192 
& 128 
& 384 \\

Depth / residual blocks 
& 3 blocks 
& 3 blocks 
& 2 blocks 
& 3 blocks
& 12 blocks \\

Attention heads 
& 4 
& 4 
& 12--24 
& 4 
& 6 \\

Attention resolutions 
& $16{\times}16$   
& $16{\times}16$    
& $16{\times}16$, $8{\times}8$ 
& $16{\times}16$, $8{\times}8$ 
& -- \\

Patch size 
& -- 
& -- 
& $2{\times}2$ 
& -- 
& $2{\times}2$ \\

Class embedding dim 
& --   
& --   
& --   
& 128 
& 384 \\

Approx. parameters 
& 189M 
& 189M 
& 148M 
& 54M 
& 33M \\

Max epochs 
& 50 
& 100 
& 200   
& 200 
& 200 \\

Learning rate 
& $5{\times}10^{-5}$ 
& $5{\times}10^{-5}$ 
& $1{\times}10^{-4}$ 
& $1{\times}10^{-4}$
& $1{\times}10^{-4}$ \\

Optimizer 
& AdamW 
& AdamW 
& AdamW 
& AdamW 
& AdamW \\

Batch size 
& 32 
& 32 
& 32 $(16{\times}2)$ 
& 128
& 1536 \\

\bottomrule
\end{tabular}
}
\end{table}

\section{Computational Resources Used}
\label{app:computational_resources}

The experiments in this paper were run on a single compute server with three NVIDIA RTX A6000 GPUs, each providing 48\,GB of memory. The host machine further comprised two Intel Xeon Gold 5317 CPUs (24 physical cores / 48 logical CPUs total) and 1\,TB of RAM, and ran AlmaLinux 9.6.

All model training, sampling, and evaluation were performed on this system. Where applicable, we used distributed data-parallel training across the three GPUs. This hardware configuration was sufficient for all experiments reported in the paper.

\section{Existing Assets and Licenses}
\label{app:assets_licenses}

The experiments use existing public datasets and standard model components. We cite the original sources for all datasets used in the paper, including MNIST, CIFAR-10, ImageNet-1k, LSUN Churches, Oxford Flowers-102, and AFHQ Cats. We use these assets only for research and evaluation purposes and respect their associated terms of use. We also cite the Stable Diffusion VAE (CC-BY 4.0) checkpoint which was used to encode our images to the latent space for higher-resolution images. No new dataset is introduced or redistributed as part of this work. Any released code will include attribution to external codebases or libraries used, together with their corresponding licenses.


\end{document}